\documentclass[11pt,twoside]{article}
\usepackage[latin9]{inputenc}
\usepackage{float}
\usepackage{units}
\usepackage{amsmath}
\usepackage{amsthm}
\usepackage{amssymb}
\usepackage{graphicx}
\usepackage{xargs}[2008/03/08]
\usepackage[unicode=true]
 {hyperref}
\usepackage{breakurl}

\makeatletter

\providecommand{\tabularnewline}{\\}
\floatstyle{ruled}
\newfloat{algorithm}{tbp}{loa}
\providecommand{\algorithmname}{Algorithm}
\floatname{algorithm}{\protect\algorithmname}

\theoremstyle{plain}
\newtheorem{thm}{\protect\theoremname}
  \theoremstyle{definition}
  \newtheorem{defn}[thm]{\protect\definitionname}
  \theoremstyle{remark}
  \newtheorem{rem}[thm]{\protect\remarkname}
  \theoremstyle{plain}
  \newtheorem{lem}[thm]{\protect\lemmaname}
  \theoremstyle{plain}
  \newtheorem*{thm*}{\protect\theoremname}

\usepackage{jmlr2e}
\usepackage{natbib}
\usepackage{flushend}
\usepackage{graphicx}
\usepackage{algolyx}

\jmlrheading{x}{2016}{y-z}{a/b}{c/d}{Trung Le, Tu Dinh Nguyen, Vu Nguyen and Dinh Phung}

\ShortHeadings{Approximation Vector Machines}{Le et al}
\firstpageno{1}

\usepackage{colortbl} 
\definecolor{header_color}{rgb}{0.74,0.88,0.91}
\definecolor{even_color}{rgb}{0.9,0.9,0.9}
\definecolor{subheader_color}{rgb}{0.85,0.93,0.95}
\definecolor{childheader_color}{rgb}{1.0,0.93,0.87}

\definecolor{ccolor_best}{rgb}{1.0,0.9,0.9}
\definecolor{ccolor_wrong}{rgb}{1.0,0.85,0.85}

\newcolumntype{x}[1]{>{\centering\arraybackslash}p{#1}}

\@ifundefined{showcaptionsetup}{}{%
 \PassOptionsToPackage{caption=false}{subfig}}
\usepackage{subfig}
\makeatother

  \providecommand{\definitionname}{Definition}
  \providecommand{\lemmaname}{Lemma}
  \providecommand{\remarkname}{Remark}
  \providecommand{\theoremname}{Theorem}
\providecommand{\theoremname}{Theorem}

\begin{document}
\newcommand{\sidenote}[1]{\marginpar{\small \emph{\color{Medium}#1}}}

\global\long\def\se{\hat{\text{se}}}

\global\long\def\interior{\text{int}}

\global\long\def\boundary{\text{bd}}

\global\long\def\ML{\textsf{ML}}

\global\long\def\GML{\mathsf{GML}}

\global\long\def\HMM{\mathsf{HMM}}

\global\long\def\support{\text{supp}}

\global\long\def\new{\text{*}}

\global\long\def\stir{\text{Stirl}}

\global\long\def\mA{\mathcal{A}}

\global\long\def\mB{\mathcal{B}}

\global\long\def\mF{\mathcal{F}}

\global\long\def\mK{\mathcal{K}}

\global\long\def\mH{\mathcal{H}}

\global\long\def\normal{\mathcal{N}}

\global\long\def\mX{\mathcal{X}}

\global\long\def\mZ{\mathcal{Z}}

\global\long\def\mS{\mathcal{S}}

\global\long\def\Ical{\mathcal{I}}

\global\long\def\mT{\mathcal{T}}

\global\long\def\Pcal{\mathcal{P}}

\global\long\def\dist{d}

\global\long\def\HX{\entro\left(X\right)}
 \global\long\def\entropyX{\HX}

\global\long\def\HY{\entro\left(Y\right)}
 \global\long\def\entropyY{\HY}

\global\long\def\HXY{\entro\left(X,Y\right)}
 \global\long\def\entropyXY{\HXY}

\global\long\def\mutualXY{\mutual\left(X;Y\right)}
 \global\long\def\mutinfoXY{\mutualXY}

\global\long\def\given{\mid}

\global\long\def\gv{\given}

\global\long\def\goto{\rightarrow}

\global\long\def\asgoto{\stackrel{a.s.}{\longrightarrow}}

\global\long\def\pgoto{\stackrel{p}{\longrightarrow}}

\global\long\def\dgoto{\stackrel{d}{\longrightarrow}}

\global\long\def\lik{\mathcal{L}}

\global\long\def\logll{\mathit{l}}

\global\long\def\vectorize#1{\mathbf{#1}}

\global\long\def\vt#1{\mathbf{#1}}

\global\long\def\gvt#1{\boldsymbol{#1}}

\global\long\def\idp{\ \bot\negthickspace\negthickspace\bot\ }
 \global\long\def\cdp{\idp}

\global\long\def\das{}

\global\long\def\id{\mathbb{I}}

\global\long\def\idarg#1#2{\id\left\{  #1,#2\right\}  }

\global\long\def\iid{\stackrel{\text{iid}}{\sim}}

\global\long\def\bzero{\vt 0}

\global\long\def\bone{\mathbf{1}}

\global\long\def\boldm{\boldsymbol{m}}

\global\long\def\bff{\vt f}

\global\long\def\bx{\boldsymbol{x}}

\global\long\def\bl{\boldsymbol{l}}

\global\long\def\bu{\boldsymbol{u}}

\global\long\def\bo{\boldsymbol{o}}

\global\long\def\bh{\boldsymbol{h}}

\global\long\def\bs{\boldsymbol{s}}

\global\long\def\bz{\boldsymbol{z}}

\global\long\def\xnew{y}

\global\long\def\bxnew{\boldsymbol{y}}

\global\long\def\bX{\boldsymbol{X}}

\global\long\def\tbx{\tilde{\bx}}

\global\long\def\by{\boldsymbol{y}}

\global\long\def\bY{\boldsymbol{Y}}

\global\long\def\bZ{\boldsymbol{Z}}

\global\long\def\bU{\boldsymbol{U}}

\global\long\def\bv{\boldsymbol{v}}

\global\long\def\bn{\boldsymbol{n}}

\global\long\def\bV{\boldsymbol{V}}

\global\long\def\bI{\boldsymbol{I}}

\global\long\def\bw{\vt w}

\global\long\def\balpha{\gvt{\alpha}}

\global\long\def\bbeta{\gvt{\beta}}

\global\long\def\bmu{\gvt{\mu}}

\global\long\def\btheta{\boldsymbol{\theta}}

\global\long\def\blambda{\boldsymbol{\lambda}}

\global\long\def\bgamma{\boldsymbol{\gamma}}

\global\long\def\bpsi{\boldsymbol{\psi}}

\global\long\def\bphi{\boldsymbol{\phi}}

\global\long\def\bpi{\boldsymbol{\pi}}

\global\long\def\bomega{\boldsymbol{\omega}}

\global\long\def\bepsilon{\boldsymbol{\epsilon}}

\global\long\def\btau{\boldsymbol{\tau}}

\global\long\def\bxi{\boldsymbol{\xi}}

\global\long\def\realset{\mathbb{R}}

\global\long\def\realn{\realset^{n}}

\global\long\def\integerset{\mathbb{Z}}

\global\long\def\natset{\integerset}

\global\long\def\integer{\integerset}

\global\long\def\natn{\natset^{n}}

\global\long\def\rational{\mathbb{Q}}

\global\long\def\rationaln{\rational^{n}}

\global\long\def\complexset{\mathbb{C}}

\global\long\def\comp{\complexset}

\global\long\def\compl#1{#1^{\text{c}}}

\global\long\def\and{\cap}

\global\long\def\compn{\comp^{n}}

\global\long\def\comb#1#2{\left({#1\atop #2}\right) }

\global\long\def\nchoosek#1#2{\left({#1\atop #2}\right)}

\global\long\def\param{\vt w}

\global\long\def\Param{\Theta}

\global\long\def\meanparam{\gvt{\mu}}

\global\long\def\Meanparam{\mathcal{M}}

\global\long\def\meanmap{\mathbf{m}}

\global\long\def\logpart{A}

\global\long\def\simplex{\Delta}

\global\long\def\simplexn{\simplex^{n}}

\global\long\def\dirproc{\text{DP}}

\global\long\def\ggproc{\text{GG}}

\global\long\def\DP{\text{DP}}

\global\long\def\ndp{\text{nDP}}

\global\long\def\hdp{\text{HDP}}

\global\long\def\gempdf{\text{GEM}}

\global\long\def\rfs{\text{RFS}}

\global\long\def\bernrfs{\text{BernoulliRFS}}

\global\long\def\poissrfs{\text{PoissonRFS}}

\global\long\def\grad{\gradient}
 \global\long\def\gradient{\nabla}

\global\long\def\partdev#1#2{\partialdev{#1}{#2}}
 \global\long\def\partialdev#1#2{\frac{\partial#1}{\partial#2}}

\global\long\def\partddev#1#2{\partialdevdev{#1}{#2}}
 \global\long\def\partialdevdev#1#2{\frac{\partial^{2}#1}{\partial#2\partial#2^{\top}}}

\global\long\def\closure{\text{cl}}

\global\long\def\cpr#1#2{\Pr\left(#1\ |\ #2\right)}

\global\long\def\var{\text{Var}}

\global\long\def\Var#1{\text{Var}\left[#1\right]}

\global\long\def\cov{\text{Cov}}

\global\long\def\Cov#1{\cov\left[ #1 \right]}

\global\long\def\COV#1#2{\underset{#2}{\cov}\left[ #1 \right]}

\global\long\def\corr{\text{Corr}}

\global\long\def\sst{\text{T}}

\global\long\def\SST{\sst}

\global\long\def\ess{\mathbb{E}}

\global\long\def\Ess#1{\ess\left[#1\right]}

\newcommandx\ESS[2][usedefault, addprefix=\global, 1=]{\underset{#2}{\ess}\left[#1\right]}

\global\long\def\fisher{\mathcal{F}}

\global\long\def\bfield{\mathcal{B}}
 \global\long\def\borel{\mathcal{B}}

\global\long\def\bernpdf{\text{Bernoulli}}

\global\long\def\betapdf{\text{Beta}}

\global\long\def\dirpdf{\text{Dir}}

\global\long\def\gammapdf{\text{Gamma}}

\global\long\def\gaussden#1#2{\text{Normal}\left(#1, #2 \right) }

\global\long\def\gauss{\mathbf{N}}

\global\long\def\gausspdf#1#2#3{\text{Normal}\left( #1 \lcabra{#2, #3}\right) }

\global\long\def\multpdf{\text{Mult}}

\global\long\def\poiss{\text{Pois}}

\global\long\def\poissonpdf{\text{Poisson}}

\global\long\def\pgpdf{\text{PG}}

\global\long\def\wshpdf{\text{Wish}}

\global\long\def\iwshpdf{\text{InvWish}}

\global\long\def\nwpdf{\text{NW}}

\global\long\def\niwpdf{\text{NIW}}

\global\long\def\studentpdf{\text{Student}}

\global\long\def\unipdf{\text{Uni}}

\global\long\def\transp#1{\transpose{#1}}
 \global\long\def\transpose#1{#1^{\mathsf{T}}}

\global\long\def\mgt{\succ}

\global\long\def\mge{\succeq}

\global\long\def\idenmat{\mathbf{I}}

\global\long\def\trace{\mathrm{tr}}

\global\long\def\argmax#1{\underset{_{#1}}{\text{argmax}} }

\global\long\def\argmin#1{\underset{_{#1}}{\text{argmin}\ } }

\global\long\def\diag{\text{diag}}

\global\long\def\norm{}

\global\long\def\spn{\text{span}}

\global\long\def\vtspace{\mathcal{V}}

\global\long\def\field{\mathcal{F}}
 \global\long\def\ffield{\mathcal{F}}

\global\long\def\inner#1#2{\left\langle #1,#2\right\rangle }
 \global\long\def\iprod#1#2{\inner{#1}{#2}}

\global\long\def\dprod#1#2{#1 \cdot#2}

\global\long\def\norm#1{\left\Vert #1\right\Vert }

\global\long\def\entro{\mathbb{H}}

\global\long\def\entropy{\mathbb{H}}

\global\long\def\Entro#1{\entro\left[#1\right]}

\global\long\def\Entropy#1{\Entro{#1}}

\global\long\def\mutinfo{\mathbb{I}}

\global\long\def\relH{\mathit{D}}

\global\long\def\reldiv#1#2{\relH\left(#1||#2\right)}

\global\long\def\KL{KL}

\global\long\def\KLdiv#1#2{\KL\left(#1\parallel#2\right)}
 \global\long\def\KLdivergence#1#2{\KL\left(#1\ \parallel\ #2\right)}

\global\long\def\crossH{\mathcal{C}}
 \global\long\def\crossentropy{\mathcal{C}}

\global\long\def\crossHxy#1#2{\crossentropy\left(#1\parallel#2\right)}

\global\long\def\breg{\text{BD}}

\global\long\def\lcabra#1{\left|#1\right.}

\global\long\def\lbra#1{\lcabra{#1}}

\global\long\def\rcabra#1{\left.#1\right|}

\global\long\def\rbra#1{\rcabra{#1}}

\global\long\def\model{\text{AVM}}

\editor{Koby Crammer}

\title{Approximation Vector Machines\\
for Large-scale Online Learning}

\author{\name{T}rung Le\thanks{Part of this work was performed while the author was affiliated with
the HCM University of Education.} \email trung.l@deakin.edu.au \\
 \addr Centre for Pattern Recognition and Data Analytics, Australia\\
\AND \name{T}u Dinh Nguyen \email tu.nguyen@deakin.edu.au\\
\addr Centre for Pattern Recognition and Data Analytics, Australia\\
\AND \name{V}u Nguyen\email v.nguyen@deakin.edu.au\\
\addr Centre for Pattern Recognition and Data Analytics, Australia\\
\AND\name{D}inh Phung \email dinh.phung@deakin.edu.au \\
 \addr Centre for Pattern Recognition and Data Analytics, Australia}

\maketitle
\begin{abstract}
One of the most challenging problems in kernel online learning is
to bound the model size and to promote model sparsity. Sparse models
not only improve computation and memory usage, but also enhance the
generalization capacity \textendash{} a principle that concurs with
the law of parsimony. However, inappropriate sparsity modeling may
also significantly degrade the performance. In this paper, we propose
Approximation Vector Machine ($\model$), a model that can simultaneously
encourage sparsity and safeguard its risk in compromising the performance.
In an online setting context, when an incoming instance arrives, we
approximate this instance by one of its neighbors whose distance to
it is less than a predefined threshold. Our key intuition is that
since the newly seen instance is expressed by its nearby neighbor
the optimal performance can be analytically formulated and maintained.
We develop theoretical foundations to support this intuition and further
establish an analysis for the common loss functions including Hinge,
smooth Hinge, and Logistic (i.e., for the classification task) and
$\ell_{1}$, $\ell_{2}$, and $\varepsilon$-insensitive (i.e., for
the regression task) to characterize the gap between the approximation
and optimal solutions. This gap crucially depends on two key factors
including the frequency of approximation (i.e., how frequent the approximation
operation takes place) and the predefined threshold. We conducted
extensive experiments for classification and regression tasks in batch
and online modes using several benchmark datasets. The quantitative
results show that our proposed $\model$ obtained comparable predictive
performances with current state-of-the-art methods while simultaneously
achieving significant computational speed-up due to the ability of
the proposed $\model$ in maintaining the model size. 
\end{abstract}
\begin{keywords}kernel, online learning, large-scale machine learning,
sparsity, big data, core set, stochastic gradient descent, convergence
analysis \end{keywords}

\section{Introduction\label{sec:Introduction}}

In modern machine learning systems, data usually arrive continuously
in stream. To enable efficient computation and to effectively handle
memory resource, the system should be able to adapt according to incoming
data. Online learning represents a family of efficient and scalable
learning algorithms for building a predictive model incrementally
from a sequence of data examples \citep{rosenblatt58a,Zinkevich03}.
In contrast to the conventional learning algorithms \citep{joachims1999,libsvm},
which usually require a costly procedure to retrain the entire dataset
when a new instance arrives, online learning aims to utilize the new
incoming instances to improve the model given the knowledge of the
correct answers to previous processed data (and possibly additional
available information), making them suitable for large-scale online
applications wherein data usually arrive sequentially and evolve rapidly. 

The seminal line of work in online learning, referred to as \textit{linear
online learning} \citep{rosenblatt58a,Crammer06onlinepassive-aggressive,Dredze2008},
aims at learning a linear predictor in the input space. The crucial
limitation of this approach lies in its over-simplified linear modeling
choice and consequently may fail to capture non-linearity commonly
seen in many real-world applications. This motivated the works in
\textit{kernel-based online learning} \citep{Freund1999:LMC:337859.337869,Kivinen2004}
in which a linear model in the feature space corresponding with a
nonlinear model in the input space, hence allows one to cope with
a variety of data distributions.

One common issue with \textit{kernel-based online learning approach},
also known as the \emph{curse of kernelization}, is that the model
size (i.e., the number of vectors with non-zero coefficients) may
grow linearly with the data size accumulated over time, hence causing
computational problem and potential memory overflow \citep{Steinwart:2003,wang2012}.
Therefore in practice, one might prefer kernel-based online learning
methods with guaranty on a limited and bounded model size. In addition,
enhancing model sparsity is also of great interest to practitioners
since this allows the generalization capacity to be improved; and
in many cases leading to a faster computation. However, encouraging
sparsity needs to be done with care since an inappropriate sparsity-encouraging
mechanism may compromise the performance. To address the curse of
kernelization, budgeted approaches \citep{Crammer04onlineclassification,DekelSS05,CavallantiCG07,Wang2010,wang2012,le_duong_dinh_nguyen_nguyen_phung_uai16budgeted,le_nguyen_phung_aistats16nonparametric}
limits the model size to a predefined budget $B$. Specifically, when
the current model size exceeds this budget, a budget maintenance strategy
(e.g., removal, projection, or merging) is triggered to recover the
model size back to the budget $B$. In these approaches, determining
a suitable value for the predefined budget in a principled way is
important, but challenging, since setting a small budget makes the
learning faster but may suffer from underfitting, whereas a large
budget makes the model fit better to data but may dramatically slow
down the training process. An alternative way to address the curse
of kernelization is to use random features \citep{Rahimi07randomfeatures}
to approximate a kernel function \citep{Lin2014,Lu_2015large,le_nips_2016}.
For example, \citet{Lu_2015large} proposed to transform data from
the input space to the random-feature space, and then performed SGD
in the feature space. However, in order for this approach to achieve
good kernel approximation, excessive number of random features is
required which could lead to a serious computational issue. To reduce
the impact number of random features, \citep{le_nips_2016} proposed
to distribute the model in dual space including the original feature
space and the random feature space that approximates the first space.
\begin{figure}[h]
\begin{centering}
\includegraphics[width=14cm]{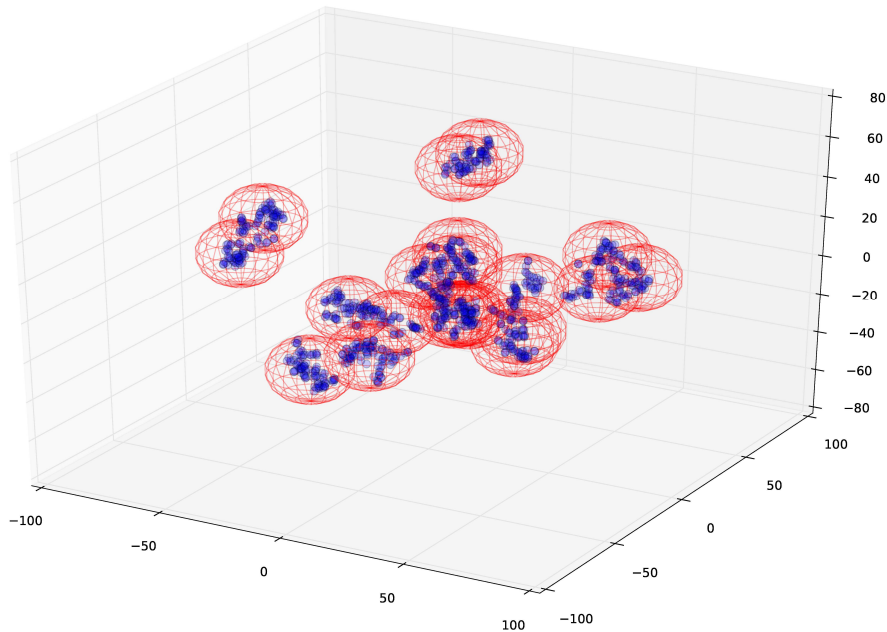}
\par\end{centering}
\caption{An illustration of the hypersphere coverage for $1,000$ data samples
which locate in 3D space. We cover this dataset using hyperspheres
with the diameter $\delta=7.0$, resulting in $20$ hypersphere cells
as shown in the figure (cf. Sections (\ref{subsec:Construction-of--Coverage},\ref{sec:Experiments})).
All data samples in a same cell are approximated by a core point in
this cell. The model size is therefore significantly reduced from
1,000 to 20.\protect\footnotemark \label{fig:a9a_coverage}}
\end{figure}

\footnotetext{In fact, we used a subset of the dataset a9a which has 123 features. We then project all data points onto 3D using t-SNE. We note that the t-SNE does not do clustering, it only reduces the dimensionality into 3D for visualization while trying to preserve the local properties of the data.}

In this paper, we propose \emph{Approximation Vector Machine} ($\model$)
to simultaneously encourage model sparsity\footnote{Model sparsity can be computed as the ratio of the model size and
the number of vectors received so far.} while preserving the model performance. Our model size is theoretically
proven to be bounded regardless of the data distribution and data
arrival order. To promote sparsity, we introduce the notion of $\delta$\textit{-coverage}
which partitions the data space into overlapped cells whose diameters
are defined by $\delta$ (cf. Figure \ref{fig:a9a_coverage}). This
coverage can be constructed in advance or on the fly. Our experiment
on the real datasets shows that the coverage can impressively boost
sparsity; for example with dataset \emph{KDDCup99} of $4,408,589$
instances, our model size is $115$ with $\delta=3$ (i.e., only $115$
cells are required); with dataset \emph{airlines }of $5,336,471$
instances, our model size is $388$ with $\delta=1$. 

In an online setting context, when an incoming instance arrives, it
can be approximated with the corresponding core point in the cell
that contains it. Our intuitive reason is that when an instance is
approximated by an its nearby core point, the performance would be
largely preserved. We further developed rigorous theory to support
this intuitive reason. In particular, our convergence analysis (covers
six popular loss functions, namely Hinge, smooth Hinge, and Logistic
for classification task and $\ell_{2}$, $\ell_{1}$, and $\varepsilon$-insensitive
for regression task) explicitly characterizes the gap between the
approximate and optimal solutions. The analysis shows that this gap
crucially depends on two key factors including the cell diameter $\delta$
and the approximation process. In addition, the cell parameter $\delta$
can be used to efficiently control the trade-off between sparsity
level and the model performance. We conducted extensive experiments
to validate the proposed method on a variety of learning tasks, including
classification in batch mode, classification and regression in online
mode on several benchmark large-scale datasets. The experimental results
demonstrate that our proposed method maintains a comparable predictive
performance while simultaneously achieving an order of magnitude speed-up
in computation comparing with the baselines due to its capacity in
maintaining model size. We would like to emphasize at the outset that
unlike budgeted algorithms (e.g., \citep{Crammer04onlineclassification,DekelSS05,CavallantiCG07,Wang2010,wang2012,le_duong_dinh_nguyen_nguyen_phung_uai16budgeted,le_nguyen_phung_aistats16nonparametric}),
our proposed method is nonparametric in the sense that the number
of core sets grow with data on demand, hence care should be exercised
in practical implementation.

The rest of this paper is organized as follows. In Section \ref{sec:Related-Work},
we review works mostly related to ours. In Section \ref{sec:Primal-and-Dual},
we present the primal and dual forms of Support Vector Machine (SVM)
as they are important background for our work. Section \ref{sec:Problem-Setting}
formulates the proposed problem. In Section \ref{sec:SGD}, we discuss
the standard SGD for kernel online learning with an emphasis on the
\textit{curse of kernelization}. Section \ref{sec:Approximation-Vector-Machines}
presents our proposed AVM with full technical details. Section \ref{sec:Loss-Function}
devotes to study the suitability of loss functions followed by Section
\ref{sec:Multiclass-Setting} where we extend the framework to multi-class
setting. Finally, in Section \ref{sec:Experiments}, we conduct extensive
experiments on several benchmark datasets and then discuss experimental
results as well as their implications. In addition, all supporting
proof is provided in the appendix sections.

\section{Related Work \label{sec:Related-Work}}

One common goal of online kernel learning is to bound the model size
and to encourage  sparsity. Generally, research in this direction
can be broadly reviewed into the following themes.

\textit{Budgeted Online Learning}. This approach limits the model
size to a predefined budget $B$. When the model size exceeds the
budget, a budget maintenance strategy is triggered to decrement the
model size by one. Three popular budget maintenance strategies are
\emph{removal}, \emph{projection}, and \emph{merging}. In the removal
strategy, the most redundant support vector is simply eliminated.
In the projection strategy, the information of the most redundant
support vector is conserved through its projection onto the linear
span of the remaining support vectors. The merging strategy first
selects two vectors, and then merges them into one before discarding
them. Forgetron \citep{DekelSS05} is the first budgeted online learning
method that employs the removal strategy for the budget maintenance.
At each iteration, if the classifier makes a mistake, it conducts
a three-step update: (i) running the standard Perceptron \citep{rosenblatt58a}
update; (ii) shrinking the coefficients of support vectors with a
scaling factor; and (iii) removing the support vector with the smallest
coefficient. Randomized Budget Perceptron (RBP) \citep{CavallantiCG07}
randomly removes a support vector when the model size overflows the
budget. Budget Perceptron \citep{Crammer04onlineclassification} and
Budgeted Passive Aggressive (BPA-S)\citep{Wang2010} attempt to discard
the most redundant support vector (SV). \citet{Orabona2009} used
the projection to automatically discover the model size. The new vector
is added to the support set if its projection onto the linear span
of others in the feature space exceeds a predefined threshold, or
otherwise its information is kept through the projection. Other works
involving the projection strategy include Budgeted Passive Aggressive
Nearest Neighbor (BPA-NN) \citep{Wang2010,wang2012}. The merging
strategy was used in some works \citep{Wang09,wang2012}.

\textit{Random Features}. The idea of random features was proposed
in \citep{Rahimi07randomfeatures}. Its aim is to approximate a shift-invariant
kernel using the harmonic functions. In the context of online kernel
learning, the problem of model size vanishes since we can store the
model directly in the random features. However, the arising question
is to determine the appropriate number of random features $D$ to
sufficiently approximate the real kernel while keeping this dimension
as small as possible for an efficient computation. \citet{Lin2014}
investigated the number of random features in the online kernel learning
context. Recently, \citet{Lu_2015large} proposed to run stochastic
gradient descent (SGD) in the random feature space rather than that
in the real feature space. The theory accompanied with this work shows
that with a high confidence level, SGD in the random feature space
can sufficiently approximate that in the real kernel space. Nonetheless,
in order to achieve good kernel approximation in this approach, excessive
number of random features is required, possibly leading to a serious
computational issue. To reduce the impact of the number of random
features to learning performance, \citep{le_nips_2016} proposed to
store core vectors in the original feature space, whilst storing remaining
vectors in the random feature space that sufficiently approximates
the first space.

\textit{Core Set}. This approach utilizes a core set to represent
the model. This core set can be constructed on the fly or in advance.
Notable works consist of the Core Vector Machine (CVM) \citep{Tsang05corevector}
and its simplified version, the Ball Vector Machine (BVM) \citep{Tsang2007}.
The CVM was based on the achievement in computational geometry \citep{Badoiu02optimalcore-sets}
to reformulate a variation of $\ell_{2}$-SVM as a problem of finding
minimal enclosing ball (MEB) and the core set includes the points
lying furthest away the current centre of the current MEB. Our work
can be categorized into this line of thinking. However, our work is
completely different to \citep{Tsang05corevector,Tsang2007} in the
mechanism to determine the core set and update the model. In addition,
the works of \citep{Tsang05corevector,Tsang2007} are not applicable
for the online learning.

\section{Primal and Dual Forms of Support Vector Machine \label{sec:Primal-and-Dual}}

Support Vector Machine (SVM) \citep{cortes1995support} represents
one of the state-of-the-art methods for classification. Given a training
set $\mathcal{D}=\left\{ \left(x_{1},y_{1}\right),\ldots,\left(x_{N},y_{N}\right)\right\} $,
the data instances are mapped to a feature space using the transformation
$\Phi\left(.\right)$, and then SVM aims to learn an optimal hyperplane
in the feature space such that the margin, the distance from the closest
data instance to the hyperplane, is maximized. The optimization problem
of SVM can be formulated as follows
\begin{align}
 & \,\min_{\bw,b}\left(\frac{\lambda}{2}\norm{\bw}^{2}+\frac{1}{N}\sum_{i=1}^{N}\xi_{i}\right)\label{eq:primal_constrained}\\
\text{s.t.}: & \,y_{i}\left(\transp{\bw}\Phi\left(x_{i}\right)+b\right)\geq1-\xi_{i},\,i=1,...,N\nonumber \\
 & \,\xi_{i}\geq0,\,i=1,...,N\nonumber 
\end{align}
where $\lambda>0$ is the regularization parameter, $\Phi\left(.\right)$
is the transformation from the input space to the feature space, and
$\boldsymbol{\xi}=\left[\xi_{i}\right]_{i=1}^{N}$ is the vector of
slack variables.

Using Karush-Kuhn-Tucker theorem, the above optimization problem is
transformed to the \emph{dual form} as follows
\begin{align*}
 & \,\min_{\boldsymbol{\alpha}}\left(\frac{1}{2}\transp{\boldsymbol{\alpha}}Q\boldsymbol{\alpha}-\transp{\boldsymbol{e}}\boldsymbol{\alpha}\right)\\
\text{s.t.}: & \,\boldsymbol{\transp y\alpha}=0\\
 & \,0\leq\alpha_{i}\leq\frac{1}{\lambda N},\,i=1,...,N
\end{align*}
where $Q=\left[y_{i}y_{j}K\left(x_{i},x_{j}\right)\right]_{i,j=1}^{N}$
is the Gram matrix, $K\left(x,x'\right)=\transp{\Phi\left(x\right)}\Phi\left(x'\right)$
is a kernel function, $\boldsymbol{e}=\left[1\right]_{N\times1}$
is the vector of all 1, and $\boldsymbol{y}=\transp{\left[y_{i}\right]_{i=1,...,N}}$. 

The dual optimization problem can be solved using the solvers \citep{joachims1999,libsvm}.
However, the computational complexity of the solvers is over-quadratic
\citep{shalev2008svm} and the dual form does not appeal to the online
learning setting. To scale up SVM and make it appealing to the online
learning, we rewrite the constrained optimization problem in Eq. (\ref{eq:primal_constrained})
in the \emph{primal form} as follows
\begin{equation}
\min_{\bw}\left(\frac{\lambda}{2}\norm{\bw}^{2}+\frac{1}{N}\sum_{i=1}^{N}l\left(\bw;x_{i},y_{i}\right)\right)\label{eq:primal}
\end{equation}
where $l\left(\bw;x,y\right)=\max\left(0,1-y\transp{\bw}\Phi\left(x\right)\right)$\footnote{We can eliminate the bias $b$ by simply adjusting the kernel.}
is Hinge loss. 

In our current interest, the advantages of formulating the optimization
problem of SVM in the primal form as in Eq. (\ref{eq:primal}) are
at least two-fold. First, it encourages the application of SGD-based
method to propose a solution for the online learning context. Second,
it allows us to extend Hinge loss to any appropriate loss functions
(cf. Section \ref{sec:Loss-Function}) to enrich a wider class of
problems that can be addressed.

\section{Problem Setting \label{sec:Problem-Setting}}

We consider two following optimization problems for batch and online
settings respectively in Eqs. (\ref{eq:ObjectiveFunction}) and (\ref{eq:online})
\begin{flalign}
\underset{\bw}{\min}\,f\left(\bw\right) & \triangleq\frac{\lambda}{2}{\normalcolor \norm{\bw}}^{2}+\mathbb{E}_{\left(x,y\right)\sim\mathbb{P}_{N}}\left[l\left(\bw;x,y\right)\right]\nonumber \\
 & \triangleq\frac{\lambda}{2}\norm{\bw}^{2}+\frac{1}{N}\sum_{i=1}^{N}l\left(\bw;x_{i},y_{i}\right)\label{eq:ObjectiveFunction}\\
\underset{\bw}{\min}\,f\left(\bw\right) & \triangleq\frac{\lambda}{2}\norm{\bw}^{2}+\mathbb{E}_{\left(x,y\right)\sim\mathbb{P}_{\mathcal{X},\mathcal{Y}}}\left[l\left(\bw;x,y\right)\right]\label{eq:online}
\end{flalign}

\noindent where $l\left(\bw;x,y\right)$ is a \emph{convex} loss function,
$\mathbb{P}_{\mathcal{X},\mathcal{Y}}$ is the joint distribution
of $\left(x,y\right)$ over $\mathcal{X}\times\mathcal{Y}$ with the
data domain $\mathcal{X}$ and the label domain $\mathcal{Y}$, and
$\mathbb{P}_{N}$ specifies the empirical distribution over the training
set $\mathcal{D}=\left\{ \left(x_{1},y_{1}\right),\ldots,\left(x_{N},y_{N}\right)\right\} $.
Furthermore, we assume that the \textit{\emph{convex}} loss function
$l\left(\bw;x,y\right)$ satisfies the following property: there exists
two positive numbers $A$ and $B$ such that $\norm{l^{'}\left(\bw;x,y\right)}\leq A\norm{\bw}^{1/2}+B,\,\forall\bw,x,y$.
As demonstrated in Section \ref{sec:Loss-Function}, this condition
is valid for all common loss functions. Hereafter, for given any function
$g(\bw)$, we use the notation $g^{'}\left(\bw_{0}\right)$ to denote
the gradient (or any sub-gradient) of $g\left(.\right)$ w.r.t $\bw$
evaluated at $\bw_{0}$. 

It is clear that given a fixed $\bw$, there exists a random variable
$g$ such that $\mathbb{E}\left[g\mid\bw\right]=f^{'}\left(\bw\right)$.
In fact, we can specify $g=\lambda\bw+l^{'}\left(\bw;x_{t},y_{t}\right)$
where $\left(x_{t},y_{t}\right)\sim\mathbb{P}_{\mathcal{X},\mathcal{Y}}\text{ or \ensuremath{\mathbb{P}_{N}}}$.
We assume that a \emph{positive semi-definite} (p.s.d.) and \emph{isotropic}
(iso.) kernel \citep{Rasmussen:2005} is used, i.e., $K\left(x,x^{'}\right)=k\left(\norm{x-x^{'}}^{2}\right)$,
where $k:\,\mathcal{X}\goto\mathbb{R}$ is an appropriate function.
Let $\Phi\left(.\right)$ be the feature map corresponding the kernel
(i.e., $K\left(x,x^{'}\right)=\transp{\Phi\left(x\right)}\Phi\left(x^{'}\right)$).
To simplify the convergence analysis, without loss of generality we
further assume that $\norm{\Phi\left(x\right)}^{2}=K\left(x,x\right)=1,\forall x\in\mathcal{X}$.
Finally, we denote the optimal solution of optimization problem in
Eq. (\ref{eq:ObjectiveFunction}) or (\ref{eq:online}) by $\bw^{*}$,
that is, $\bw^{*}=\text{argmin}_{\bw}\,f\left(\bw\right)$.

\section{Stochastic Gradient Descent Method \label{sec:SGD}}

We introduce the standard kernel stochastic gradient descent (SGD)
in Algorithm \ref{alg:SGD} wherein the standard learning rate $\eta_{t}=\frac{1}{\lambda t}$
is used \citep{Shalev-Shwartz:2007,shalev2011pegasos}. Let $\alpha_{t}$
be a scalar such that $l^{'}\left(\bw_{t};x_{t},y_{t}\right)=\alpha_{t}\Phi\left(x_{t}\right)$
(we note that this scalar exists for all common loss functions as
presented in Section \ref{sec:Loss-Function}). It is apparent that
at the iteration $t$ the model $\bw_{t}$ has the form of $\bw_{t}=\sum_{i=1}^{t}\alpha_{i}^{\left(t\right)}\Phi\left(x_{i}\right)$.
The vector $x_{i}$ ($1\leq i\leq t)$ is said to be a \emph{support
vector} if its coefficient $\alpha_{i}^{\left(t\right)}$ is nonzero.
The model is represented through the support vectors, and hence we
can define the model size to be $\norm{\balpha^{\left(t\right)}}_{0}$
and model sparsity as the ratio between the current model size and
$t$ (i.e., $\norm{\balpha^{\left(t\right)}}_{0}/t$). Since it is
likely that $\alpha_{t}$ is nonzero (e.g., with Hinge loss, it happens
if $x_{t}$ lies in the margins of the current hyperplane), the standard
kernel SGD algorithm is vulnerable to the curse of kernelization,
that is, the model size, is almost linearly grown with the data size
accumulated over time \citep{Steinwart:2003}. Consequently, the computation
gradually becomes slower or even infeasible when the data size grows
rapidly.

\begin{algorithm}[h]
\caption{Stochastic Gradient Descent algorithm.\label{alg:SGD}}

\textbf{Input:} $\lambda$, $\text{p.s.d. kernel }K\left(.,.\right)=\transp{\Phi\left(\cdot\right)}\Phi\left(\cdot\right)$
\begin{algor}[1]
\item [{{*}}] $\bw_{1}=\bzero$
\item [{for}] $t=1,2,\ldots T$ 
\item [{{*}}] Receive $\left(x_{t},y_{t}\right)$\hspace*{\fill} //$\left(x_{t},y_{t}\right)\sim\mathbb{P}_{\mathcal{X},\mathcal{Y}}$
\textit{or} $\ensuremath{\mathbb{P}_{N}}$
\item [{{*}}] $\eta_{t}=\frac{1}{\lambda t}$
\item [{{*}}] $g_{t}=\lambda\bw_{t}+l^{'}\left(\bw_{t};x_{t},y_{t}\right)=\lambda\bw_{t}+\alpha_{t}\Phi\left(x_{t}\right)$
\item [{{*}}] $\bw_{t+1}=\bw_{t}-\eta_{t}g_{t}=\frac{t-1}{t}\bw_{t}-\eta_{t}\alpha_{t}\Phi\left(x_{t}\right)$ 
\item [{endfor}]~
\end{algor}
\textbf{Output}: $\overline{\bw}_{T}=\frac{1}{T}\sum_{t=1}^{T}\bw_{t}$
or $\bw_{T+1}$
\end{algorithm}

\section{Approximation Vector Machines for Large-scale Online Learning \label{sec:Approximation-Vector-Machines}}

In this section, we introduce our proposed Approximation Vector Machine
($\model$) for online learning. The main idea is that we employ an
overlapping partition of sufficiently small cells to cover the data
domain, i.e., $\mathcal{X}$ or $\Phi\left(\mathcal{X}\right)$; when
an instance arrives, we approximate this instance by a corresponding
core point in the cell that contains this instance. Our intuition
behind this approximation procedure is that since the instance is
approximated by its neighbor, the performance would not be significantly
compromised while gaining significant speedup. We start this section
with the definition of $\delta$-coverage, its properties and connection
with the feature space. We then present $\model$ and the convergence
analysis. 

\subsection{$\delta$-coverage over a domain \label{subsec:coverage}}

To facilitate our technical development in sequel, we introduce the
notion of $\delta$-coverage in this subsection. We first start with
the usual definition of a diameter for a set.
\begin{defn}
(\emph{diameter}) Given a set $A$, the diameter of this set is defined
as $D\left(A\right)=\underset{x,x^{'}\in A}{\sup}||x-x^{'}||$. This
is the maximal pairwise distance between any two points in $A$.
\end{defn}
Next, given a domain $\mathbb{\mathcal{X}}$ (e.g., the data domain,
input space) we introduce the concept of $\delta$-coverage for $\mathcal{X}$
using a collection of sets.
\begin{defn}
($\delta$\emph{-coverage}) The collection of sets $\mathcal{P}=\left(P_{i}\right)_{i\in I}$
is said to be an $\delta$-coverage of the domain $\mathcal{X}$ iff
$\mathcal{X}\subset\cup_{i\in I}P_{i}$ and $D\left(P_{i}\right)\leq\delta,\forall i\in I$
where $I$ is the index set (not necessarily discrete) and each element
$P_{i}\in\mathcal{P}$ is further referred to as a \emph{cell}. Furthermore
if the index set $I$ is finite, the collection $\mathcal{P}$ is
called a \textit{finite} $\delta$-coverage.
\end{defn}

\begin{defn}
(\emph{core set, core point}) Given an $\delta$-coverage $\mathcal{P}=\left(P_{i}\right)_{i\in I}$
over a given domain $\mathcal{X}$, for each $i\in I$, we select
an arbitrary point $c_{i}$ from the cell $P_{i}$, then the collection
of all $c_{i}$ (s) is called the \emph{core set} $\mathcal{C}$ of
the $\delta$-coverage $\mathcal{P}$.  Each point $c_{i}\in\mathcal{C}$
is further referred to as a \textit{core point}.

We show that these definitions can be also extended to the feature
space with the mapping $\Phi$ and kernel $K$ via the following theorem.
\end{defn}
\begin{thm}
\label{thm:coverage_image}Assume that the p.s.d. and isotropic kernel
$K(x,x^{'})=k\left(||x-x^{'}||^{2}\right)$, where $k\left(.\right)$
is a monotonically continuous decreasing function with $k\left(0\right)=1$,
is examined and $\Phi\left(.\right)$ is its induced feature map.
If $\mathcal{P}=\left(P_{i}\right)_{i\in I}$ is an $\delta$-coverage
of the domain $\mathcal{X}$ then $\Phi\left(\mathcal{P}\right)=\left(\Phi\left(\mathcal{P}_{i}\right)\right)_{i\in I}$
is \textbf{\emph{also}} an $\delta_{\Phi}$-coverage of the domain
$\Phi\left(\mathcal{X}\right)$, where $\delta_{\Phi}=\sqrt{2\left(1-k\left(\delta^{2}\right)\right)}$
is a monotonically increasing function and $\underset{\delta\goto0}{\lim}\,\delta_{\Phi}=0$. 
\end{thm}
In particular, the Gaussian kernel given by $K(x,x^{'})=\exp(-\gamma\norm{x-x^{'}}^{2})$
is a p.s.d. and iso. kernel and $\delta_{\Phi}=\sqrt{2\left(1-\exp\left(-\gamma\delta^{2}\right)\right)}$.
Theorem \ref{thm:coverage_image} further reveals that the image of
an $\delta$-coverage in the input space is an $\delta_{\Phi}$-coverage
in the feature space and when the diameter $\delta$ approaches $0$,
so does the induced diameter $\delta_{\Phi}$. For readability, the
proof of this theorem is provided in Appendix \ref{sec:appxA}. 

We have further developed methods and algorithms to efficiently construct
$\delta$-coverage, however to maintain the readability, we defer
this construction to Section \ref{subsec:Construction-of--Coverage}.

\subsection{Approximation Vector Machines\label{subsec:avm}}

We now present our proposed Approximation Vector Machine ($\model$)
for online learning. In an online setting, instances arise on the
fly and we need an efficient approach to incorporate incoming instances
into the learner. Different from the existing works (cf. Section \ref{sec:Related-Work}),
our approach is to construct an $\delta$-coverage $\mathcal{P}=(P_{i})_{i\in I}$
over the input domain $\mathcal{X}$, and for each incoming instance
$x$ we find the cell $P_{i}$ that contains this instance and  approximate
this instance by a core point $c_{i}\in P_{i}$. The coverage $\mathcal{P}$
and core set $\mathcal{C}$ can either be constructed in advance or
on the fly as presented in Section \ref{subsec:Construction-of--Coverage}.

In Algorithm \ref{alg:AVM}, when receiving an incoming instance $\left(x_{t},y_{t}\right)$,
we compute the scalar $\alpha_{t}$ such that $\alpha_{t}\Phi\left(x_{t}\right)=l^{'}\left(\bw_{t};x_{t},y_{t}\right)$
(cf. Section \ref{sec:Loss-Function}) in Step 5. Furthermore at Step
7 we introduce a Bernoulli random variable $Z_{t}$ to govern the
approximation procedure. This random variable could be either statistically
independent or dependent with the incoming instances and the current
model. In Section \ref{subsec:exp_hyperparam}, we report on different
settings for $Z_{t}$ and how they influence the model size and learning
performance. Our findings at the outset is that, the naive setting
with $\mathbb{P}\left(Z_{t}=1\right)=1,\,\forall t$ (i.e., always
performing approximation) returns the sparsest model while obtaining
comparable learning performance comparing with the other settings.
Moreover, as shown in Steps 9 and 11, we only approximate the incoming
data instance by the corresponding core point (i.e., $c_{i_{t}}$)
if $Z_{t}=1$. In addition, if $Z_{t}=1$, we find a cell that contains
this instance in Step 8. It is worth noting that the $\delta$-coverage
and the cells are constructed on the fly along with the data arrival
(cf. Algorithms \ref{alg:sphere_coverage} and \ref{alg:rec_coverage}).
In other words, the incoming data instance might belong to an existing
cell or a new cell that has the incoming instance as its core point
is created.

Furthermore to ensure that $\norm{\bw_{t}}$ is bounded for all $t\geq1$
in the case of $\ell_{2}$ loss, if $\lambda\leq1$ then we project
$\bw_{t}-\eta_{t}h_{t}$ onto the hypersphere with centre origin and
radius $y_{\text{max}}\lambda^{-1/2}$, i.e., $\mathcal{B}\left(\bzero,y_{\text{max}}\lambda^{-1/2}\right)$.
Since it can be shown that with $\ell_{2}$ loss the optimal solution
$\bw^{*}$ lies in $\mathcal{B}\left(\bzero,y_{\text{max}}\lambda^{-1/2}\right)$
(cf. Theorem \ref{thm:bound_w*} in Appendix \ref{sec:appxC}), this
operation could possibly result in a faster convergence. In addition,
by reusing the previous information, this operation can be efficiently
implemented. Finally, we note that with $\ell_{2}$ loss and $\lambda>1$,
we do not need to perform a projection to bound $\norm{\bw_{t}}$
since according to Theorem \ref{thm:bound_wT} in Appendix \ref{sec:appxC},
$\norm{\bw_{t}}$ is bounded by $\frac{y_{\text{max}}}{\lambda-1}$.
Here it is worth noting that we have defined $y_{\text{max}}=\max_{y\in\mathcal{Y}}\left|y\right|$
and this notation is only used in the analysis for the regression
task with the $\ell_{2}$ loss. 
\begin{algorithm}[H]
\caption{Approximation Vector Machine.\label{alg:AVM}}

\textbf{Input:} $\lambda$, p.s.d. \& iso. $K\left(.,.\right)=\transp{\Phi\left(\cdot\right)}\Phi\left(\cdot\right)$,
$\delta$-coverage $\mathcal{P}=(P_{i})_{i\in I}$\hspace*{\fill} 
\begin{algor}[1]
\item [{{*}}] $\bw_{1}=0$
\item [{for}] $t=1,\dots,T$ 
\item [{{*}}] Receive $\left(x_{t},y_{t}\right)$\hspace*{\fill} //$\left(x_{t},y_{t}\right)\sim\mathbb{P}_{\mathcal{X},\mathcal{Y}}$
\textit{or} $\ensuremath{\mathbb{P}_{N}}$
\item [{{*}}] $\eta_{t}=\frac{1}{\lambda t}$
\item [{{*}}] $l^{'}\left(\bw_{t};x_{t},y_{t}\right)=\alpha_{t}\Phi\left(x_{t}\right)$\hspace*{\fill}
//cf. Section \ref{sec:Loss-Function}
\item [{{*}}] Sample a Bernoulli random variable $Z_{t}$
\item [{if}] $Z_{t}=1$
\item [{{*}}] Find $i_{t}\in I$ such that $x_{t}\in P_{i_{t}}$
\item [{{*}}] $h_{t}=\lambda\bw_{t}+\alpha_{t}\Phi\left(c_{i_{t}}\right)$
\hspace*{\fill}//\textit{do approximation}
\item [{else}]~
\item [{{*}}] $h_{t}=\lambda\bw_{t}+\alpha_{t}\Phi\left(x_{t}\right)$
\item [{endif}]~
\item [{if}] $\ell_{2}\,\text{loss is used}\,\text{\textbf{and}}\,\lambda\leq1$
\item [{{*}}] $\bw_{t+1}=\prod_{\mathcal{B}\left(\bzero,y_{\text{max}}\lambda^{-1/2}\right)}\left(\bw_{t}-\eta_{t}h_{t}\right)$
\item [{else}]~
\item [{{*}}] $\bw_{t+1}=\bw_{t}-\eta_{t}h_{t}$
\item [{endif}]~
\item [{endfor}]~
\end{algor}
\textbf{Output:} $\overline{\bw}_{T}=\frac{\sum_{t=1}^{T}\bw_{t}}{T}$
or $\bw_{T+1}$
\end{algorithm}

In what follows, we present the theoretical results for our proposed
AVM including the convergence analysis for a general convex or smooth
loss function and the upper bound of the model size under the assumption
that the incoming instances are drawn from an arbitrary distribution
and arrive in a random order. 

\subsubsection{Analysis for Generic Convex Loss Function}

We start with the theoretical analysis for Algorithm \ref{alg:AVM}.
The \textit{decision of approximation }(i.e., the random variable
$Z_{t}$) could be statistically independent or dependent with the
current model parameter $\bw_{t}$ and the incoming instance $\left(x_{t},y_{t}\right)$.
For example, one can propose an algorithm in which the \textit{decision
of approximation} is performed iff the confidence level of the incoming
instance w.r.t the current model is greater than $1$, i.e., $y_{t}\transp{\bw_{t}}\Phi\left(x_{t}\right)\geq1$.
We shall develop our theory to take into account all possible cases. 

Theorem \ref{thm:regret1} below establishes an upper bound on the
regret under the possible assumptions of the statistical relationship
among the decision of approximation, the data distribution, and the
current model. Based on Theorem \ref{thm:regret1}, in Theorem \ref{thm:wT_confidence}
we further establish an inequality for the error incurred by a single-point
output with a high confidence level.
\begin{thm}
\label{thm:regret1}Consider the running of Algorithm \ref{alg:AVM}
where $\left(x_{t},y_{t}\right)$ is uniformly sampled from the training
set $\mathcal{D}$ or the joint distribution $\mathbb{P}_{\mathcal{X},\mathcal{Y}}$,
the following statements hold

i) If $Z_{t}$ and $\bw_{t}$ are independent for all $t$ (i.e.,
the decision of approximation only depends on the data distribution)
then
\[
\mathbb{E}\left[f\left(\overline{\bw}_{T}\right)-f\left(\bw^{*}\right)\right]\leq\frac{H\left(\log\left(T\right)+1\right)}{2\lambda T}+\frac{\delta_{\Phi}M^{1/2}W^{1/2}}{T}\sum_{t=1}^{T}\mathbb{P}\left(Z_{t}=1\right)^{1/2}
\]
where $H,\,M,\,W$ are positive constants.

ii) If $Z_{t}$ is independent with both $\left(x_{t},y_{t}\right)$
and $\bw_{t}$ for all $t$ (i.e., the decision of approximation is
independent with the current hyperplane and the data distribution)
then
\[
\mathbb{E}\left[f\left(\overline{\bw}_{T}\right)-f\left(\bw^{*}\right)\right]\leq\frac{H\left(\log\left(T\right)+1\right)}{2\lambda T}+\frac{\delta_{\Phi}M^{1/2}W^{1/2}}{T}\sum_{t=1}^{T}\mathbb{P}\left(Z_{t}=1\right)
\]

iii) In general, we always have
\[
\mathbb{E}\left[f\left(\overline{\bw}_{T}\right)-f\left(\bw^{*}\right)\right]\leq\frac{H\left(\log\left(T\right)+1\right)}{2\lambda T}+\delta_{\Phi}M^{1/2}W^{1/2}
\]
\end{thm}

\begin{rem}
Theorem \ref{thm:regret1} consists of the standard convergence analysis.
In particular, if the approximation procedure is never performed,
i.e., $\mathbb{P}\left(Z_{t}=1\right)=0,\,\forall t$, we have the
regret bound $\mathbb{E}\left[f\left(\overline{\bw}_{T}\right)-f\left(\bw^{*}\right)\right]\leq\frac{H\left(\log\left(T\right)+1\right)}{2\lambda T}$.
\end{rem}

\begin{rem}
Theorem \ref{thm:regret1} further indicates that there exists an
error gap between the optimal and the approximate solutions. When
$\delta$ decreases to $0$, this gap also decreases to $0$. Specifically,
when $\delta=0$ (so does $\delta_{\Phi}$), any incoming instance
is approximated by itself and consequently, the gap is exactly $0$.
\end{rem}
\begin{thm}
\label{thm:wT_confidence}Let us define the gap by $d_{T}$, which
is $\frac{\delta_{\Phi}M^{1/2}W^{1/2}}{T}\sum_{t=1}^{T}\mathbb{P}\left(Z_{t}=1\right)^{1/2}$(if
$Z_{t}$ is independent with $\bw_{t}$), $\frac{\delta_{\Phi}M^{1/2}W^{1/2}}{T}\sum_{t=1}^{T}\mathbb{P}\left(Z_{t}=1\right)$
(if $Z_{t}$ is independent with $\left(x_{t},y_{t}\right)$ and $\bw_{t}$),
or $\delta_{\Phi}M^{1/2}W^{1/2}$. Let $r$ be any number randomly
picked from $\left\{ 1,2,\ldots,T\right\} $. With the probability
at least $\left(1-\delta\right)$, the following statement holds
\[
f\left(\bw_{r}\right)-f\left(\bw^{*}\right)\leq\frac{H\left(\log\left(T\right)+1\right)}{2\lambda T}+d_{T}+\Delta_{T}\sqrt{\frac{1}{2}\log\frac{1}{\delta}}
\]
where $\Delta_{T}=\underset{1\leq t\leq T}{\max}\left(f\left(\bw_{t}\right)-f\left(\bw^{*}\right)\right)$.
\end{thm}
We now present the convergence analysis for the case when we output
the $\alpha$-suffix average result as proposed in \citep{RakhlinSS12}.
With $0<\alpha<1$, let us denote
\[
\overline{\bw}_{T}^{\alpha}=\frac{1}{\alpha T}\sum_{t=\left(1-\alpha\right)T+1}^{T}\bw_{t}
\]

\noindent where we assume that the fractional indices are rounded
to their ceiling values.

Theorem \ref{thm:regret_alpha} establishes an upper bound on the
regret for the $\alpha$-suffix average case, followed by Theorem
\ref{thm:walphaT_confidence} which establishes an inequality for
the error incurred by a $\alpha$-suffix average output with a high
confidence level.
\begin{thm}
\label{thm:regret_alpha}Consider the running of Algorithm \ref{alg:AVM}
where $\left(x_{t},y_{t}\right)$ is uniformly sampled from the training
set $\mathcal{D}$ or the joint distribution $\mathbb{P}_{\mathcal{X},\mathcal{Y}}$,
the following statements hold

i) If $Z_{t}$ and $\bw_{t}$ are independent for all $t$ (i.e.,
the decision of approximation only depends on the data distribution)
then
\[
\mathbb{E}\left[f\left(\overline{\bw}_{T}^{\alpha}\right)-f\left(\bw^{*}\right)\right]\leq\frac{\lambda\left(1-\alpha\right)}{2\alpha}W_{T}^{\alpha}+\frac{\delta_{\Phi}M^{1/2}W^{1/2}}{\alpha T}\sum_{t=\left(1-\alpha\right)T+1}^{T}\mathbb{P}\left(Z_{t}=1\right)^{1/2}+\frac{H\log\left(1/\left(1-\alpha\right)\right)}{2\lambda\alpha T}
\]
where $H,M,W$ are positive constants and $W_{T}^{\alpha}=\mathbb{E}\left[\norm{\bw_{\left(1-\alpha\right)T+1}-\bw^{*}}^{2}\right]$.

ii) If $Z_{t}$ is independent with both $\left(x_{t},y_{t}\right)$
and $\bw_{t}$ for all $t$ (i.e., the decision of approximation is
independent with the current hyperplane and the data distribution)
then
\[
\mathbb{E}\left[f\left(\overline{\bw}_{T}^{\alpha}\right)-f\left(\bw^{*}\right)\right]\leq\frac{\lambda\left(1-\alpha\right)}{2\alpha}W_{T}^{\alpha}+\frac{\delta_{\Phi}M^{1/2}W^{1/2}}{\alpha T}\sum_{t=\left(1-\alpha\right)T+1}^{T}\mathbb{P}\left(Z_{t}=1\right)+\frac{H\log\left(1/\left(1-\alpha\right)\right)}{2\lambda\alpha T}
\]

iii) In general, we always have
\[
\mathbb{E}\left[f\left(\overline{\bw}_{T}^{\alpha}\right)-f\left(\bw^{*}\right)\right]\leq\frac{\lambda\left(1-\alpha\right)}{2\alpha}W_{T}^{\alpha}+\delta_{\Phi}M^{1/2}W^{1/2}+\frac{H\log\left(1/\left(1-\alpha\right)\right)}{2\lambda\alpha T}
\]
\end{thm}

\begin{thm}
\label{thm:walphaT_confidence}Let us once again define the induced
gap by $d_{T}$, which is respectively $\frac{\lambda\left(1-\alpha\right)}{2\alpha}W_{T}^{\alpha}+\frac{\delta_{\Phi}M^{1/2}W^{1/2}}{\alpha T}\sum_{t=\left(1-\alpha\right)T+1}^{T}\mathbb{P}\left(Z_{t}=1\right)^{1/2}$(if
$Z_{t}$ is independent with $\bw_{t}$), $\frac{\lambda\left(1-\alpha\right)}{2\alpha}W_{T}^{\alpha}+\frac{\delta_{\Phi}M^{1/2}W^{1/2}}{\alpha T}\sum_{t=\left(1-\alpha\right)T+1}^{T}\mathbb{P}\left(Z_{t}=1\right)$
(if $Z_{t}$ is independent with $\left(x_{t},y_{t}\right)$ and $\bw_{t}$),
or $\frac{\lambda\left(1-\alpha\right)}{2\alpha}W_{T}^{\alpha}+\delta_{\Phi}M^{1/2}W^{1/2}$.
Let $r$ be any number randomly picked from $\left\{ \left(1-\alpha\right)T+1,2,\ldots,T\right\} $.
With the probability at least $\left(1-\delta\right)$, the following
statement holds
\[
f\left(\bw_{r}\right)-f\left(\bw^{*}\right)\leq\frac{H\log\left(1/\left(1-\alpha\right)\right)}{2\lambda\alpha T}+d_{T}+\Delta_{T}^{\alpha}\sqrt{\frac{1}{2}\log\frac{1}{\delta}}
\]
where $\Delta_{T}^{\alpha}=\underset{\left(1-\alpha\right)T+1\leq t\leq T}{\max}\left(f\left(\bw_{t}\right)-f\left(\bw^{*}\right)\right)$.
\end{thm}
\begin{rem}
Theorems \ref{thm:wT_confidence} and \ref{thm:walphaT_confidence}
concern with the theoretical warranty if rendering any single-point
output $\bw_{r}$ rather than the average outputs. The upper bound
gained in Theorem \ref{thm:walphaT_confidence} is tighter than that
gained in Theorem \ref{thm:wT_confidence} in the sense that the quantity
$\frac{H\log\left(1/\left(1-\alpha\right)\right)}{2\lambda\alpha T}+\Delta_{T}^{\alpha}\sqrt{\frac{1}{2}\log\frac{1}{\delta}}$
decreases faster and may decrease to $0$ when $T\goto+\infty$ given
a confidence level $1-\delta$.
\end{rem}

\subsubsection{Analysis for Smooth Loss Function}
\begin{defn}
A loss function $l\left(\bw;x,y\right)$ is said to be $\mu$-strongly
smooth w.r.t a norm $\norm .$ iff for all $\boldsymbol{u},\boldsymbol{v}$
and $\left(x,y\right)$ the following condition satisfies
\[
l\left(\boldsymbol{v};x,y\right)\leq l\left(\boldsymbol{u};x,y\right)+\transp{l^{'}\left(\boldsymbol{u};x,y\right)}\left(\boldsymbol{v}-\boldsymbol{u}\right)+\frac{\mu}{2}\norm{\boldsymbol{v}-\boldsymbol{u}}^{2}
\]
Another equivalent definition of $\mu$-strongly smooth function is
\[
\norm{l^{'}\left(\boldsymbol{u};x,y\right)-l^{'}\left(\boldsymbol{v};x,y\right)}_{*}\leq\mu\norm{\boldsymbol{v}-\boldsymbol{u}}
\]
where $\norm ._{*}$ is used to represent the dual norm of the norm
$\norm .$.

\noindent It is well-known that 
\end{defn}
\begin{itemize}
\item $\ell_{2}$ loss is $1$-strongly smooth w.r.t $\norm ._{2}$.
\item Logistic loss is $1$-strongly smooth w.r.t $\norm ._{2}$.
\item $\tau$-smooth Hinge loss \citep{DBLP:journals/jmlr/Shalev-Shwartz013}
is $\frac{1}{\tau}$-strongly smooth w.r.t $\norm ._{2}$.
\end{itemize}
\begin{thm}
\label{thm:smooth_loss}Assume that $\ell_{2}$, Logistic, or $\tau$-smooth
Hinge loss is used, let us denote $L=\frac{\lambda}{2}+1$, $\frac{\lambda}{2}+1$,
or $\frac{\lambda}{2}+\tau^{-1}$ respectively. Let us define the
gap by $d_{T}$ as in Theorem \ref{thm:walphaT_confidence}. Let $r$
be any number randomly picked from $\left\{ \left(1-\alpha\right)T+1,2,\ldots,T\right\} $.
With the probability at least $\left(1-\delta\right)$, the following
statement holds
\[
f\left(\bw_{r}\right)-f\left(\bw^{*}\right)\leq\frac{H\log\left(1/\left(1-\alpha\right)\right)}{2\lambda\alpha T}+d_{T}+\frac{LM_{T}^{\alpha}}{2}\sqrt{\frac{1}{2}\log\frac{1}{\delta}}
\]
where $M_{T}^{\alpha}=\underset{\left(1-\alpha\right)T+1\leq t\leq T}{\max}\norm{\bw_{t}-\bw^{*}}$.
\end{thm}
\begin{rem}
Theorem \ref{thm:smooth_loss} extends Theorem \ref{thm:walphaT_confidence}
for the case of smooth loss function. This allows the gap $\frac{H\log\left(1/\left(1-\alpha\right)\right)}{2\lambda\alpha T}+\frac{LM_{T}^{\alpha}}{2}\sqrt{\frac{1}{2}\log\frac{1}{\delta}}$
to be quantified more precisely regarding the discrepancy in the model
itself rather than that in the objective function. The gap $\frac{H\log\left(1/\left(1-\alpha\right)\right)}{2\lambda\alpha T}+\frac{LM_{T}^{\alpha}}{2}\sqrt{\frac{1}{2}\log\frac{1}{\delta}}$
could possibly decrease rapidly when $T$ approaches $+\infty$.

\begin{table}[H]
\begin{centering}
\begin{tabular}{|c|c|c|}
\hline 
\textbf{Algorithms} & \textbf{Regret} & \textbf{Budget}\tabularnewline
\hline 
Forgetron \citep{DekelSS05} & \emph{NA} & \emph{MB}\tabularnewline
PA-I, II \citep{Crammer06onlinepassive-aggressive} & \emph{NA} & \emph{NB}\tabularnewline
Randomized Budget Perceptron \citep{CavallantiCG07} & \emph{NA} & \emph{NB}\tabularnewline
Projection \citep{Orabona2009} & \emph{NA} & \emph{AB}\tabularnewline
Kernelized Pegasos \citep{shalev2011pegasos} & $\text{O}\left(\frac{\log\left(T\right)}{T}\right)$ & \emph{NB}\tabularnewline
Budgeted SGD \citep{wang2012} & $\text{O}\left(\frac{\log\left(T\right)}{T}\right)$ & \emph{MB}\tabularnewline
Fourier OGD \citep{Lu_2015large} & $\text{O}\left(\frac{1}{\sqrt{T}}\right)$ & \emph{MB}\tabularnewline
Nystrom OGD \citep{Lu_2015large} & $\text{O}\left(\frac{1}{\sqrt{T}}\right)$ & \emph{MB}\tabularnewline
AVM (average output)  & $\text{O}\left(\frac{\log\left(T\right)}{T}\right)$ & \emph{AB}\tabularnewline
AVM ($\alpha$-suffix average output)  & $\text{O}\left(\frac{1}{T}\right)$ & \emph{AB}\tabularnewline
\hline 
\end{tabular}
\par\end{centering}
\caption{Comparison on the regret bounds and the budget sizes of the kernel
online algorithms. On the column of budget size, NB stands for \emph{Not
Bound} (i.e., the model size is not bounded and learning method is
vulnerable to the curse of kernelization), MB stands for \emph{Manual
Bound} (i.e., the model size is manually bounded by a predefined budget),
and AB is an abbreviation of \emph{Automatic Bound} (i.e., the model
size is automatically bounded and this model size is automatically
inferred). \label{tab:comparison_rate}}
\end{table}
\end{rem}
To end this section, we present the regret bound and the obtained
budget size for our AVM(s) together with those of algorithms listed
in Table \ref{tab:comparison_rate}. We note that some early works
on online kernel learning mainly focused on the mistake rate and did
not present any theoretical results regarding the regret bounds.

\subsubsection{Upper Bound of Model Size\label{subsec:avm_upperbound}}

In what follows, we present the theoretical results regarding the
model size and sparsity level of our proposed AVM. Theorem \ref{thm:ModelSize}
shows that AVM offers a high level of freedom to control the model
size. Especially, if we use the always-on setting (i.e., $\mathbb{P}\left(Z_{t}=1\right)=1,\,\forall t$),
the model size is bounded regardless of the data distribution and
data arrival order.
\begin{thm}
\label{thm:ModelSize}Let us denote $\mathbb{P}\left(Z_{t}=1\right)=p_{t}$,
$\mathbb{P}\left(Z_{t}=0\right)=q_{t}$, and the number of cells generated
after the iteration $t$ by $M_{t}$. If we define the model size,
i.e., the size of support set, after the iteration $t$ by $S_{t}$,
the following statement holds
\[
\mathbb{E}\left[S_{T}\right]\leq\sum_{t=1}^{T}q_{t}+\sum_{t=1}^{T}p_{t}\mathbb{E}\left[M_{t}-M_{t-1}\right]\leq\sum_{t=1}^{T}q_{t}+\mathbb{E}\left[M_{T}\right]
\]
Specially, if we use some specific settings for $p_{t}$, we can bound
the model size $\mathbb{E}\left[S_{t}\right]$ accordingly as follows

\noindent i) If $p_{t}=1,\,\forall t$ then $\mathbb{E}\left[S_{T}\right]\leq\mathbb{E}\left[M_{T}\right]\leq|\mathcal{P}|$,
where $|\mathcal{P}|$ specifies the size of the partition $\mathcal{P}$,
i.e., its number of cells.

\noindent ii) If $p_{t}=\max\left(0,1-\frac{\beta}{t}\right),\,\forall t$
then $\mathbb{E}\left[S_{T}\right]\leq\beta\left(\log\left(T\right)+1\right)+\mathbb{E}\left[M_{T}\right]$.

\noindent iii) If $p_{t}=\max\left(0,1-\frac{\beta}{t^{\rho}}\right),\,\forall t$,
where $0<\rho<1$, then $\mathbb{E}\left[S_{T}\right]\leq\frac{\beta T^{1-\rho}}{1-\rho}+\mathbb{E}\left[M_{T}\right]$.

\noindent iv) If $p_{t}=\max\left(0,1-\frac{\beta}{t^{\rho}}\right),\,\forall t$,
where $\rho>1$, then $\mathbb{E}\left[S_{T}\right]\leq\beta\zeta\left(\rho\right)+\mathbb{E}\left[M_{T}\right]\leq\beta\zeta\left(\rho\right)+|\mathcal{P}|$,
where $\zeta\left(.\right)$ is $\zeta$- Riemann function defined
by the integral $\zeta\left(s\right)=\frac{1}{\Gamma\left(s\right)}\int_{0}^{+\infty}\frac{t^{s-1}}{e^{s}-1}dt$.
\end{thm}
\begin{rem}
We use two parameters $\beta$ and $\rho$ to flexibly control the
rate of approximation $p_{t}$. It is evident that when $\beta$ increases,
the rate of approximation decreases and consequently the model size
and accuracy increase. On the other hand, when $\rho$ increases,
the rate of approximation increases as well and it follows that the
model size and accuracy decreases. We conducted experiment to investigate
how the variation of these two parameters influence the model size
and accuracy (cf. Section \ref{subsec:exp_hyperparam}).
\end{rem}

\begin{rem}
\label{rem:bound_model}The items i) and iv) in Theorem \ref{thm:ModelSize}
indicate that if $\mathbb{P}\left(Z_{t}=1\right)=p_{t}=\max\left(0,1-\frac{\beta}{t^{\rho}}\right)$,
where $\rho>1$ or $\rho=+\infty$, then the model size is bounded
by $\beta\zeta\left(\rho\right)+|\mathcal{P}|$ (by convention we
define $\zeta\left(+\infty\right)=0$). In fact, the tight upper bound
is $\beta\zeta\left(\rho\right)+\mathbb{E}\left[M_{T}\right]$, where
$M_{T}$ is the number of unique cells used so far. It is empirically
proven that $M_{T}$ could be very small comparing with $T$ and $|\mathcal{P}|$.
In addition, since all support sets of $\bw_{t}\,(1\leq t\leq T)$
are all lain in the core set, if we output the average $\overline{\bw}_{T}=\frac{\sum_{t=1}^{T}\bw_{t}}{T}$
or $\alpha$-suffix average $\overline{\bw}_{T}^{\alpha}=\frac{1}{\alpha T}\sum_{t=\left(1-\alpha\right)T+1}^{T}\bw_{T}$,
the model size is still bounded.
\end{rem}

\begin{rem}
The items ii) and iii) in Theorem \ref{thm:ModelSize} indicate that
if $\mathbb{P}\left(Z_{t}=1\right)=p_{t}=\max\left(0,1-\frac{\beta}{t^{\rho}}\right)$,
where $0<\rho\leq1$ then although the model size is not bounded,
it would slowly increase comparing with $T$, i.e., $\log\left(T\right)$
or $T^{1-\rho}$ when $\rho$ is around $1$.
\end{rem}

\subsection{Construction of $\delta$-Coverage \label{subsec:Construction-of--Coverage}}

In this section, we return to the construction of $\delta$-coverage
defined in Section \ref{subsec:coverage} and present two methods
to construct a finite $\delta$-coverage. The first method employs
hypersphere cells (cf. Algorithm \ref{alg:sphere_coverage}) whereas
the second method utilizes the hyperrectangle cells (cf. Algorithm
\ref{alg:rec_coverage}). In these two methods, the cells in coverage
are constructed on the fly when the incoming instances arrive. Both
are theoretically proven to be a finite coverage.

\begin{algorithm}[h]
\caption{Constructing hypersphere $\delta$-coverage.\label{alg:sphere_coverage}}

\begin{algor}[1]
\item [{{*}}] $\mathcal{P}=\emptyset$
\item [{{*}}] $n=0$
\item [{for}] $t=1,2,\ldots$
\item [{{*}}] Receive $\left(x_{t},y_{t}\right)$
\item [{{*}}] $i_{t}=\text{argmin}{}_{k\leq n}\norm{x_{t}-c_{k}}$
\item [{if}] $\norm{x_{t}-c_{i_{t}}}\geq\delta/2$
\item [{{*}}] $n=n+1$
\item [{{*}}] $c_{n}=x_{t}$
\item [{{*}}] $i_{t}=n$
\item [{{*}}] $\mathcal{P}=\mathcal{P}\cup\left[\mathcal{B}\left(c_{n},\delta/2\right)\right]$
\item [{endif}]~
\item [{endfor}]~
\end{algor}
\end{algorithm}
\begin{algorithm}[H]
\caption{Constructing hyperrectangle $\delta$-coverage.\label{alg:rec_coverage}}

\begin{algor}[1]
\item [{{*}}] $\mathcal{P}=\emptyset$
\item [{{*}}] $a=\delta/\sqrt{d}$
\item [{{*}}] $n=0$
\item [{for}] $t=1,2,\ldots$
\item [{{*}}] Receive $\left(x_{t},y_{t}\right)$
\item [{{*}}] $i_{t}=0$
\item [{for}] $i=1$ to $n$
\item [{if}] $\norm{x_{t}-c_{i}}_{\infty}<a$
\item [{{*}}] $i_{t}=i$
\item [{{*}}] \textbf{break}
\item [{endif}]~
\item [{endfor}]~
\item [{if}] $i_{t}=0$
\item [{{*}}] $n=n+1$
\item [{{*}}] $c_{n}=x_{t}$
\item [{{*}}] $i_{t}=n$
\item [{{*}}] $\mathcal{P}=\mathcal{P}\cup\left[\mathcal{R}\left(c_{n},a\right)\right]$
\item [{endif}]~
\item [{endfor}]~
\end{algor}
\end{algorithm}

Algorithm \ref{alg:sphere_coverage} employs a collection of open
hypersphere cell $\mathcal{B}\left(c,R\right)$, which is defined
as $\mathcal{B}\left(c,R\right)=\left\{ x\in\mathbb{R}^{d}:\norm{x-c}<R\right\} $,
to cover the data domain. Similar to Algorithm \ref{alg:sphere_coverage},
Algorithm \ref{alg:rec_coverage} uses a collection of open hyperrectangle
$\mathcal{R}\left(c,a\right)$, which is given by $\mathcal{R}\left(c,a\right)=\left\{ x\in\mathbb{R}^{d}:\norm{x-c}_{\infty}<a\right\} $,
to cover the data domain.

Both Algorithms \ref{alg:sphere_coverage} and \ref{alg:rec_coverage}
are constructed in the common spirit: if the incoming instance $\left(x_{t},y_{t}\right)$
is outside all current cells, a new cell whose centre or vertex is
this instance is generated. It is noteworthy that the variable $i_{t}$
in these two algorithms specifies the cell that contains the new incoming
instance and is the same as itself in Algorithm \ref{alg:AVM}.

Theorem \ref{thm:finite_coverage} establishes that regardless of
the data distribution and data arrival order, Algorithms \ref{alg:sphere_coverage}
and \ref{alg:rec_coverage} always generate a finite $\delta$-coverage
which implies a bound on the model size of AVM. It is noteworthy at
this point that in some scenarios of data arrival, Algorithms \ref{alg:sphere_coverage}
and \ref{alg:rec_coverage} might not generate a coverage for the
entire space $\mathcal{X}$. However, since the generated sequence
$\left\{ x_{t}\right\} _{t}$ cannot be outside the set $\cup_{i}\,\mathcal{B}\left(c_{i},\delta\right)$
and $\cup_{i}\,\mathcal{R}\left(c_{i},\delta\right)$, without loss
of generality we can restrict $\mathcal{X}$ to $\cup_{i}\,\mathcal{B}\left(c_{i},\delta\right)$
or $\cup_{i}\,\mathcal{R}\left(c_{i},\delta\right)$ by assuming that
$\mathcal{X=}\cup_{i}\,\mathcal{B}\left(c_{i},\delta\right)$ or $\mathcal{X=}\cup_{i}\,\mathcal{R}\left(c_{i},\delta\right)$.
\begin{thm}
\label{thm:finite_coverage}Let us consider the coverages formed by
the running of Algorithms \ref{alg:sphere_coverage} and \ref{alg:rec_coverage}.
If the data domain $\mathcal{X}$ is compact (i.e., close and bounded)
then these coverages are all finite $\delta$-coverages whose sizes
are all dependent on the data domain $\mathcal{X}$ and independent
with the sequence of incoming data instances $\left(x_{t},y_{t}\right)$
received.
\end{thm}
\begin{rem}
Theorem \ref{thm:finite_coverage} also reveals that regardless of
the data arrival order, the model size of AVM is always bounded (cf.
Remark \ref{rem:bound_model}). Referring to the work of \citep{Cucker02onthe},
it is known that this model size cannot exceed $\left(\frac{4D\left(\mathcal{X}\right)}{\delta}\right)^{d}$.
However with many possible data arrival orders, the number of active
cells or the model size of AVM is significantly smaller than the aforementioned
theoretical bound.
\end{rem}

\subsection{Complexity Analysis}

We now present the computational complexity of our AVM(s) with the
hypersphere $\delta$-coverage at the iteration $t$. The cost to
find the hypersphere cell in Step 5 of Algorithm \ref{alg:AVM} is
$\text{O}\left(d^{2}M_{t}\right)$. The cost to calculate $\alpha_{t}$
in Step 6 of Algorithm \ref{alg:AVM} is $\text{O}\left(S_{t}\right)$
if we consider the kernel operation as a unit operation. If $\ell_{2}$
loss is used and $\lambda\leq1$, we need to do a projection onto
the hypersphere $\mathcal{B}\left(\bzero,y_{\text{max}}\lambda^{-1/2}\right)$
which requires the evaluation of the length of the vector $\bw_{t}-\eta_{t}h_{t}$
(i.e., $\norm{\bw_{t}-\eta_{t}h_{t}}$) which costs $S_{t}$ unit
operations using incremental implementation. Therefore, the computational
operation at the iteration $t$ of AVM(s) is either $\text{O}\left(d^{2}M_{t}+S_{t}\right)=\text{O}\left(\left(d^{2}+1\right)S_{t}\right)$
or $\text{O}\left(d^{2}M_{t}+S_{t}+S_{t}\right)=\text{O}\left(\left(d^{2}+2\right)S_{t}\right)$
(since $M_{t}\leq S_{t}$).

\section{Suitability of Loss Functions\label{sec:Loss-Function}}

We introduce six types of loss functions that can be used in our proposed
algorithm, namely Hinge, Logistic, $\ell_{2}$, $\ell_{1}$, $\varepsilon-$insensitive,
and $\tau$-smooth Hinge. We verify that these loss functions satisfying
the necessary condition, that is, $\norm{l^{'}\left(\bw;x,y\right)}\leq A\norm{\bw}^{1/2}+B$
for some appropriate positive numbers $A,B$ (this is required for
our problem formulation presented in Section \ref{sec:Problem-Setting}). 

For comprehensibility, without loss of generality, we assume that
$\norm{\Phi\left(x\right)}=K\left(x,x\right)^{1/2}=1,\,\forall x\in\mathcal{X}$.
At the outset of this section, it is noteworthy that for classification
task (i.e., Hinge, Logistic, and $\tau$-smooth Hinge cases), the
label $y$ is either $-1$ or $1$ which instantly implies $|y|=y^{2}=1$.
\begin{itemize}
\item \textbf{Hinge loss}
\begin{align*}
l\left(\bw;x,y\right) & =\max\left\{ 0,1-y\transp{\bw}\Phi\left(x\right)\right\} \\
l^{'}\left(\bw;x,y\right) & =-\mathbb{I}_{\left\{ y\transp{\bw}\Phi\left(x\right)\leq1\right\} }y\Phi\left(x\right)
\end{align*}
where $\mathbb{I}_{S}$ is the indicator function which renders $1$
if the logical statement $S$ is true and $0$ otherwise.

Therefore, by choosing $A=0,\,B=1$ we have 
\[
\norm{l^{'}\left(\bw;x,y\right)}\leq\norm{\Phi\left(x\right)}\leq1=A\norm{\bw}^{1/2}+B
\]

\item \textbf{$\ell_{2}$ loss}

In this case, at the outset we cannot verify that $\norm{l^{'}\left(\bw;x,y\right)}\leq A\norm{\bw}^{1/2}+B$
for all $\bw,x,y$. However, to support the proposed theory, we only
need to check that $\norm{l^{'}\left(\bw_{t};x,y\right)}\leq A\norm{\bw_{t}}^{1/2}+B$
for all $t\geq1$. We derive as follows
\begin{align*}
l\left(\bw;x,y\right) & =\frac{1}{2}\left(y-\transp{\bw}\Phi\left(x\right)\right)^{2}\\
l^{'}\left(\bw;x,y\right) & =\left(\transp{\bw}\Phi\left(x\right)-y\right)\Phi\left(x\right)
\end{align*}
\begin{flalign*}
\norm{l^{'}\left(\bw_{t};x,y\right)} & =\vert\transp{\bw_{t}}\Phi\left(x\right)+y\vert\norm{\Phi\left(x\right)}\leq\vert\transp{\bw_{t}}\Phi\left(x\right)\vert+y_{\text{max}}\\
 & \leq\norm{\Phi\left(x\right)}\norm{\bw_{t}}+y_{\text{max}}\leq A\norm{\bw_{t}}^{1/2}+B
\end{flalign*}
where $B=y_{\text{max}}$ and $A=\begin{cases}
y_{\text{max}}^{1/2}\lambda^{-1/4} & \text{if \ensuremath{\lambda\leq1}}\\
y_{\text{max}}^{1/2}\left(\lambda-1\right)^{-1/2} & \text{otherwise}
\end{cases}$.

Here we note that we make use of the fact that $\norm{\bw_{t}}\leq y_{\text{max}}\left(\lambda-1\right)^{-1}$
if $\lambda>1$ (cf. Theorem \ref{thm:bound_wT} in Appendix \ref{sec:appxC})
and $\norm{\bw_{t}}\leq y_{\text{max}}\lambda^{-1/2}$ otherwise (cf.
Line 12 in Algorithm \ref{alg:AVM} ).
\item \textbf{$\ell_{1}$ loss}
\begin{align*}
l\left(\bw;x,y\right) & =\vert y-\transp{\bw}\Phi\left(x\right)\vert\\
l^{'}\left(\bw;x,y\right) & =\text{sign}\left(\transp{\bw}\Phi\left(x\right)-y\right)\Phi\left(x\right)
\end{align*}

Therefore, by choosing $A=0,\,B=1$ we have
\[
\norm{l^{'}\left(\bw;x,y\right)}=\norm{\Phi\left(x\right)}\leq1=A\norm{\bw}^{1/2}+B
\]

\item \textbf{Logistic loss}
\begin{align*}
l\left(\bw;x,y\right) & =\log\left(1+\exp\left(-y\transp{\bw}\Phi\left(x\right)\right)\right)\\
l^{'}\left(\bw;x,y\right) & =\frac{-y\exp\left(-y\transp{\bw}\Phi\left(x\right)\right)\Phi\left(x\right)}{\exp\left(-y\transp{\bw}\Phi\left(x\right)\right)+1}
\end{align*}

Therefore, by choosing $A=0,\,B=1$ we have
\[
\norm{l^{'}\left(\bw;x,y\right)}<\norm{\Phi\left(x\right)}\leq1=A\norm{\bw}^{1/2}+B
\]

\item $\boldsymbol{\varepsilon}$-\textbf{insensitive loss}
\begin{align*}
l\left(\bw;x,y\right) & =\max\left\{ 0,\vert y-\transp{\bw}\Phi\left(x\right)\vert-\varepsilon\right\} \\
l^{'}\left(\bw;x,y\right) & =\mathbb{I}_{\left\{ \vert y-\transp{\bw}\Phi\left(x\right)\vert>\varepsilon\right\} }\text{sign}\left(\transp{\bw}\Phi\left(x\right)-y\right)\Phi\left(x\right)
\end{align*}

Therefore, by choosing $A=0,\,B=1$ we have
\[
\norm{l^{'}\left(\bw;x,y\right)}\leq\norm{\Phi\left(x\right)}\leq1=A\norm{\bw}^{1/2}+B
\]

\item \textbf{$\tau$-smooth Hinge loss} \citep{DBLP:journals/jmlr/Shalev-Shwartz013}
\begin{align*}
l\left(\bw;x,y\right) & =\begin{cases}
0 & \text{if}\,\,y\transp{\bw}\Phi\left(x\right)>1\\
1-y\transp{\bw}\Phi\left(x\right)-\frac{\tau}{2} & \text{if}\,\,y\transp{\bw}\Phi\left(x\right)<1-\tau\\
\frac{1}{2\tau}\left(1-y\transp{\bw}\Phi\left(x\right)\right)^{2} & \text{otherwise}
\end{cases}\\
l^{'}\left(\bw;x,y\right) & =-\mathbb{I}_{\left\{ y\transp{\bw}\Phi\left(x\right)<1-\tau\right\} }y\Phi\left(x\right)\\
 & +\tau^{-1}\mathbb{I}_{1-\tau\leq y\transp{\bw}\Phi\left(x\right)\leq1}\left(y\transp{\bw}\Phi\left(x\right)-1\right)y\Phi\left(x\right)
\end{align*}

Therefore, by choosing $A=0,\,B=2$, we have
\begin{align*}
\norm{l^{'}\left(\bw;x,y\right)} & \leq\left|y\right|\norm{\Phi\left(x\right)}+\tau^{-1}\left|y\right|\norm{\Phi\left(x\right)}\tau\leq2\\
 & =A\norm{\bw}^{1/2}+B
\end{align*}
\end{itemize}

\section{Multiclass Setting \label{sec:Multiclass-Setting}}

In this section, we show that our proposed framework could also easily
extend to the multi-class setting. We base on the work of \citep{Crammer2002}
for multiclass classification to formulate the optimization problem
in multi-class setting as
\[
\underset{W}{\min}\left(f\left(W\right)\triangleq\frac{\lambda}{2}\norm W_{2,2}^{2}+\frac{1}{N}\sum_{i=1}^{N}l\left(\transp{\bw_{y_{i}}}\Phi\left(x_{i}\right)-\transp{\bw_{z_{i}}}\Phi\left(x_{i}\right)\right)\right)
\]
where we have defined
\begin{gather*}
z_{i}=\argmax{j\neq y_{i}}\,\transp{\bw_{j}}\Phi\left(x_{i}\right),\\
W=\left[\bw_{1},\bw_{2},\ldots,\bw_{m}\right],\,\norm W_{2,2}^{2}=\sum_{j=1}^{m}\norm{\bw_{j}}^{2},\\
l\left(a\right)=\begin{cases}
\max\left(0,1-a\right) & \text{Hinge loss}\\
\log\left(1+e^{-a}\right) & \text{Logistic loss}
\end{cases}
\end{gather*}

For the exact update, at the $t$-th iteration, we receive the instance
$\left(x_{t},y_{t}\right)$ and modify $W$ as follows
\[
\bw_{j}^{(t+1)}=\begin{cases}
\frac{t-1}{t}\bw_{j}^{(t)}-\eta_{t}l^{'}\left(a\right)\Phi\left(x_{t}\right) & \text{if \ensuremath{j=y_{t}}}\\
\frac{t-1}{t}\bw_{j}^{(t)}+\eta_{t}l^{'}\left(a\right)\Phi\left(x_{t}\right) & \text{if \ensuremath{j=z_{t}}}\\
\frac{t-1}{t}\bw_{j}^{(t)} & \text{otherwise}
\end{cases}
\]
where $a=\transp{\bw_{y_{t}}}\Phi\left(x_{t}\right)-\transp{\bw_{z_{t}}}\Phi\left(x_{t}\right)$
and $l^{'}(a)=-\mathbb{I}_{\left\{ a<1\right\} }$ or $-1/\left(1+e^{a}\right)$.

The algorithm for Approximation Vector Machine with multiclass setting
proceeds as in Algorithm \ref{alg:MAVM}. 

\begin{algorithm}[H]
\caption{Multiclass Approximation Vector Machine.\label{alg:MAVM}}

\textbf{Input:} $\lambda$, p.s.d. \& iso. kernel $K\left(.,.\right)$,
$\delta$-coverage $\mathcal{P}=(P_{i})_{i\in I}$
\begin{algor}[1]
\item [{{*}}] $W_{1}=0$
\item [{for}] $t=1,\dots,T$ 
\item [{{*}}] Receive $\left(x_{t},y_{t}\right)$\hspace*{\fill} //$\left(x_{t},y_{t}\right)\sim\mathbb{P}_{\mathcal{X},\mathcal{Y}}$
\textit{or} $\ensuremath{\mathbb{P}_{N}}$
\item [{{*}}] $a=\transp{\bw_{y_{t}}}\Phi\left(x_{t}\right)-\underset{j\neq y_{t}}{\max}\,\transp{\bw_{j}}\Phi\left(x_{t}\right)$
\item [{{*}}] $W^{\left(t+1\right)}=\frac{t-1}{t}W^{\left(t\right)}$
\item [{{*}}] Sample a Bernoulli random variable $Z_{t}$
\item [{if}] $Z_{t}=1$
\item [{{*}}] Find $i_{t}\in I$ such that $x_{t}\in P_{i_{t}}$
\item [{{*}}] $\bw_{y_{t}}^{\left(t+1\right)}=\bw_{y_{t}}^{\left(t+1\right)}-\eta_{t}l^{'}\left(a\right)\Phi\left(c_{i_{t}}\right)$
\hspace*{\fill}//\textit{do approximation}
\item [{{*}}] $\bw_{z_{t}}^{\left(t+1\right)}=\bw_{z_{t}}^{\left(t+1\right)}+\eta_{t}l^{'}\left(a\right)\Phi\left(c_{i_{t}}\right)$
\item [{else}]~
\item [{{*}}] $\bw_{y_{t}}^{\left(t+1\right)}=\bw_{y_{t}}^{\left(t+1\right)}-\eta_{t}l^{'}\left(a\right)\Phi\left(x_{t}\right)$ 
\item [{{*}}] $\bw_{z_{t}}^{\left(t+1\right)}=\bw_{z_{t}}^{\left(t+1\right)}+\eta_{t}l^{'}\left(a\right)\Phi\left(x_{t}\right)$
\item [{endif}]~
\item [{endfor}]~
\end{algor}
 \textbf{Output} $\overline{W}^{\left(T\right)}=\frac{\sum_{t=1}^{T}W^{\left(t\right)}}{T}$
or $W^{\left(t+1\right)}$

\end{algorithm}

\section{Experiments\label{sec:Experiments}}

In this section, we conduct comprehensive experiments to quantitatively
evaluate the capacity and scalability of our proposed Approximation
Vector Machine ($\model$) on classification and regression tasks
under three different settings:
\begin{itemize}
\item \emph{Batch classification}\footnote{This setting is also known as offline classification.}:
the regular binary and multiclass classification tasks that follow
a standard validation setup, wherein each dataset is partitioned into
training set and testing set. The models are trained on the training
part, and then their discriminative capabilities are verified on the
testing part using classification accuracy measure. The computational
costs are commonly measured based on the training time.
\item \emph{Online classification}: the binary and multiclass classification
tasks that follow a purely online learning setup, wherein there is
no division of training and testing sets as in batch setting. The
algorithms sequentially receive and process a single data sample turn-by-turn.
When an individual data point comes, the models perform prediction
to compute the mistake rate first, then use the feature and label
information of such data point to continue their learning procedures.
Their predictive performances and computational costs are measured
basing on the average of mistake rate and execution time, respectively,
accumulated in the learning progress on the entire dataset.
\item \emph{Online regression}: the regression task that follows the same
setting of online classification, except the predictive performances
are measured based on the regression error rate accumulated in the
learning progress on the entire dataset.
\end{itemize}
Our main goal is to examine the scalability, classification and regression
capabilities of $\model$s by directly comparing with those of several
recent state-of-the-art batch and online learning approaches using
a number of datasets with a wide range of sizes. Our models are implemented
in Python with Numpy package. The source code and experimental scripts
are published for reproducibility\footnote{\href{https://github.com/tund/avm}{https://github.com/tund/avm}}.
In what follows, we present the data statistics, experimental setup,
results and our observations.

\subsection{Data statistics and experimental setup}

We use $11$ datasets whose statistics are summarized in Table~\ref{tab:data_stat}.
The datasets are selected in a diverse array of sizes in order to
clearly expose the differences among scalable capabilities of the
models. Five of which (\emph{year}, \emph{covtype}, \emph{poker, KDDCup99},
\emph{airlines}) are large-scale datasets with hundreds of thousands
and millions of data points, whilst the rest are ordinal-size databases.
Except the \emph{airlines}, all of the datasets can be downloaded
from LIBSVM\footnote{\href{https://www.csie.ntu.edu.tw/~cjlin/libsvmtools/datasets/}{https://www.csie.ntu.edu.tw/$\sim$cjlin/libsvmtools/datasets/}}
and UCI\footnote{\href{https://archive.ics.uci.edu/ml/datasets.html}{https://archive.ics.uci.edu/ml/datasets.html}}
websites.

\begin{table}[h]
\noindent \centering{}%
\begin{tabular}{|c|r|r|r|r|c|}
\hline 
\rowcolor{header_color}\textbf{Dataset} & \textbf{\#training} & \textbf{\#testing} & \textbf{\#features} & \textbf{\#classes} & \textbf{Source}\tabularnewline
\hline 
\emph{a9a} & $32,561$ & $16,281$ & $123$ & $2$ & UCI\tabularnewline
\emph{w8a} & $49,749$ & $14,951$ & $300$ & $2$ & LIBSVM\tabularnewline
\emph{cod-rna} & $59,535$ & $271,617$ & $8$ & $2$ & LIBSVM\tabularnewline
\emph{ijcnn1} & $49,990$ & $91,701$ & $22$ & $2$ & LIBSVM\tabularnewline
\emph{covtype} & $522,911$ & $58,101$ & $54$ & $7$ & LIBSVM\tabularnewline
\emph{poker} & $25,010$ & $1,000,000$ & $10$ & $10$ & UCI\tabularnewline
\emph{KDDCup99} & $4,408,589$ & $489,842$ & $41$ & $23$ & UCI\tabularnewline
\emph{airlines} & $5,336,471$ & $592,942$ & $8$ & $2$ & ASA\tabularnewline
\hline 
\hline 
\rowcolor{header_color}\textbf{Dataset} & \textbf{\#training} & \textbf{\#testing} & \textbf{\#features} & \textbf{value} & \textbf{Source}\tabularnewline
\hline 
\emph{casp} & $45,730$ & \textendash{} & $9$ & $\left[0,1\right]$ & UCI\tabularnewline
\emph{slice} & $53,500$ & \textendash{} & $384$ & $\left[0,1\right]$ & UCI\tabularnewline
\emph{year} & $515,345$ & \textendash{} & $90$ & $\left[0,1\right]$ & UCI\tabularnewline
\emph{airlines} & $5,929,413$ & \textendash{} & $8$ & $\realset^{+}$ & ASA\tabularnewline
\hline 
\end{tabular}\caption{Data statistics. \#training: number of training samples; \#testing:
number of testing samples.\label{tab:data_stat}}
\end{table}

The airlines dataset is provided by American Statistical Association
(ASA\footnote{The data can be downloaded from \href{http://stat-computing.org/dataexpo/2009/}{http://stat-computing.org/dataexpo/2009/}.}).
The dataset contains information of all commercial flights in the
US from October 1987 to April 2008. The aim is to predict whether
a flight will be delayed or not and how long in minutes the flight
will be delayed in terms of departure time. The departure delay time
is provided in the flight database. A flight is considered \emph{delayed}
if its delay time is above $15$ minutes, and \emph{non-delayed} otherwise.
The average delay of a flight in 2008 was of $56.3$ minutes. Following
the procedure of \citep{hensman_etal_uai13_Gaussian}, we further
process the data in two steps. First, we join the data with the information
of individual planes basing on their tail numbers in order to obtain
the manufacture year. This additional information is provided as a
supplemental data source on ASA website. We then extract $8$ features
of many available fields: the age of the aircraft (computed based
on the manufacture year), journey distance, airtime, scheduled departure
time, scheduled arrival time, month, day of week and month. All features
are normalized into the range $\left[0,1\right]$.

In batch classification experiments, we follow the original divisions
of training and testing sets in LIBSVM and UCI sites wherever available.
For \emph{KDDCup99}, \emph{covtype} and \emph{airlines} datasets,
we split the data into $90\%$ for training and $10\%$ for testing.
In online classification and regression tasks, we either use the entire
datasets or concatenate training and testing parts into one. The online
learning algorithms are then trained in a single pass through the
data. In both batch and online settings, for each dataset, the models
perform $10$ runs on different random permutations of the training
data samples. Their prediction results and time costs are then reported
by taking the average with the standard deviation of the results over
these runs.

For comparison, we employ some baseline methods that will be described
in the following sections. Their C++ implementations with Matlab interfaces
are published as a part of LIBSVM, BudgetedSVM\footnote{\href{http://www.dabi.temple.edu/budgetedsvm/index.html}{http://www.dabi.temple.edu/budgetedsvm/index.html}}
and LSOKL\footnote{\href{http://lsokl.stevenhoi.com/}{http://lsokl.stevenhoi.com/}}
toolboxes. Throughout the experiments, we utilize RBF kernel, i.e.,
$K\left(x,x^{'}\right)=\exp\left(-\gamma\left\Vert x-x^{'}\right\Vert ^{2}\right)$
for all algorithms including ours. We use hypersphere strategy to
construct the $\delta$-coverage (cf. Section~\ref{subsec:Construction-of--Coverage}),
due to its better performance than that of hyperrectangle approach
during model evaluation. All experiments are conducted using a Windows
machine with 3.46GHz Xeon processor and 96GB RAM.

\subsection{Model evaluation on the effect of hyperparameters\label{subsec:exp_hyperparam}}

In the first experiment, we investigate the effect of hyperparameters,
i.e., $\delta$-coverage diameter, sampling parameters $\beta$ and
$\rho$ (cf. Section~\ref{subsec:avm_upperbound}) on the performance
of $\model$s. Particularly, we conduct an initial analysis to quantitatively
evaluate the sensitivity of these hyperparameters and their impact
on the predictive accuracy and model size. This analysis provides
a heuristic approach to find the best setting of hyperparameters.
Here the $\model$ with Hinge loss is trained following the online
classification scheme using two datasets \emph{a9a} and \emph{cod-rna}.

To find the plausible range of coverage diameter, we use a heuristic
approach as follows. First we compute the mean and standard deviation
of pairwise Euclidean distances between any two data samples. Treating
the mean as the radius, the coverage diameter is then varied around
twice of this mean bounded by twice of the standard deviation. Fig.~\ref{fig:exp_hyperparam_delta_a9a}
and Fig.~\ref{fig:exp_hyperparam_delta_codrna} report the average
mistake rates and model sizes of $\model$s with respect to (w.r.t)
these values for datasets \emph{a9a} and \emph{cod-rna}, respectively.
Here we set $\beta=0$ and $\rho=1.0$. There is a consistent pattern
in both figures: the classification errors increase for larger $\delta$
whilst the model sizes decrease. This represents the trade-off between
model performance and model size via the model coverage. To balance
the performance and model size, in these cases, we can choose $\delta=7.0$
for \emph{a9a} data and $\delta=1.0$ for \emph{cod-rna} data.

\begin{figure}[h]
\noindent \begin{centering}
\subfloat[The effect of $\delta$-coverage diameter on the mistake rate and
model size.\label{fig:exp_hyperparam_delta_a9a}]{\noindent \begin{centering}
\includegraphics[width=0.42\textwidth]{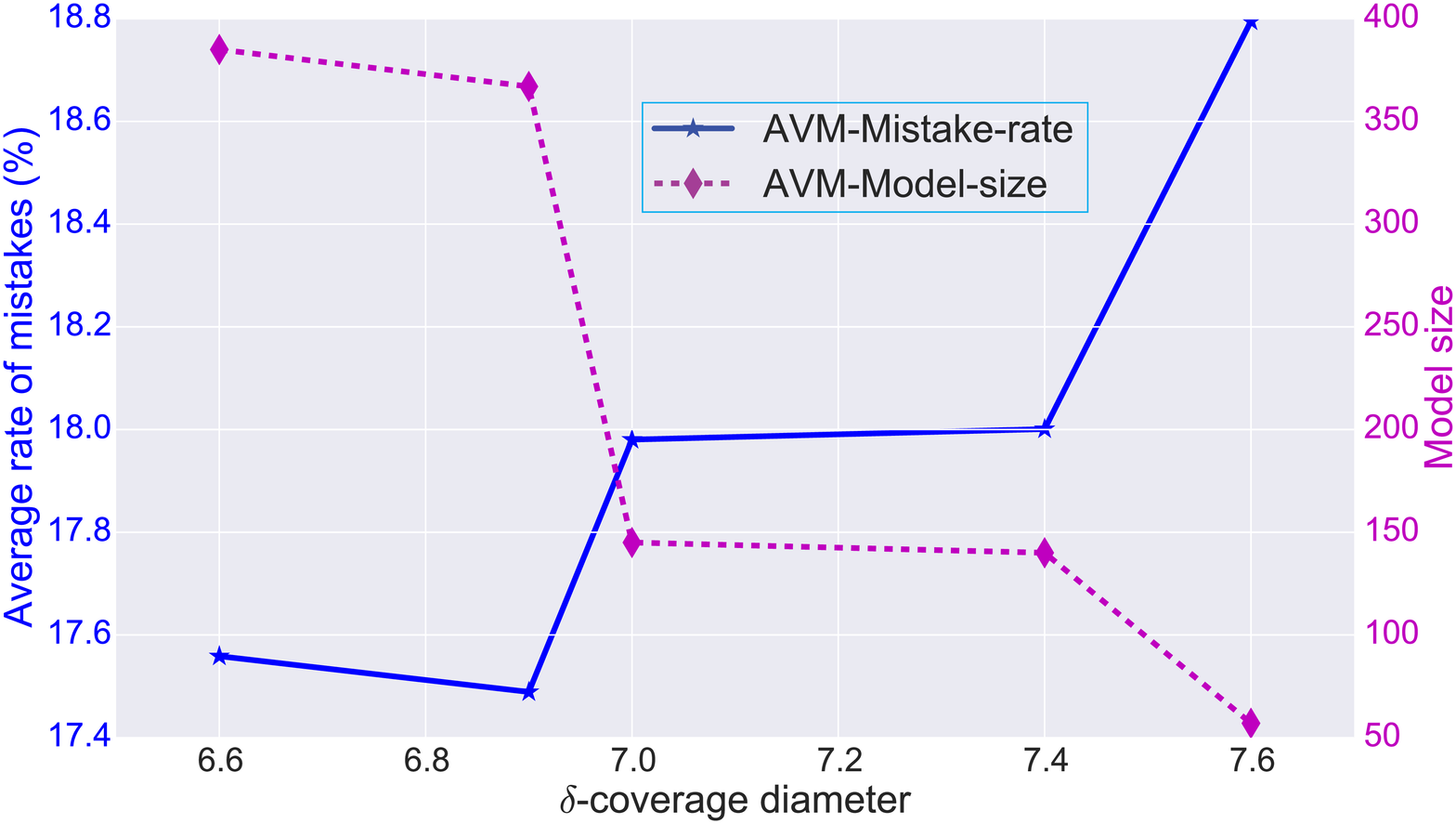}
\par\end{centering}
\noindent \centering{}}\hspace{0.02\textwidth}\subfloat[The effect of $\beta$ and $\rho$ on the classification mistake rate.
$\beta=0$ means always approximating.\label{fig:exp_hyperparam_beta_rho_a9a}]{\noindent \begin{centering}
\includegraphics[width=0.5\textwidth]{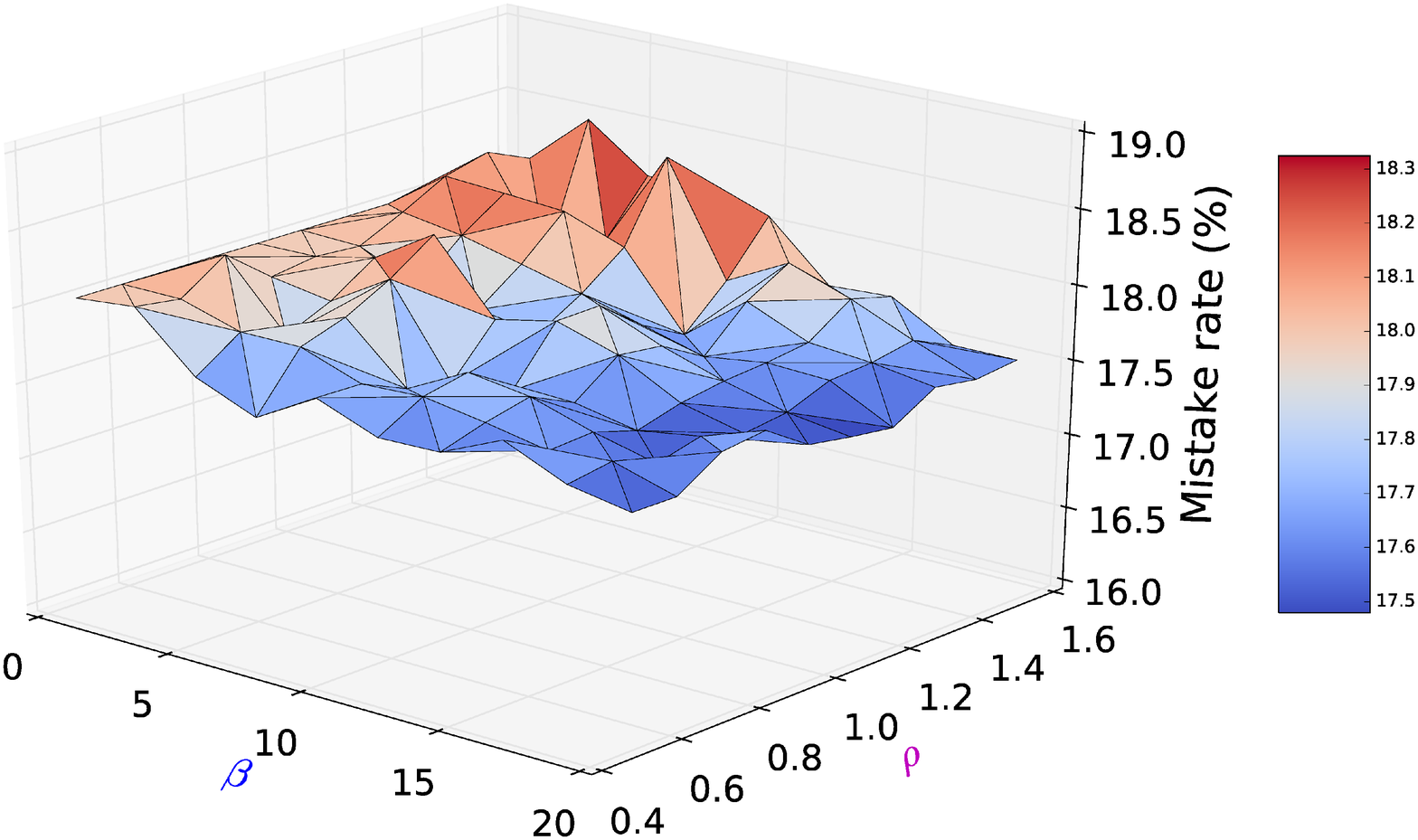}
\par\end{centering}
\noindent \centering{}}
\par\end{centering}
\noindent \centering{}\caption{Performance evaluation of $\protect\model$ with Hinge loss trained
using \emph{a9a} dataset with different values of hyperparameters.}
\end{figure}

\begin{figure}[h]
\noindent \begin{centering}
\subfloat[The effect of $\delta$-coverage diameter on the mistake rate and
model size.\label{fig:exp_hyperparam_delta_codrna}]{\noindent \begin{centering}
\includegraphics[width=0.42\textwidth]{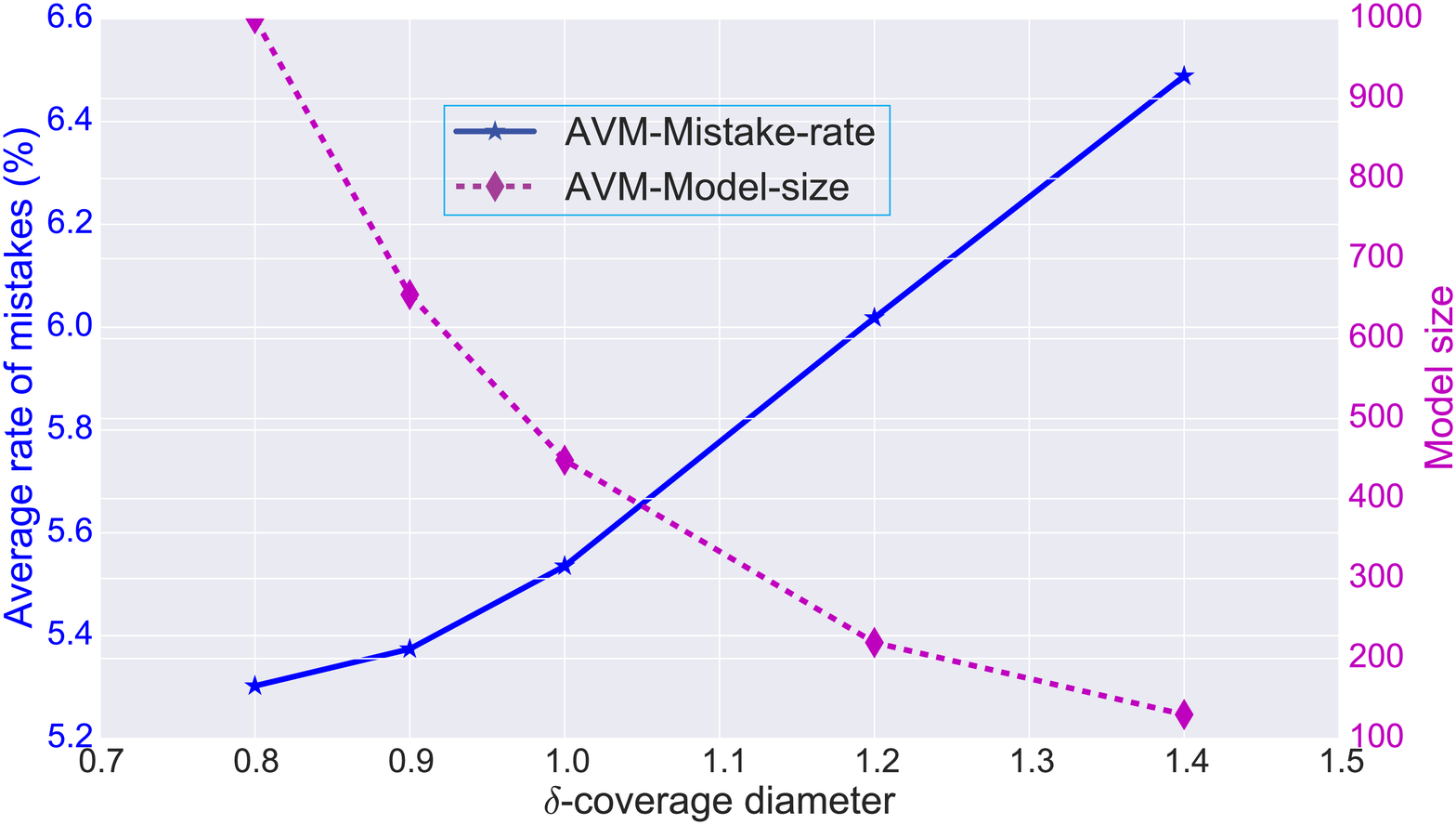}
\par\end{centering}
\noindent \centering{}}\hspace{0.02\textwidth}\subfloat[The effect of $\beta$ and $\rho$ on the classification mistake rate.
$\beta=0$ means always approximating.\label{fig:exp_hyperparam_beta_rho_codrna}]{\noindent \begin{centering}
\includegraphics[width=0.5\textwidth]{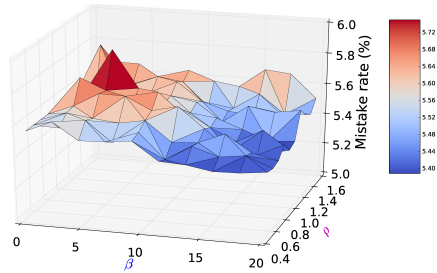}
\par\end{centering}
\noindent \centering{}}
\par\end{centering}
\noindent \centering{}\caption{Performance evaluation of $\protect\model$ with Hinge loss trained
using \emph{cod-rna} dataset with different values of hyperparameters.}
\end{figure}

Fixing the coverage diameters, we vary $\beta$ and $\rho$ in $10$
values monotonically increasing from $0$ to $10$ and from $0.5$
to $1.5$, respectively, to evaluate the classification performance.
The smaller $\beta$ and larger $\rho$ indicate that the machine
approximates the new incoming data more frequently, resulting in less
powerful prediction capability. This can be observed in Fig.~\ref{fig:exp_hyperparam_beta_rho_a9a}
and Fig.~\ref{fig:exp_hyperparam_beta_rho_codrna}, which depict
the average mistake rates in 3D as a function of these values for
dataset \emph{a9a} and \emph{cod-rna}. Here $\beta=0$ means that
the model always performs approximation without respect to the value
of $\rho$. From these visualizations, we found that the $\model$
with always-on approximation mode still can achieve fairly comparable
classification results. Thus we set $\beta=0$ for all following experiments.

\subsection{Batch classification\label{subsec:exp_batch_classification}}

We now examine the performances of $\model$s in classification task
following batch mode. We use eight datasets: \emph{a9a}, \emph{w8a},
\emph{cod-rna}, \emph{KDDCup99}, \emph{ijcnn1}, \emph{covtype}, \emph{poker}
and \emph{airlines} (delayed and non-delayed labels). We create two
versions of our approach: $\model$ with Hinge loss ($\model$-Hinge)
and $\model$ with Logistic loss ($\model$-Logit). It is noteworthy
that the Hinge loss is not a smooth function with undefined gradient
at the point that the classification confidence $yf\left(x\right)=1$.
Following the sub-gradient definition, in our experiment, we compute
the gradient given the condition that $yf\left(x\right)<1$, and set
it to $0$ otherwise.

\paragraph{Baselines.}

For discriminative performance comparison, we recruit the following
state-of-the-art baselines to train kernel SVMs for classification
in batch mode:
\begin{itemize}
\item LIBSVM: one of the most widely-used and state-of-the-art implementations
for batch kernel SVM solver \citep{libsvm}. We use the one-vs-all
approach as the default setting for the multiclass tasks;
\item LLSVM: low-rank linearization SVM algorithm that approximates kernel
SVM optimization by a linear SVM using low-rank decomposition of the
kernel matrix \citep{zhang_etal_aistats12_scaling};
\item BSGD-M: budgeted stochastic gradient descent algorithm which extends
the Pegasos algorithm \citep{shalev2011pegasos} by introducing a
merging strategy for support vector budget maintenance \citep{wang2012};
\item BSGD-R: budgeted stochastic gradient descent algorithm which extends
the Pegasos algorithm \citep{shalev2011pegasos} by introducing a
removal strategy for support vector budget maintenance \citep{wang2012};
\item FOGD: Fourier online gradient descent algorithm that applies the random
Fourier features for approximating kernel functions \citep{Lu_2015large};
\item NOGD: Nystrom online gradient descent (NOGD) algorithm that applies
the Nystrom method to approximate large kernel matrices \citep{Lu_2015large}.
\end{itemize}

\paragraph{Hyperparameters setting.}

There are a number of different hyperparameters for all methods. Each
method requires a different set of hyperparameters, e.g., the regularization
parameters ($C$ in LIBSVM, $\lambda$ in Pegasos and $\model$),
the learning rates ($\eta$ in FOGD and NOGD), the coverage diameter
($\delta$ in $\model$) and the RBF kernel width ($\gamma$ in all
methods). Thus, for a fair comparison, these hyperparameters are specified
using cross-validation on training subset.

Particularly, we further partition the training set into $80\%$ for
learning and $20\%$ for validation. For large-scale databases, we
use only $1\%$ of training set, so that the searching can finish
within an acceptable time budget. The hyperparameters are varied in
certain ranges and selected for the best performance on the validation
set. The ranges are given as follows: $C\in\left\{ 2^{-5},2^{-3},...,2^{15}\right\} $,
$\lambda\in\left\{ \nicefrac{2^{-4}}{N},\nicefrac{2^{-2}}{N},...,\nicefrac{2^{16}}{N}\right\} $,
$\gamma\in\left\{ 2^{-8},2^{-4},2^{-2},2^{0},2^{2},2^{4},2^{8}\right\} $,
$\eta\in\left\{ 16.0,8.0,4.0,2.0,0.2,0.02,0.002,0.0002\right\} $
where $N$ is the number of data points. The coverage diameter $\delta$
of $\model$ is selected following the approach described in Section~\ref{subsec:exp_hyperparam}.
For the budget size $B$ in NOGD and Pegasos algorithm, and the feature
dimension $D$ in FOGD for each dataset, we use identical values to
those used in Section~7.1.1 of \citep{Lu_2015large}.

\paragraph{Results.}

The classification results, training and testing time costs are reported
in Table~\ref{tab:exp_batch_classification}. Overall, the batch
algorithms achieve the highest classification accuracies whilst those
of online algorithms are lower but fairly competitive. The online
learning models, however, are much sparser, resulting in a substantial
speed-up, in which the training time costs and model sizes of $\model$s
are smallest with orders of magnitude lower than those of the standard
batch methods. More specifically, the LIBSVM outperforms the other
approaches in most of datasets, on which its training phase finishes
within the time limit (i.e., two hours), except for the \emph{ijcnn1}
data wherein its testing score is less accurate but very close to
that of BSGD-M. The LLSVM achieves good results which are slightly
lower than those of the state-of-the-art batch kernel algorithm. The
method, however, does not support multiclass classification. These
two batch algorithms \textendash{} LIBSVM and LLSVM could not be trained
within the allowable amount of time on large-scale datasets (e.g.,
\emph{airlines}), thus are not scalable.

Furthermore, six online algorithms in general have significant advantages
against the batch methods in computational efficiency, especially
when running on large-scale datasets. Among these algorithms, the
BSGD-M (Pegasos+merging) obtains the highest classification scores,
but suffers from a high computational cost. This can be seen in almost
all datasets, especially for the airlines dataset on which its learning
exceeds the time limit. The slow training of BSGD-M is caused by the
merging step with computational complexity $\mathcal{O}\left(B^{2}\right)$
($B$ is the budget size). By contrast, the BSGD-R (Pegasos+removal)
runs faster than the merging approach, but suffers from very high
inaccurate results due to its naive budget maintenance strategy, that
simply discards the most redundant support vector which may contain
important information.
\begin{table}[H]
\noindent \begin{centering}
\caption{Classification performance of our $\protect\model$s and the baselines
in batch mode. The notation $\left[\delta\mid S\mid B\mid D\right]$,
next to the dataset name, denotes the diameter $\delta$, the model
size $S$ of $\protect\model$-based models, the budget size $B$
of budgeted algorithms, and the number of random features $D$ of
FOGD, respectively. The accuracy is reported in percent $\left(\%\right)$,
the training time and testing time are in second. The best performance
is in \textbf{bold}. It is noteworthy that the LLSVM does not support
multiclass classification and we terminate all runs exceeding the
limit of two hours, therefore some results are unavailable.\label{tab:exp_batch_classification}}
\par\end{centering}
\noindent \centering{}\resizebox{1.0\textwidth}{!}{
\begin{tabular}{|c|r|r|r|r|r|r|}
\hline 
\textbf{\emph{\cellcolor{header_color}Dataset $\left[\delta\mid S\mid B\mid D\right]$}} & \multicolumn{3}{c|}{\cellcolor{header_color}\textbf{\emph{a9a $\left[7.0\mid135\mid1,000\mid4,000\right]$}}} & \multicolumn{3}{c|}{\cellcolor{header_color}\textbf{\emph{w8a $\left[13.0\mid131\mid1,000\mid4,000\right]$}}}\tabularnewline
\hline 
\cellcolor{header_color}\textbf{\emph{Algorithm}} & Train & Test & Accuracy & Train & Test & Accuracy\tabularnewline
\hline 
LIBSVM & 84.57 & 22.23 & \textbf{84.92} & 50.96 & 2.95 & \textbf{99.06}\tabularnewline
LLSVM & 50.73 & 8.73 & 83.00 & 92.19 & 10.41 & 98.64\tabularnewline
\hline 
BSGD-M & 232.59 & 2.88 & 84.76$\pm$0.16 & 264.70 & 5.16 & 98.17$\pm$0.07\tabularnewline
BSGD-R & 90.48 & 2.72 & 80.26$\pm$3.38 & 253.30 & 4.98 & 97.10$\pm$0.04\tabularnewline
FOGD & 15.99 & 2.87 & 81.15$\pm$5.05 & 32.16 & 3.55 & 97.92$\pm$0.38\tabularnewline
NOGD & 82.40 & 0.60 & 82.33$\pm$2.18 & 374.87 & 0.65 & 98.06$\pm$0.18\tabularnewline
\hline 
$\model$-Hinge & \textbf{4.96} & \textbf{0.25} & 83.55$\pm$0.50 & \textbf{11.84} & \textbf{0.52} & 96.87$\pm$0.28\tabularnewline
$\model$-Logit & 5.35 & \textbf{0.25} & 83.83$\pm$0.34 & 12.54 & \textbf{0.52} & 96.96$\pm$0.00\tabularnewline
\hline 
\hline 
\textbf{\emph{\cellcolor{header_color}Dataset $\left[\delta\mid S\mid B\mid D\right]$}} & \multicolumn{3}{c|}{\cellcolor{header_color}\textbf{\emph{cod-rna $\left[1.0\mid436\mid400\mid1,600\right]$}}} & \multicolumn{3}{c|}{\cellcolor{header_color}\textbf{\emph{ijcnn1 $\left[1.0\mid500\mid1,000\mid4,000\right]$}}}\tabularnewline
\hline 
\cellcolor{header_color}\textbf{\emph{Algorithm}} & Train & Test & Accuracy & Train & Test & Accuracy\tabularnewline
\hline 
LIBSVM & 114.90 & 85.34 & \textbf{96.39} & 38.63 & 11.17 & 97.35\tabularnewline
LLSVM & 20.17 & 19.38 & 94.16 & 40.62 & 54.22 & 96.99\tabularnewline
\hline 
BSGD-M & 90.62 & 5.66 & 95.67$\pm$0.21 & 93.05 & 6.13 & \textbf{97.69$\pm$0.11}\tabularnewline
BSGD-R & 19.31 & 5.48 & 66.83$\pm$0.11 & 41.70 & 7.07 & 90.90$\pm$0.18\tabularnewline
FOGD & 7.62 & 11.95 & 92.65$\pm$4.20 & 7.31 & 10.10 & 90.64$\pm$0.07\tabularnewline
NOGD & 9.81 & 3.24 & 91.83$\pm$3.35 & 21.58 & 3.68 & 90.43$\pm$1.22\tabularnewline
\hline 
$\model$-Hinge & \textbf{6.52} & \textbf{2.69} & 94.38$\pm$1.16 & \textbf{6.47} & 2.71 & 91.14$\pm$0.71\tabularnewline
$\model$-Logit & 7.03 & 2.86 & 93.10$\pm$2.11 & 6.86 & \textbf{2.67} & 91.19$\pm$0.95\tabularnewline
\hline 
\hline 
\textbf{\emph{\cellcolor{header_color}Dataset $\left[\delta\mid S\mid B\mid D\right]$}} & \multicolumn{3}{c|}{\emph{\cellcolor{header_color}}\textbf{\emph{covtype $\left[3.0\mid59\mid400\mid1,600\right]$}}} & \multicolumn{3}{c|}{\cellcolor{header_color}\textbf{\emph{poker $\left[12.0\mid393\mid1,000\mid4,000\right]$}}}\tabularnewline
\hline 
\textbf{\emph{\cellcolor{header_color}Algorithm}} & Train & Test & Accuracy & Train & Test & Accuracy\tabularnewline
\hline 
LIBSVM & \textendash{} & \textendash{} & \textendash{} & 40.03 & 932.58 & \textbf{57.91}\tabularnewline
LLSVM & \textendash{} & \textendash{} & \textendash{} & \textendash{} & \textendash{} & \textendash{}\tabularnewline
\hline 
BSGD-M & 2,413.15 & 3.75 & \textbf{72.26$\pm$0.16} & 414.09 & 123.57 & 54.10$\pm$0.22\tabularnewline
BSGD-R & 418.68 & 3.02 & 61.09$\pm$1.69 & 35.76 & 102.84 & 52.14$\pm$1.05\tabularnewline
FOGD & 69.94 & 2.45 & 59.34$\pm$5.85 & 9.61 & 101.29 & 46.62$\pm$5.00\tabularnewline
NOGD & 679.50 & 0.76 & 68.20$\pm$2.96 & 118.54 & 36.84 & 54.65$\pm$0.27\tabularnewline
\hline 
$\model$-Hinge & \textbf{60.27} & 0.26 & 64.31$\pm$0.37 & 3.86 & 8.21 & 55.49$\pm$0.13\tabularnewline
$\model$-Logit & 61.92 & \textbf{0.22} & 64.42$\pm$0.34 & \textbf{3.36} & \textbf{7.54} & 55.60$\pm$0.17\tabularnewline
\hline 
\hline 
\textbf{\emph{\cellcolor{header_color}Dataset $\left[\delta\mid S\mid B\mid D\right]$}} & \multicolumn{3}{c|}{\cellcolor{header_color}\textbf{\emph{KDDCup99 $\left[3.0\mid115\mid200\mid400\right]$}}} & \multicolumn{3}{c|}{\cellcolor{header_color}\textbf{\emph{airlines $\left[1.0\mid388\mid1,000\mid4,000\right]$}}}\tabularnewline
\hline 
\textbf{\emph{\cellcolor{header_color}Algorithm}} & Train & Test & Accuracy & Train & Test & Accuracy\tabularnewline
\hline 
LIBSVM & 4,380.58 & 661.04 & \textbf{99.91} & \textendash{} & \textendash{} & \textendash{}\tabularnewline
LLSVM & \textendash{} & \textendash{} & \textendash{} & \textendash{} & \textendash{} & \textendash{}\tabularnewline
\hline 
BSGD-M & 2,680.58 & 21.25 & 99.73$\pm$0.00 & \textendash{} & \textendash{} & \textendash{}\tabularnewline
BSGD-R & 1,644.25 & 14.33 & 39.81$\pm$2.26 & 4,741.68 & 29.98 & 80.27$\pm$0.06\tabularnewline
FOGD & 706.20 & 22.73 & 99.75$\pm$0.11 & 1,085.73 & 861.52 & 80.37$\pm$0.21\tabularnewline
NOGD & 3,726.21 & 3.11 & 99.80$\pm$0.02 & 3,112.08 & 18.53 & 74.83$\pm$0.20\tabularnewline
\hline 
$\model$-Hinge & \textbf{554.42} & \textbf{2.75} & 99.82$\pm$0.05 & \textbf{586.90} & 6.55 & \textbf{80.72$\pm$0.00}\tabularnewline
$\model$-Logit & 576.76 & 2.80 & 99.72$\pm$0.06 & 642.23 & \textbf{6.10} & \textbf{80.72$\pm$0.00}\tabularnewline
\hline 
\end{tabular}}
\end{table}

In terms of predictive performance, our proposed methods outperform
the recent advanced online learning algorithms \textendash{} FOGD
and NOGD in most scenarios. The $\model$-based models are able to
achieve slightly less accurate but fairly comparable results compared
with those of the state-of-the-art LIBSVM algorithm. In terms of sparsity
and speed, the $\model$s are the fastest ones in the training and
testing phases in all cases thanks to their remarkable smaller model
sizes. The difference between the training speed of our $\model$s
and that of two approaches varies across datasets. The gap is more
significant for datasets with higher dimensional feature spaces. This
is expected because the procedure to compute random features for each
data point of FOGD involves \emph{sin} and \emph{cos} operators which
are costly. These facts indicate that our proposed online kernel learning
algorithms are both efficient and effective in solving large-scale
kernel classification problems. Thus we believe that the $\model$
is the fast alternative to the existing SVM solvers for large-scale
classification tasks.

Finally, comparing two versions of $\model$s, it can be seen that
the discriminative performances of $\model$ with Logistic loss are
better than those of $\model$ with Hinge loss in most of datasets.
This is because the Logistic function is smoother than the Hinge function,
whilst the Hinge loss encourages sparsity of the model. The $\model$-Logit,
however, contains additional exponential operators, resulting in worse
training time.

\subsection{Online classification\label{subsec:exp_online_classification}}

The next experiment investigates the performance of the $\model$s
in online classification task where individual data point continuously
come turn-by-turn in a stream. Here we also use eight datasets and
two versions of our approach: $\model$ with Hinge loss ($\model$-Hinge)
and $\model$ with Logistic loss ($\model$-Logit) which are used
in batch classification setting (cf. Section~\ref{subsec:exp_batch_classification}). 

\paragraph{Baselines.}

We recruit the two widely-used algorithms \textendash{} Perceptron
and OGD for regular online kernel classification without budget maintenance
and 8 state-of-the-art budget online kernel learning methods as follows:
\begin{itemize}
\item Perceptron: the kernelized variant without budget of Perceptron algorithm
\citep{Freund1999:LMC:337859.337869};
\item OGD: the kernelized variant without budget of online gradient descent
\citep{Kivinen2004}.
\item RBP: a budgeted Perceptron algorithm using random support vector removal
strategy \citep{CavallantiCG07};
\item Forgetron: a kernel-based Perceptron maintaining a fixed budget by
discarding oldest support vectors \citep{DekelSS05};
\item Projectron: a Projectron algorithm using the projection strategy \citep{Orabona2009};
\item Projectron++: the aggressive version of Projectron algorithm \citep{Orabona2009};
\item BPAS: a budgeted variant of Passive-Aggressive algorithm with simple
SV removal strategy \citep{Wang2010};
\item BOGD: a budgeted variant of online gradient descent algorithm using
simple SV removal strategy \citep{Zhao2012};
\item FOGD and NOGD: described in Section~\ref{subsec:exp_batch_classification}.
\end{itemize}

\paragraph{Hyperparameters setting.}

For each method learning on each dataset, we follow the same hyperparameter
setting which is optimized in the batch classification task. For time
efficiency, we only include the fast algorithms FOGD, NOGD and $\model$s
for the experiments on large-scale datasets. The other methods would
exceed the time limit when running on such data.

\paragraph{Results.}

Fig.~\ref{fig:exp_online_classification_mistake_rate} and Fig.~\ref{fig:exp_online_classification_time_cost}
shows the relative performance convergence w.r.t classification error
and computation cost of the $\model$s in comparison with those of
the baselines. Combining these two figures, we compare the average
mistake rate and running time in Fig.~\ref{fig:exp_online_classification_mistake_vs_time}.
Table~\ref{tab:exp_online_classification} reports the final average
results in detailed numbers after the methods see all data samples.
It is worthy to note that for the four biggest datasets (\emph{KDDCup99},
\emph{covtype}, \emph{poker}, \emph{airlines}) that consist of millions
data points, we exclude the non-budgeted online learning algorithm
because of their substantially expensive time costs. From these results,
we can draw some observations as follows.

First of all, as can be seen from Fig.~\ref{fig:exp_online_classification_mistake_rate},
there are three groups of algorithms that have different learning
progresses in terms of classification mistake rate. The first group
includes the BOGD, Projectron and Forgetron that have the error rates
fluctuating at the beginning, but then being stable till the end.
In the meantime, the rates of the models in the second group, including
Perceptron, OGD, RBP, Projectron++ and BPAS, quickly saturate at a
plateau after these methods see a few portions, i.e., one-tenth to
two-tenth, of the data. By contrast, the last group includes the recent
online learning approaches \textendash{} FOGD, NOGD, and our proposed
ones \textendash{} $\model$-Hinge, $\model$-Logit, that regularly
perform better as more data points come. Exceptionally, for the dataset
\emph{w8a}, the classification errors of the methods in the first
group keep increasing after seeing four-tenth of the data, whilst
those of the last group are unexpectedly worse.

Second, Fig.~\ref{fig:exp_online_classification_mistake_vs_time}
plots average mistake rate against computational cost, which shows
similar patterns as in the our first observation. In addition, it
can be seen from Fig.~\ref{fig:exp_online_classification_time_cost}
that all algorithms have normal learning pace in which the execution
time is accumulated over the learning procedure. Only the Projectron++
is slow at the beginning but then performs faster after receiving
more data.

According to final results summarized in Table~\ref{tab:exp_online_classification},
the budgeted online approaches show efficacies with substantially
faster computation than the ones without budgets. This is more obvious
for larger datasets wherein the execution time costs of our proposed
models are several orders of magnitude lower than those of regular
online algorithms. This is because the coverage scheme of $\model$s
impressively boost their model sparsities, e.g., using $\delta=3$
resulting in $115$ core points for dataset \emph{KDDCup99} consisting
of $4,408,589$ instances, and using $\delta=1$ resulting in $388$
core points for dataset \emph{airlines} containing $5,336,471$ data
samples.

\begin{figure}[H]
\noindent \begin{centering}
\includegraphics[width=0.8\textwidth]{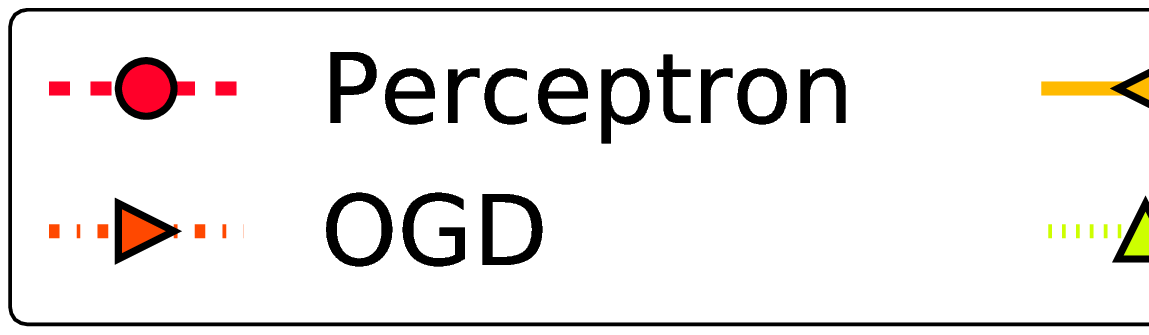}
\par\end{centering}
\noindent \begin{centering}
\subfloat[a9a]{\noindent \centering{}\includegraphics[width=0.45\textwidth]{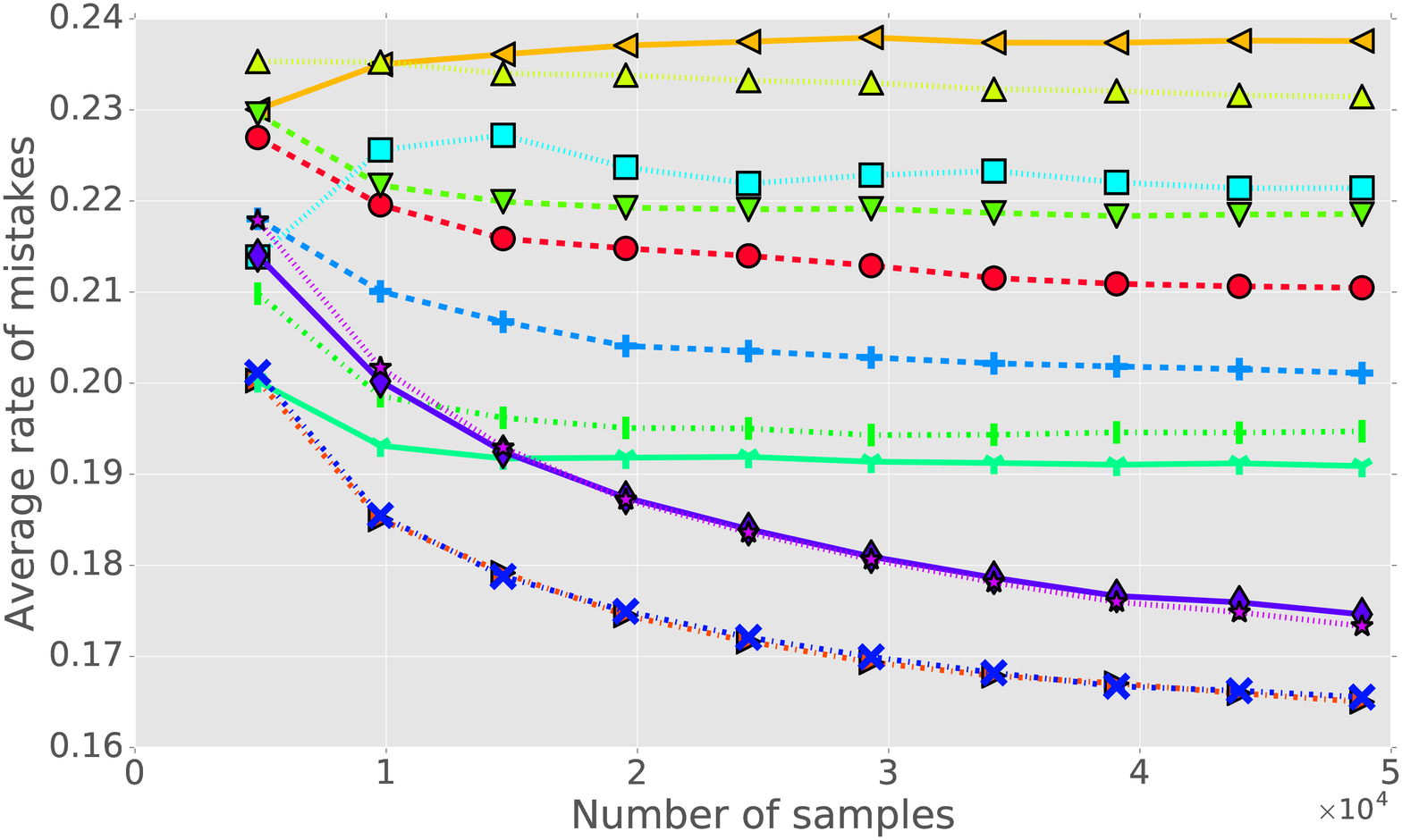}}\hspace{0.04\textwidth}\subfloat[w8a]{\noindent \centering{}\includegraphics[width=0.45\textwidth]{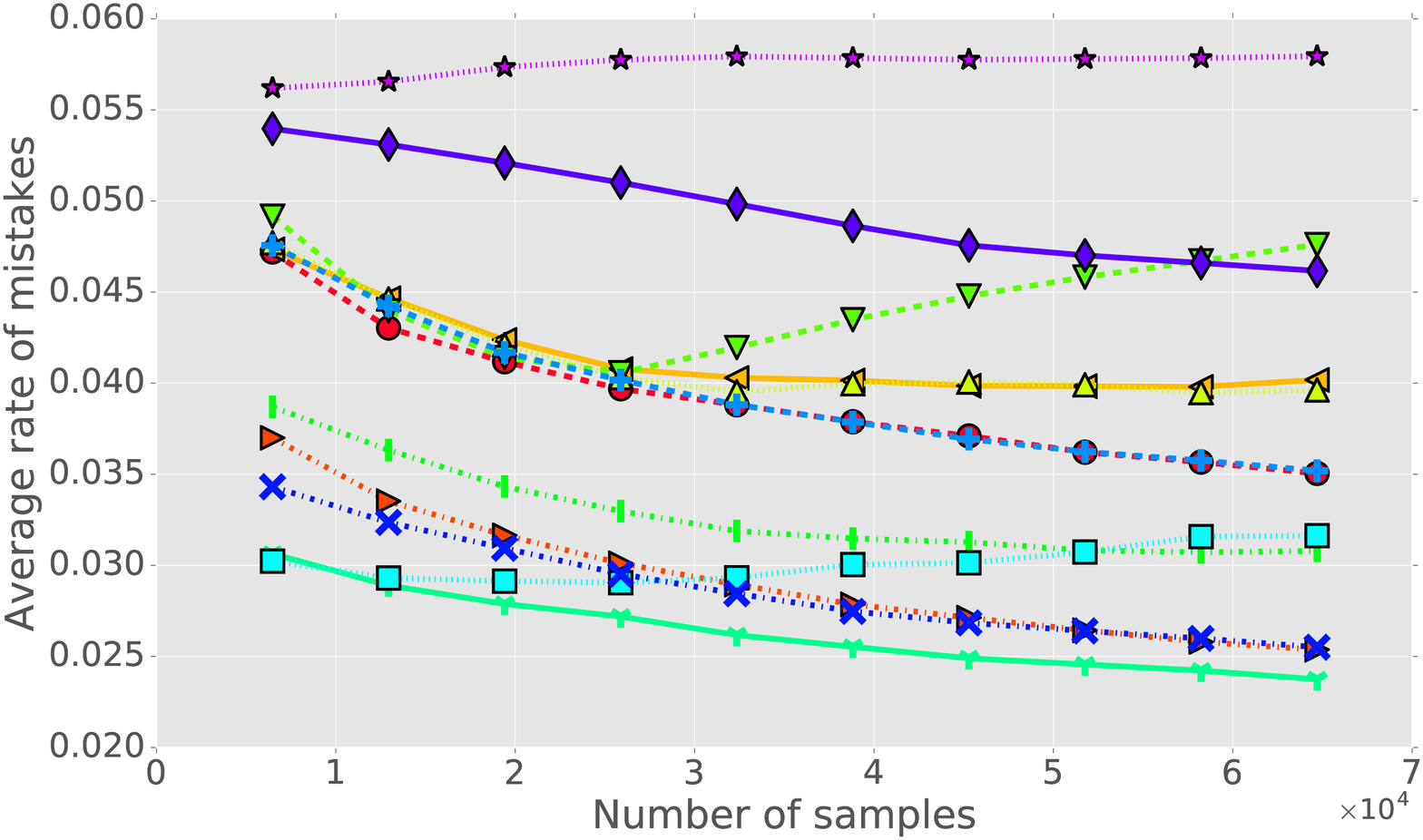}}\vspace{-2mm}
\par\end{centering}
\noindent \begin{centering}
\subfloat[cod-rna]{\noindent \centering{}\includegraphics[width=0.45\textwidth]{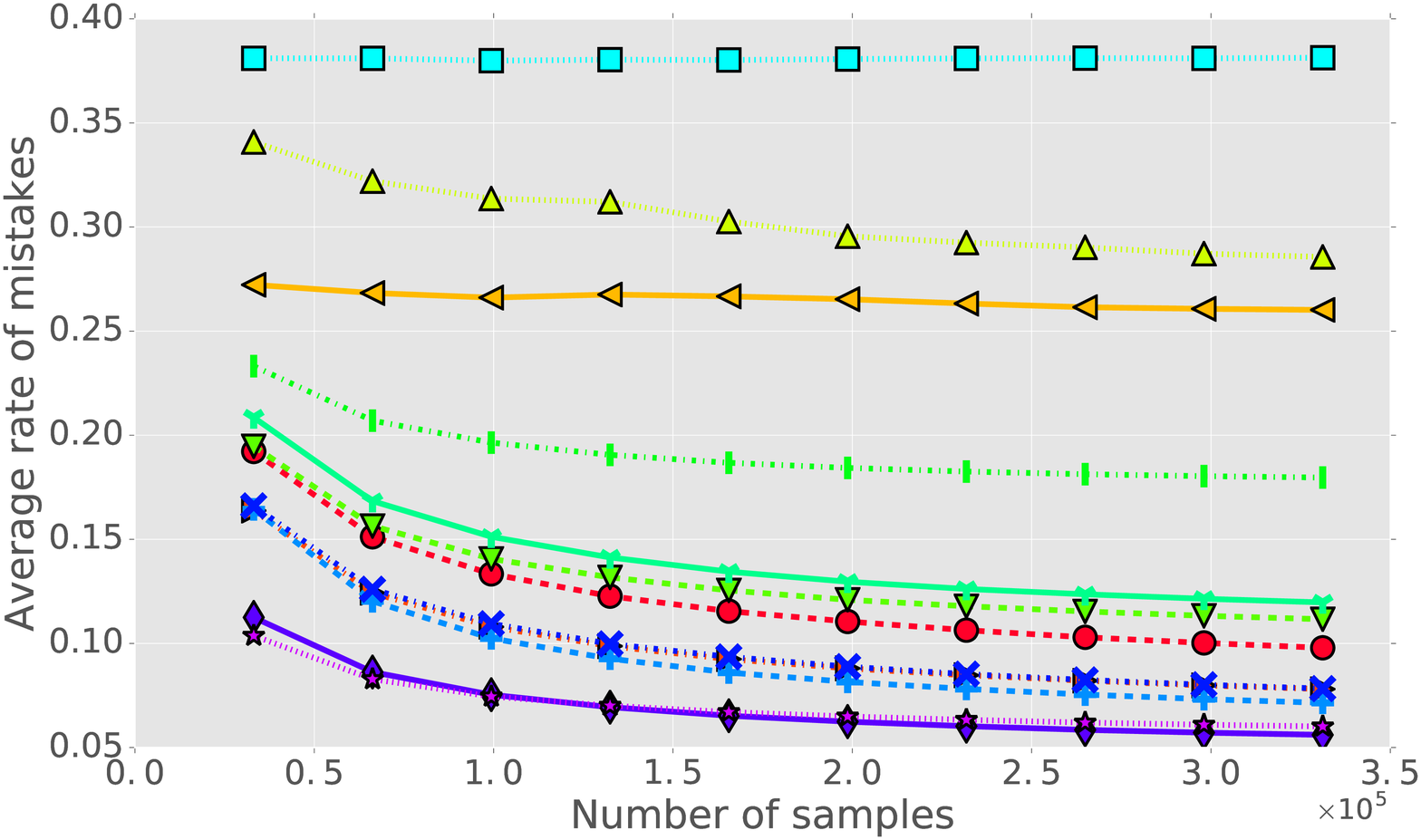}}\hspace{0.04\textwidth}\subfloat[ijcnn1]{\noindent \centering{}\includegraphics[width=0.45\textwidth]{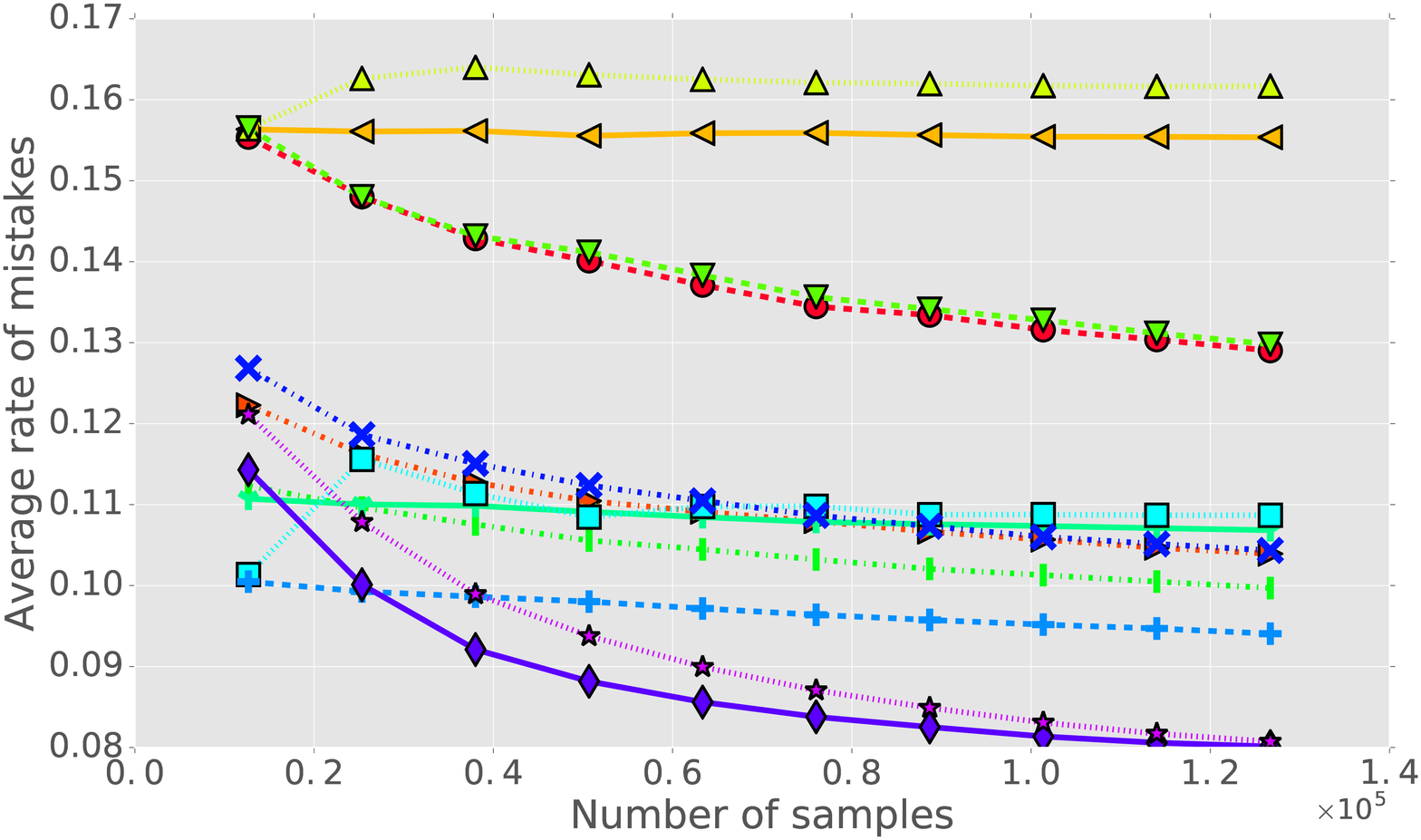}}\vspace{-2mm}
\par\end{centering}
\noindent \begin{centering}
\subfloat[KDDCup99]{\noindent \centering{}\includegraphics[width=0.45\textwidth]{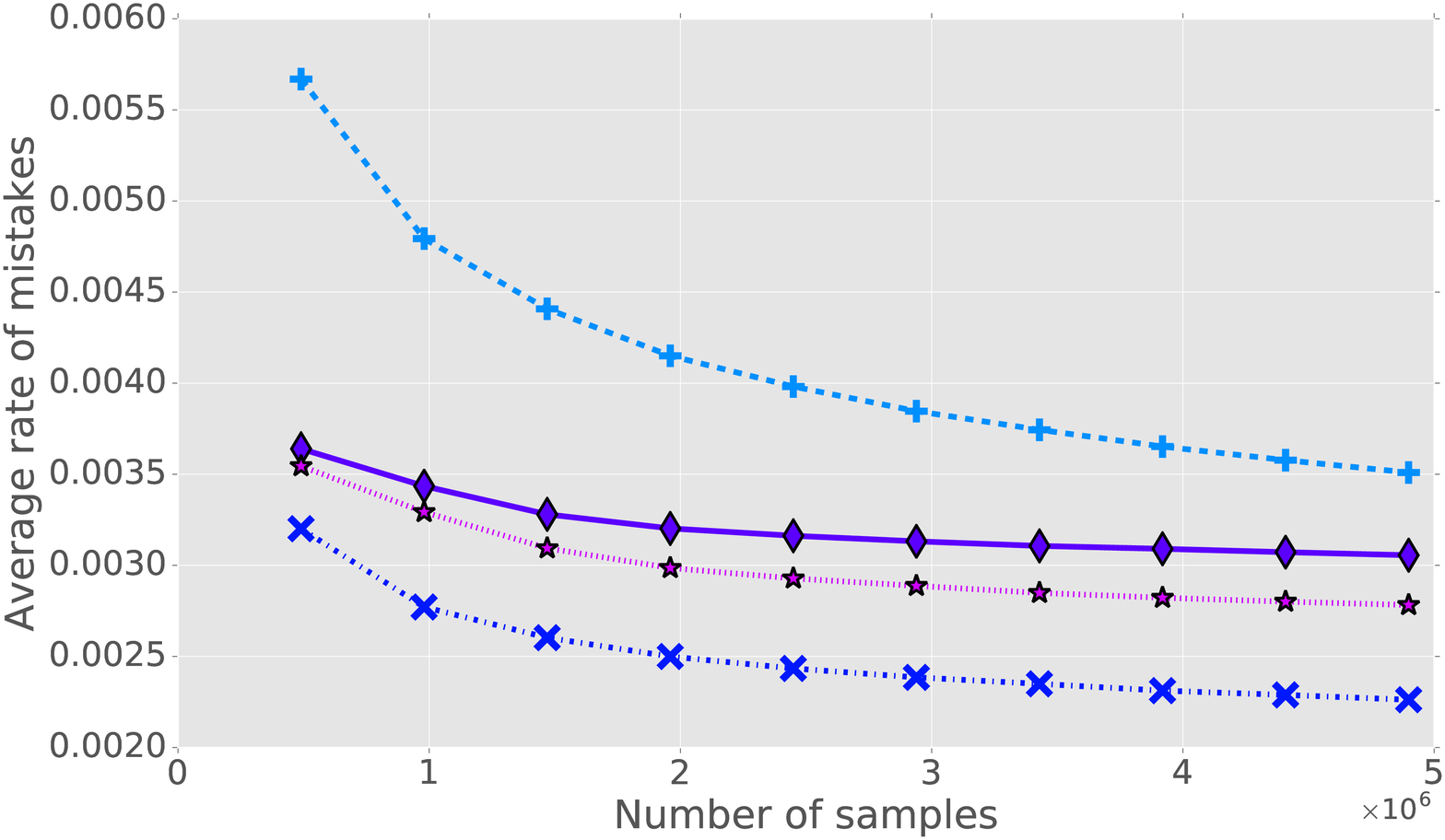}}\hspace{0.04\textwidth}\subfloat[covtype]{\noindent \centering{}\includegraphics[width=0.45\textwidth]{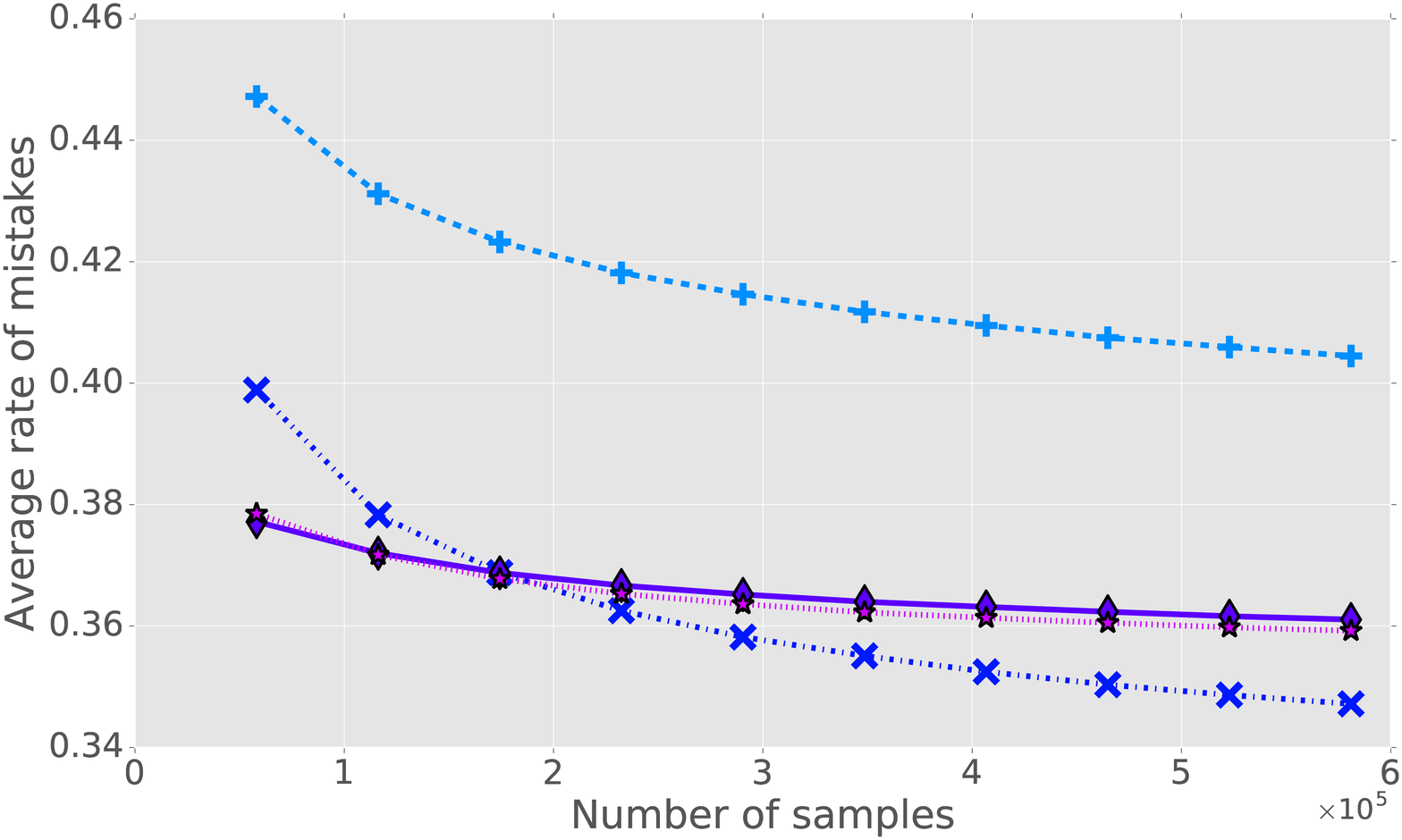}}\vspace{-2mm}
\par\end{centering}
\noindent \begin{centering}
\subfloat[poker]{\noindent \centering{}\includegraphics[width=0.45\textwidth]{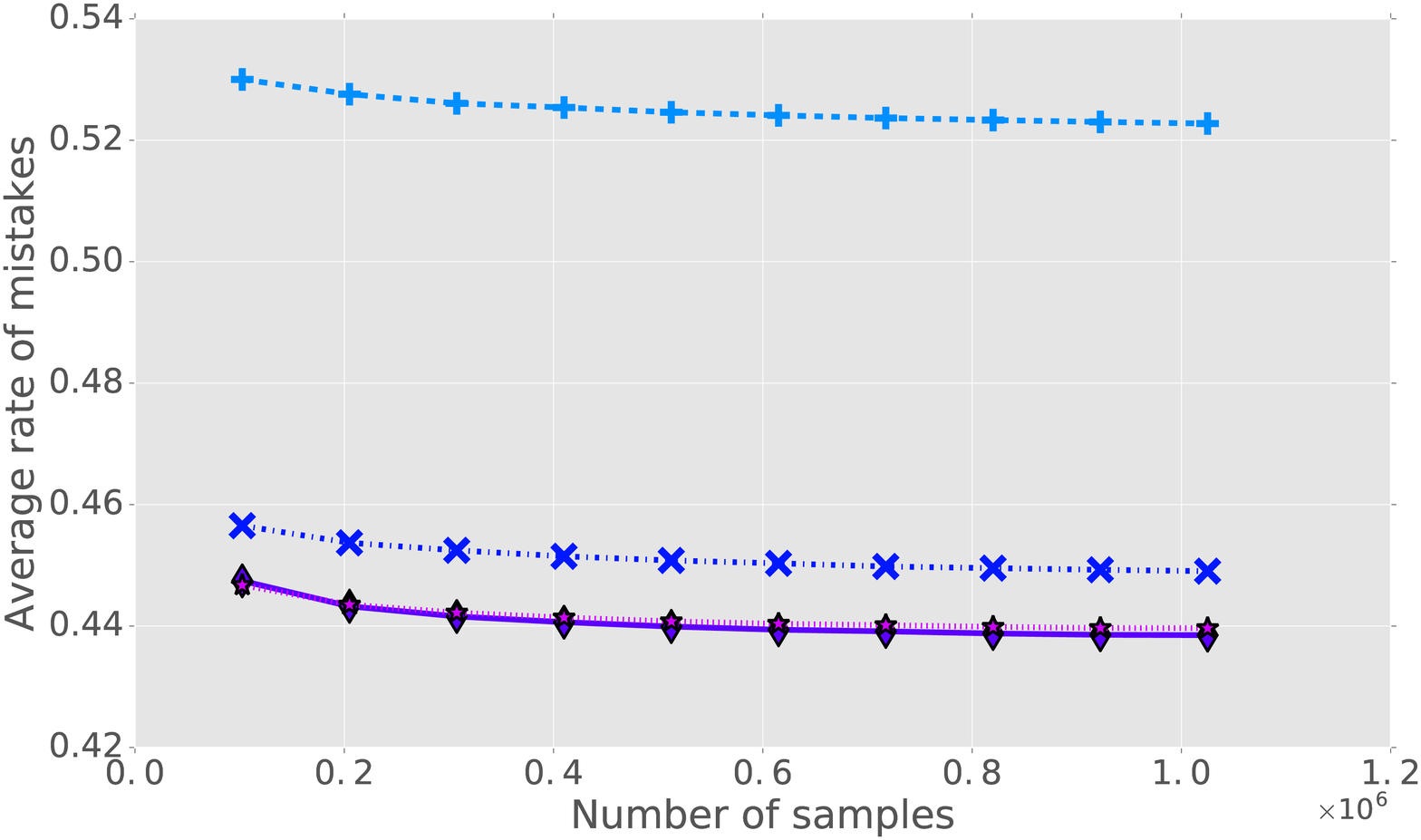}}\hspace{0.04\textwidth}\subfloat[airlines]{\noindent \centering{}\includegraphics[width=0.45\textwidth]{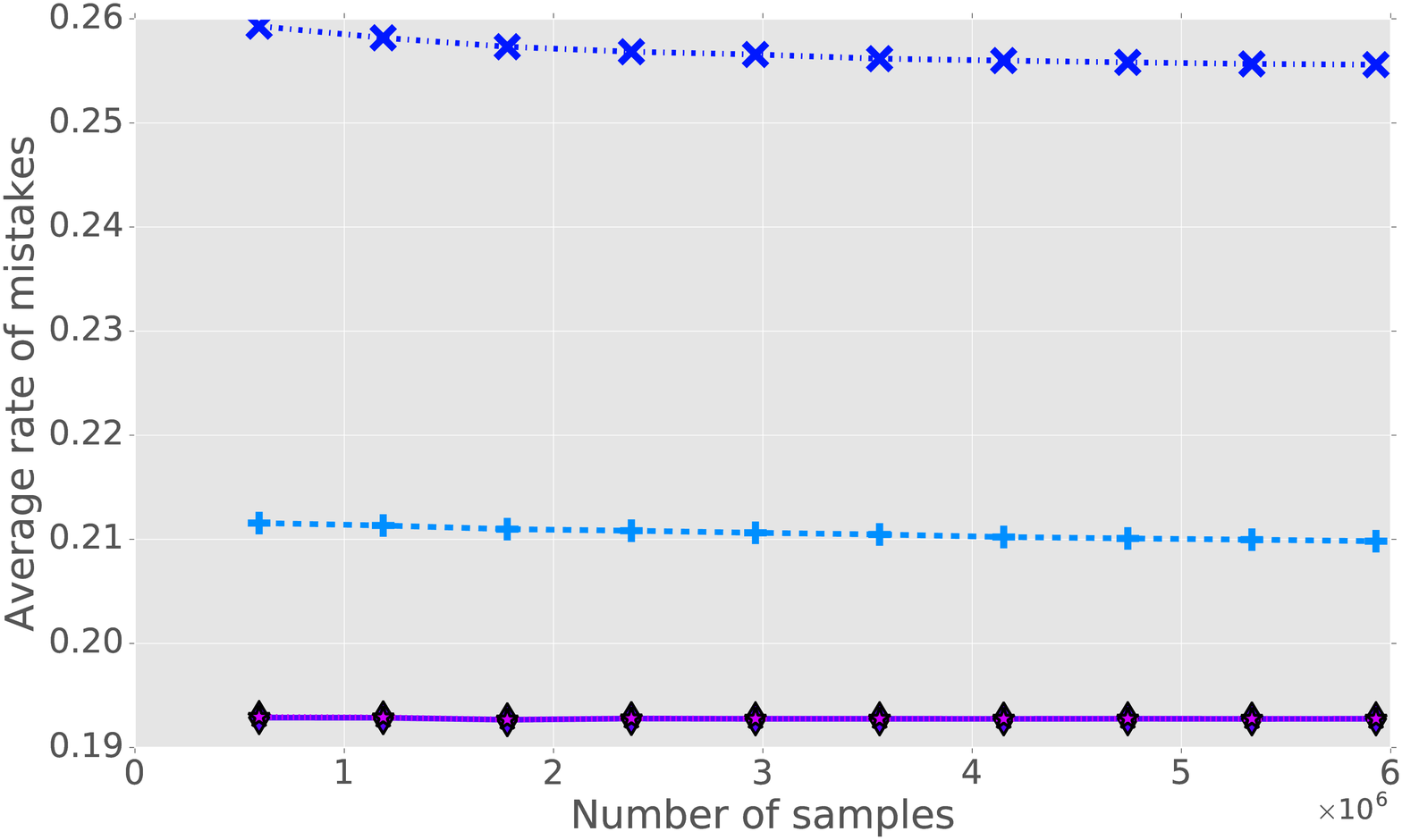}}
\par\end{centering}
\noindent \centering{}\caption{Convergence evaluation of online classification tasks: the average
rate of mistakes as a function of the number of samples seen by the
models. (Best viewed in colors).\label{fig:exp_online_classification_mistake_rate}}
\end{figure}
\begin{figure}[H]
\noindent \begin{centering}
\includegraphics[width=0.8\textwidth]{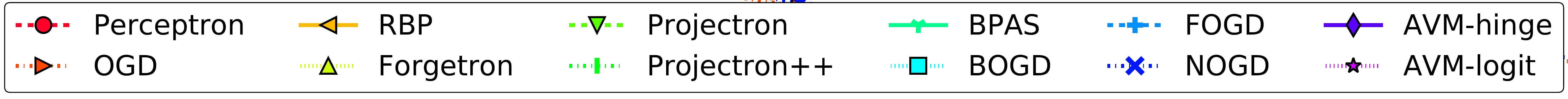}
\par\end{centering}
\noindent \begin{centering}
\subfloat[a9a]{\noindent \centering{}\includegraphics[width=0.45\textwidth]{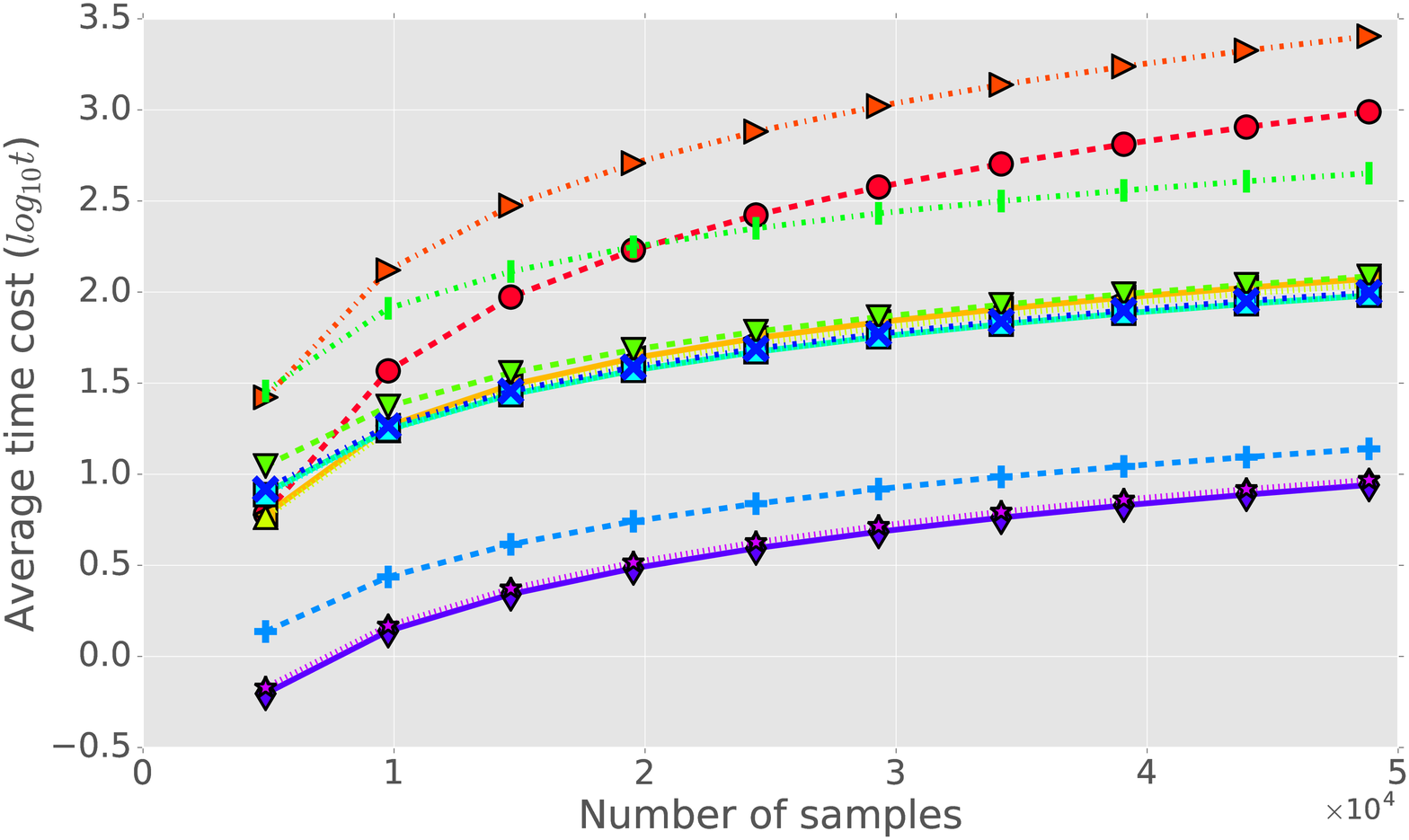}}\hspace{0.04\textwidth}\subfloat[w8a]{\noindent \centering{}\includegraphics[width=0.45\textwidth]{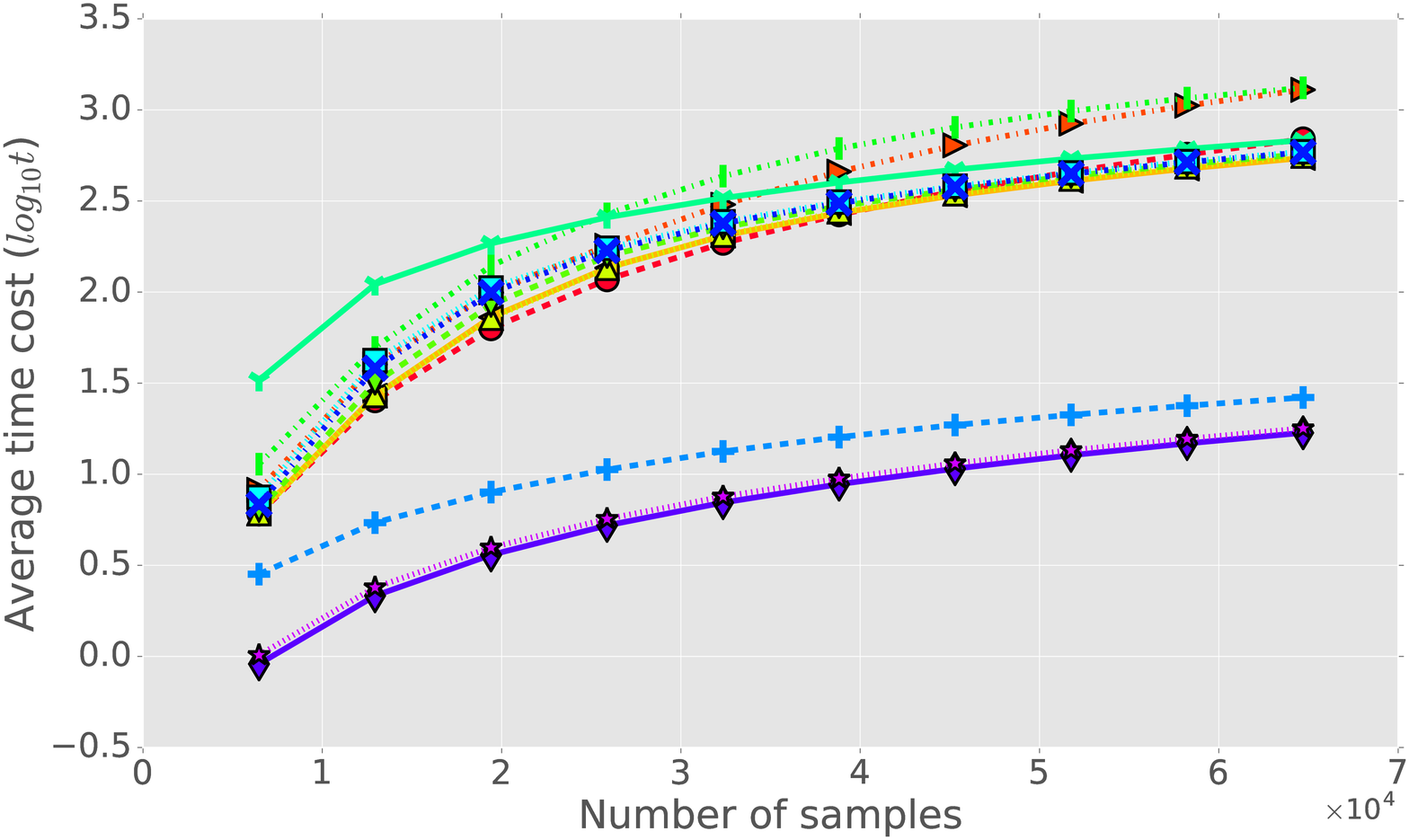}}\vspace{-2mm}
\par\end{centering}
\noindent \begin{centering}
\subfloat[cod-rna]{\noindent \centering{}\includegraphics[width=0.45\textwidth]{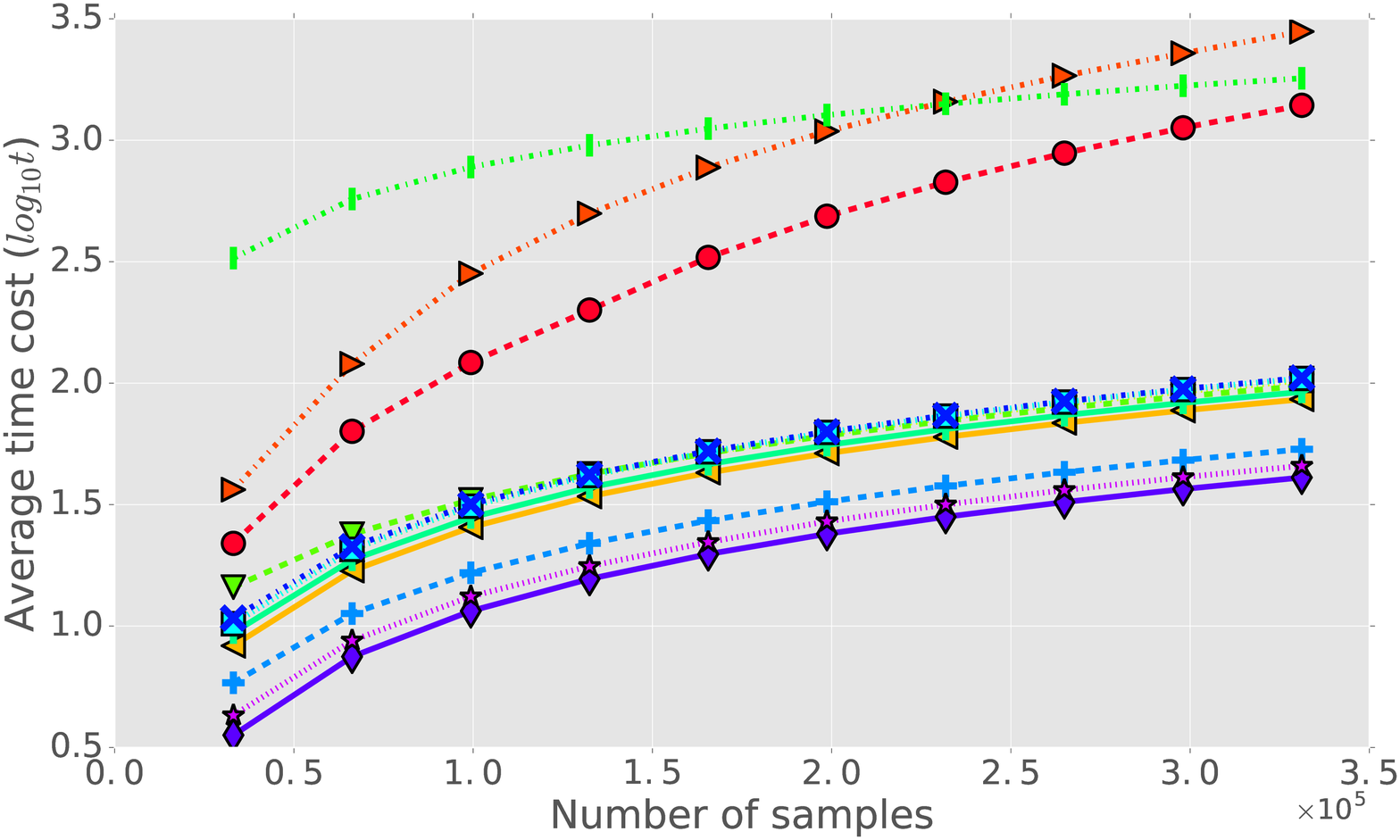}}\hspace{0.04\textwidth}\subfloat[ijcnn1]{\noindent \centering{}\includegraphics[width=0.45\textwidth]{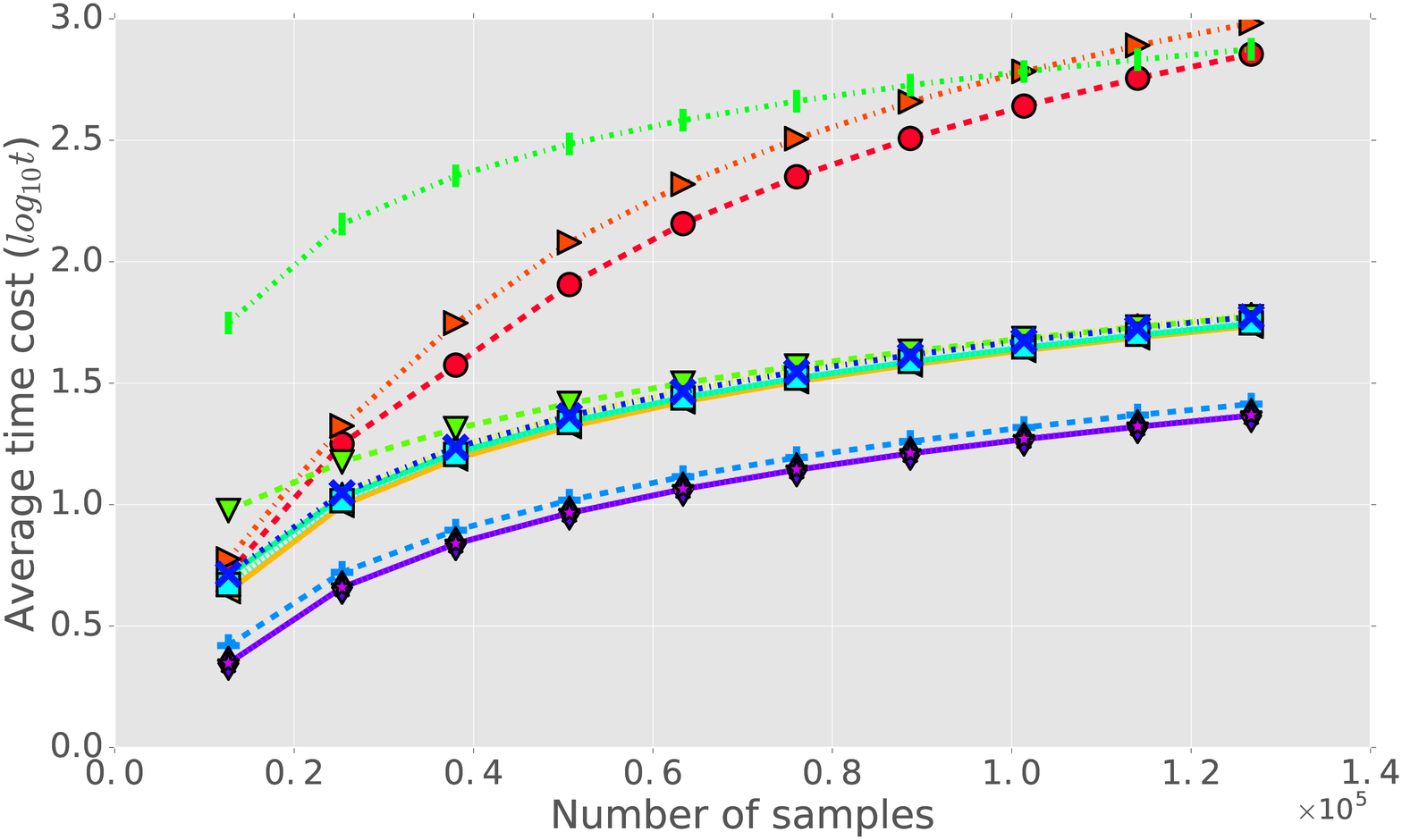}}\vspace{-2mm}
\par\end{centering}
\noindent \begin{centering}
\subfloat[KDDCup99]{\noindent \centering{}\includegraphics[width=0.45\textwidth]{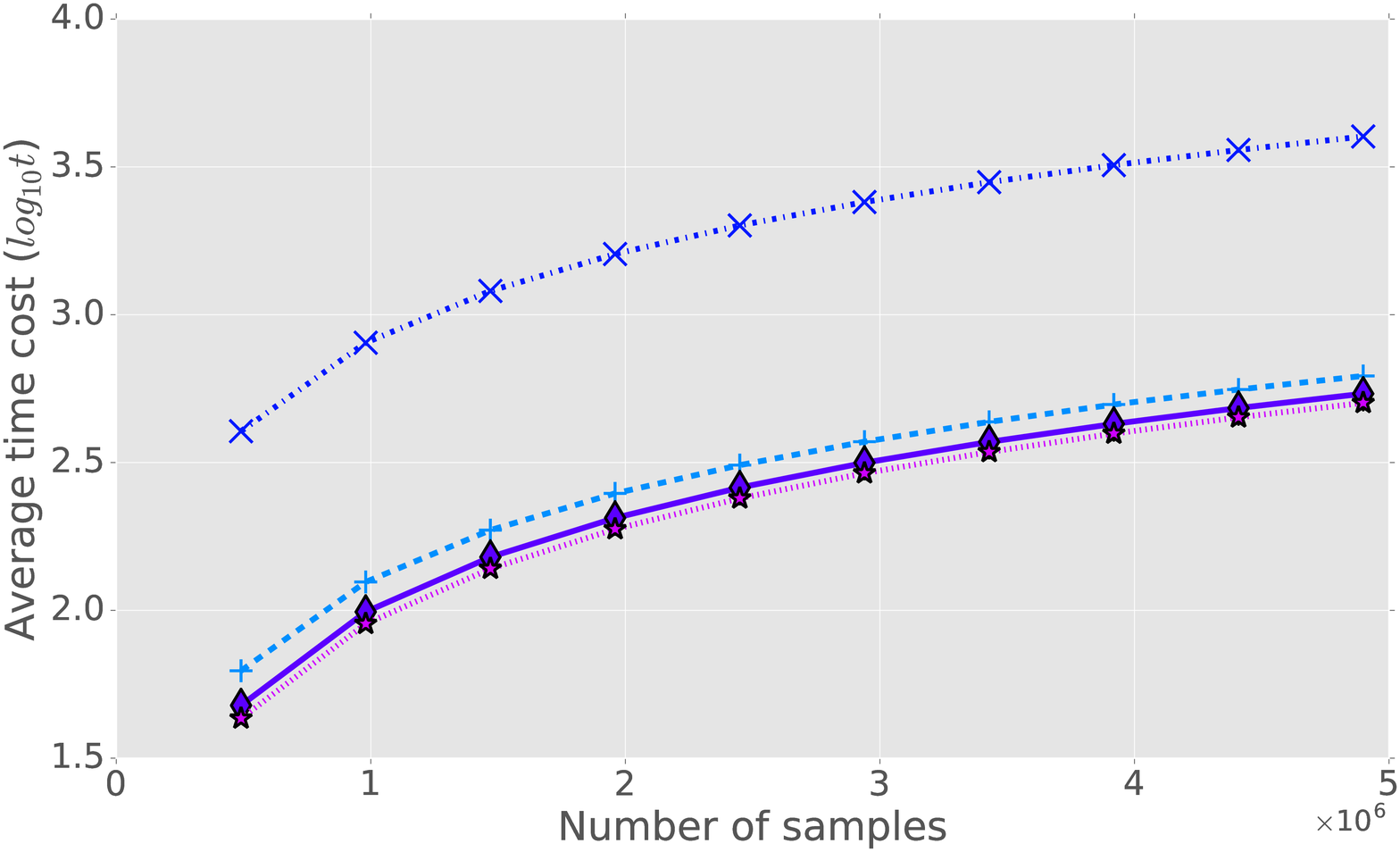}}\hspace{0.04\textwidth}\subfloat[covtype]{\noindent \centering{}\includegraphics[width=0.45\textwidth]{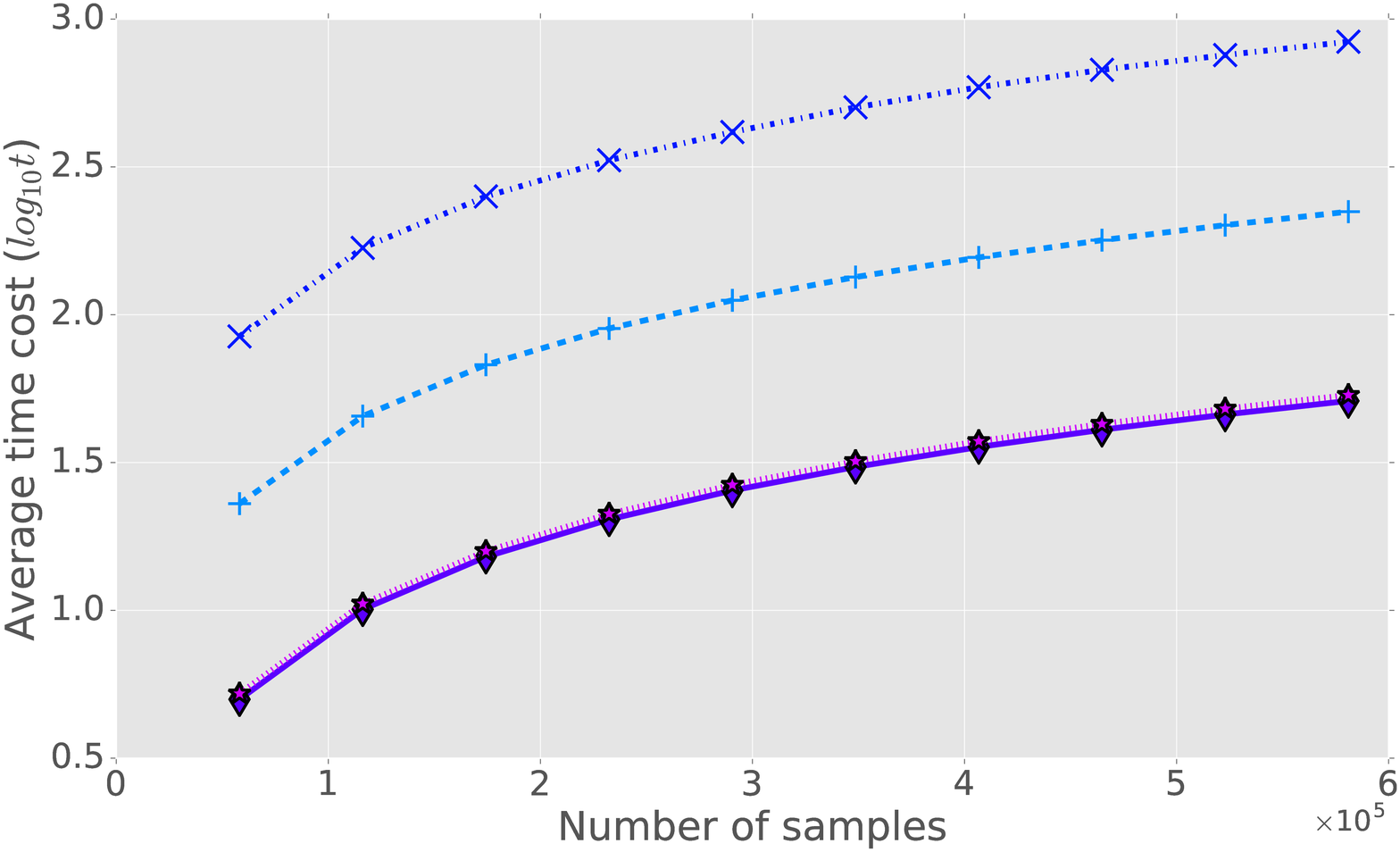}}\vspace{-2mm}
\par\end{centering}
\noindent \begin{centering}
\subfloat[poker]{\noindent \centering{}\includegraphics[width=0.45\textwidth]{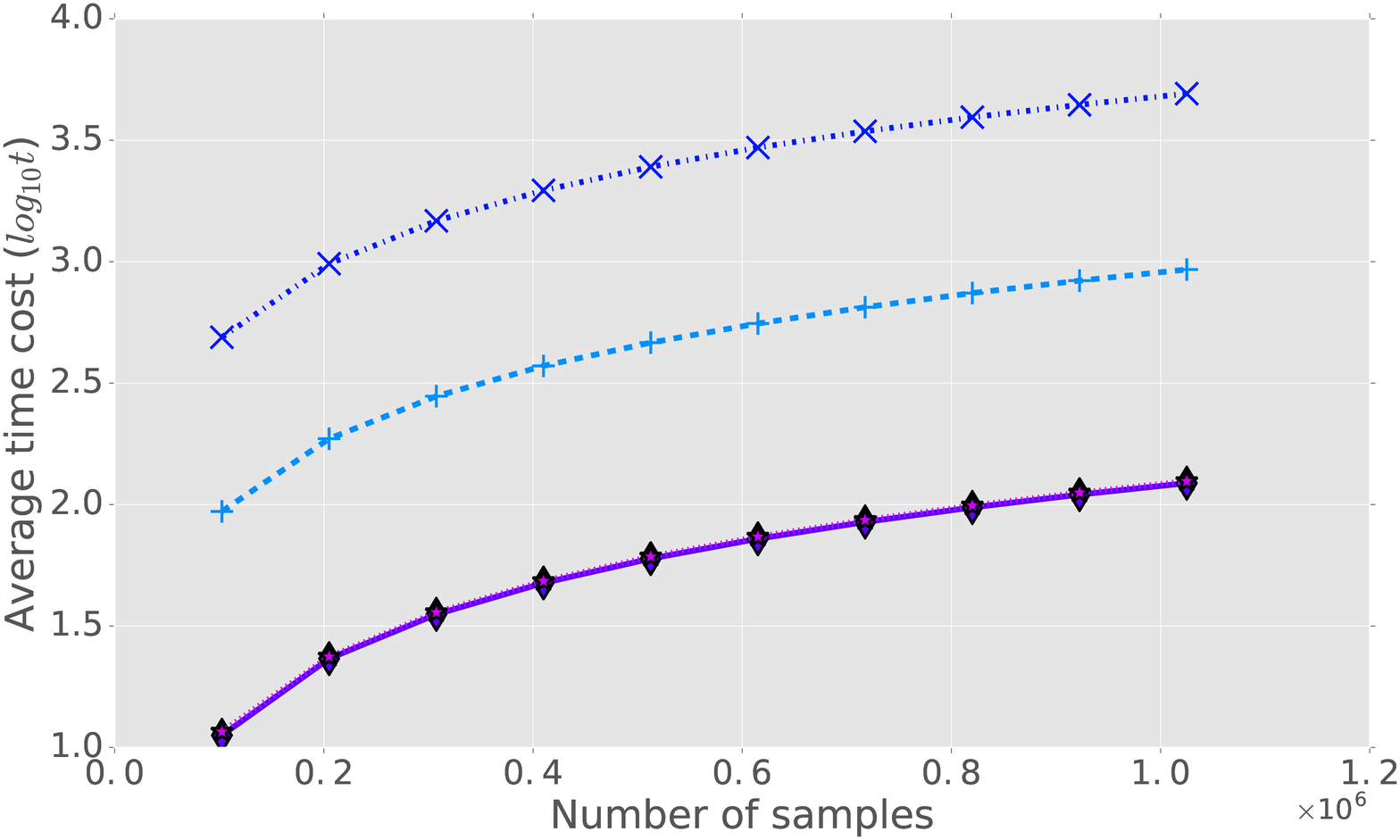}}\hspace{0.04\textwidth}\subfloat[airlines]{\noindent \centering{}\includegraphics[width=0.45\textwidth]{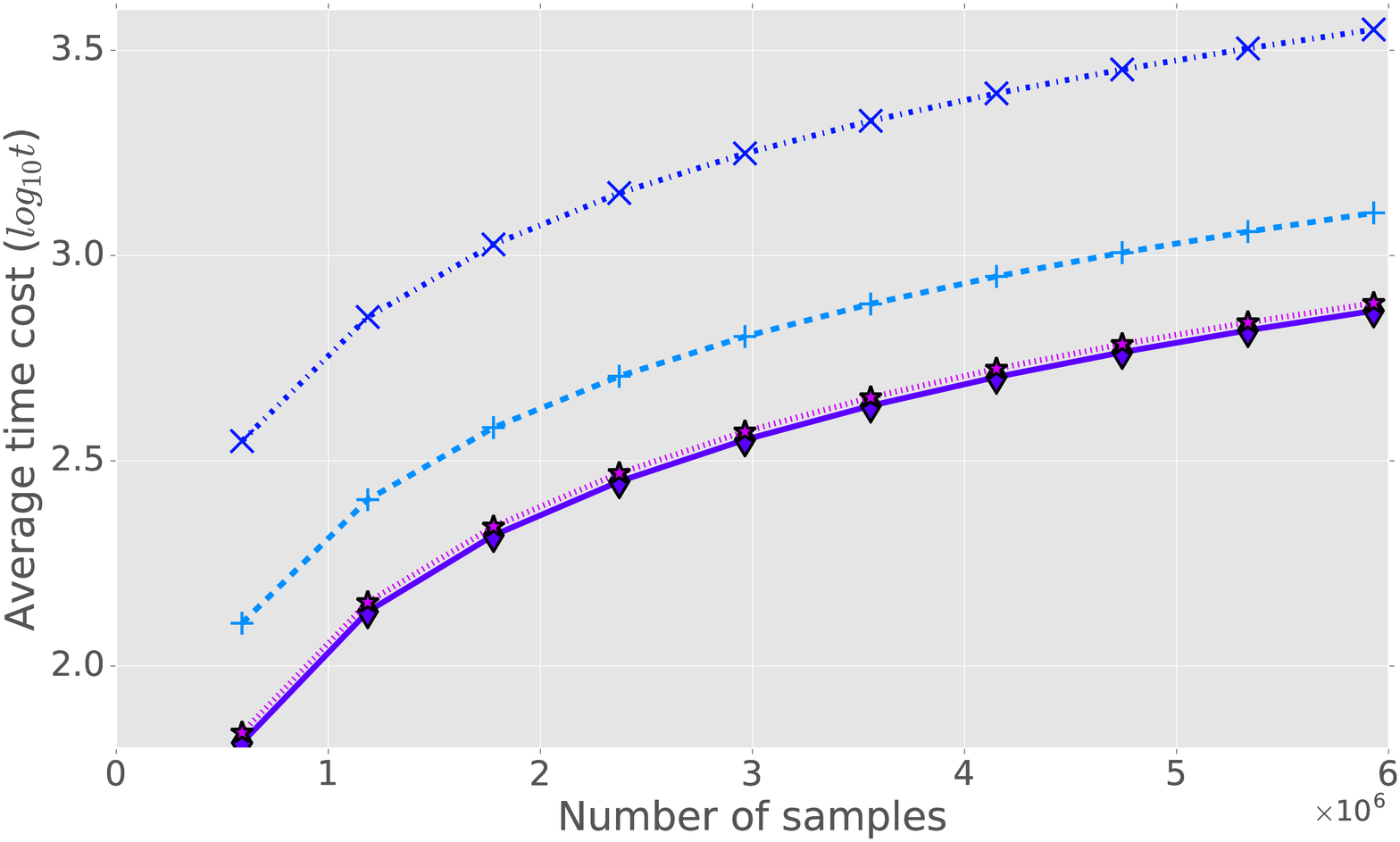}}
\par\end{centering}
\noindent \centering{}\caption{Convergence evaluation of online classification task: the average
time costs (seconds shown in the logarithm with base 10) as a function
of the number of samples seen by the models. (Best viewed in colors).\label{fig:exp_online_classification_time_cost}}
\end{figure}

\noindent 
\begin{figure}[H]
\noindent \begin{centering}
\includegraphics[width=0.8\textwidth]{figs/legend}
\par\end{centering}
\noindent \begin{centering}
\subfloat[a9a]{\noindent \centering{}\includegraphics[width=0.45\textwidth]{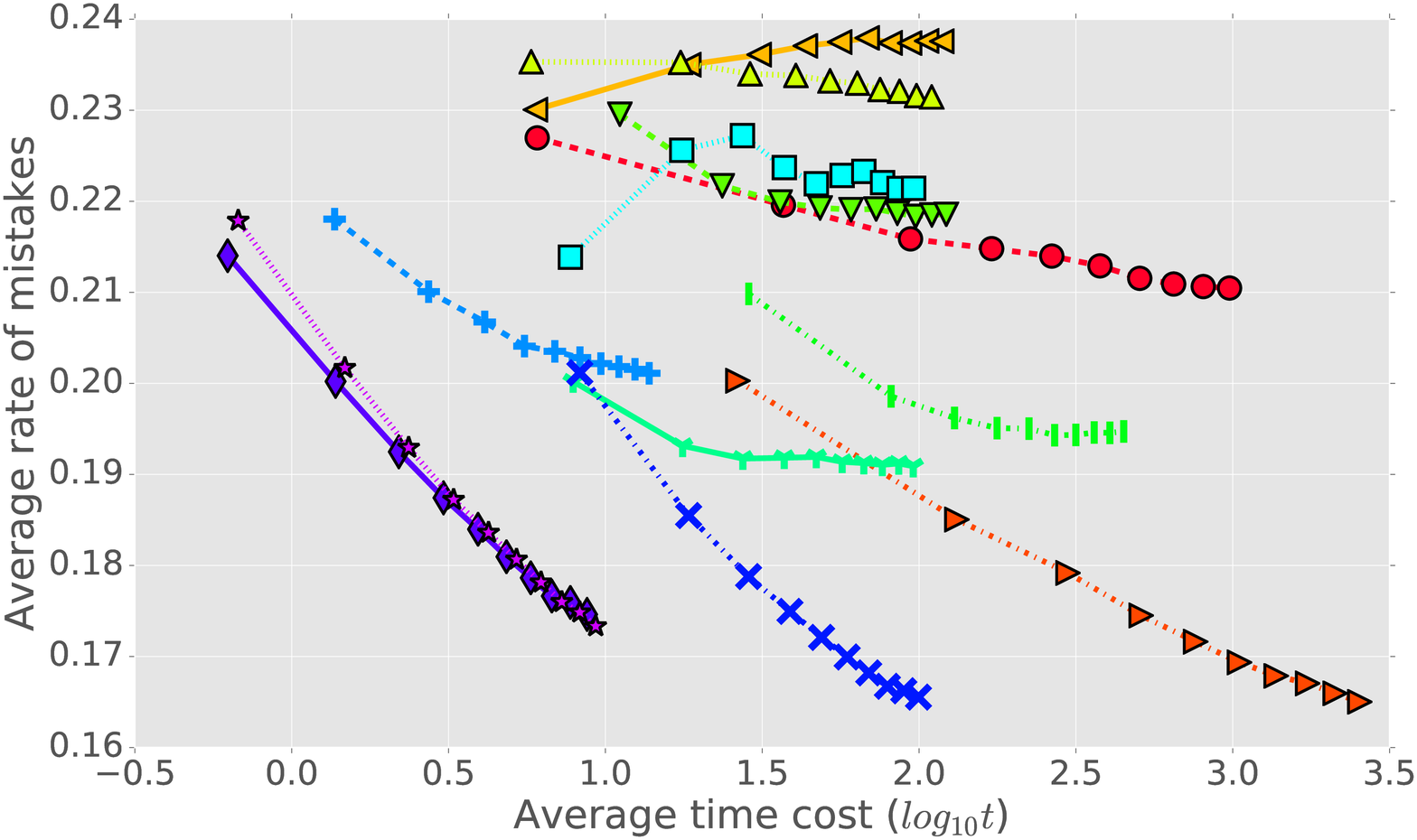}}\hspace{0.04\textwidth}\subfloat[w8a]{\noindent \centering{}\includegraphics[width=0.45\textwidth]{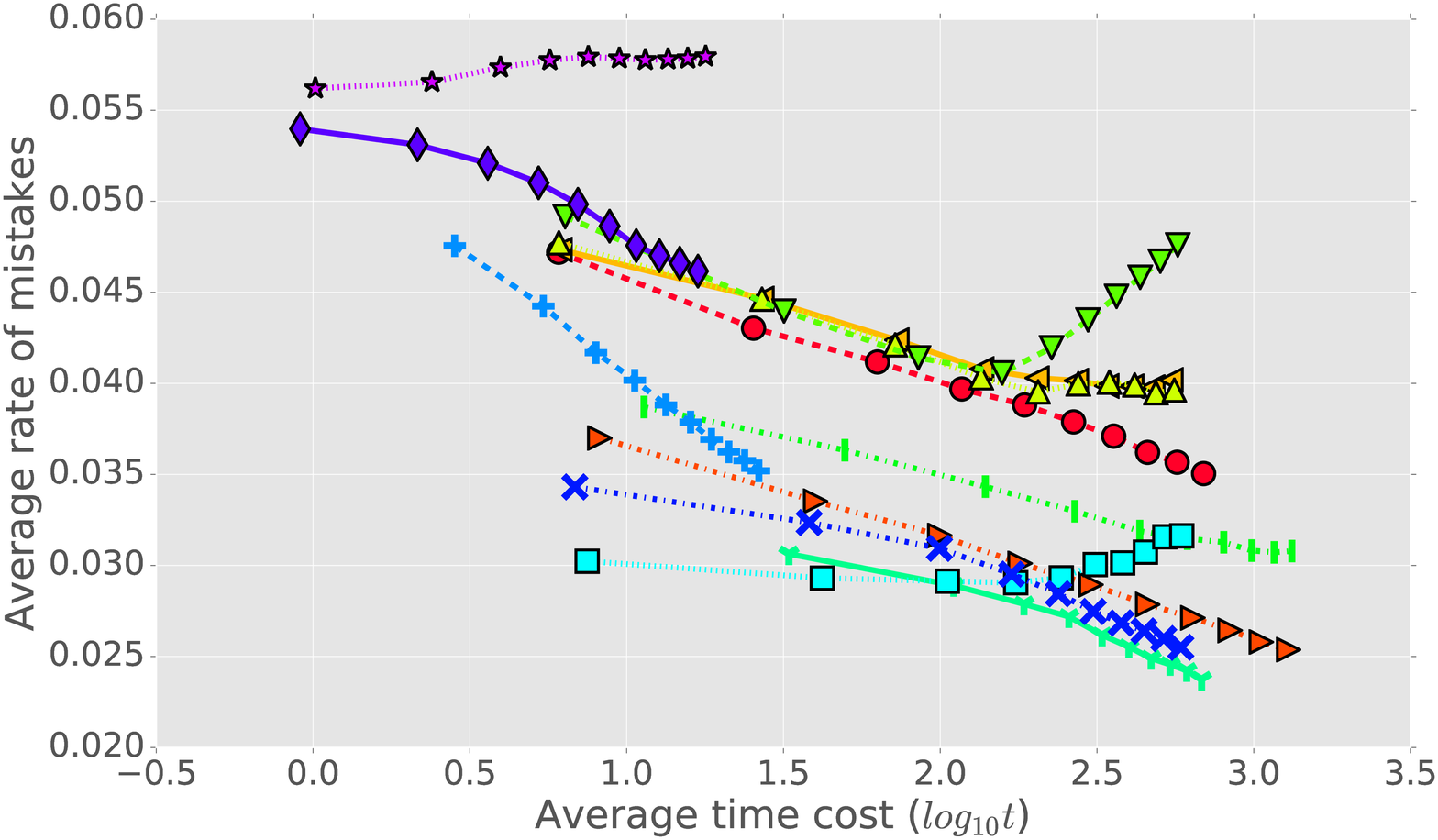}}\vspace{-2mm}
\par\end{centering}
\noindent \begin{centering}
\subfloat[cod-rna]{\noindent \centering{}\includegraphics[width=0.45\textwidth]{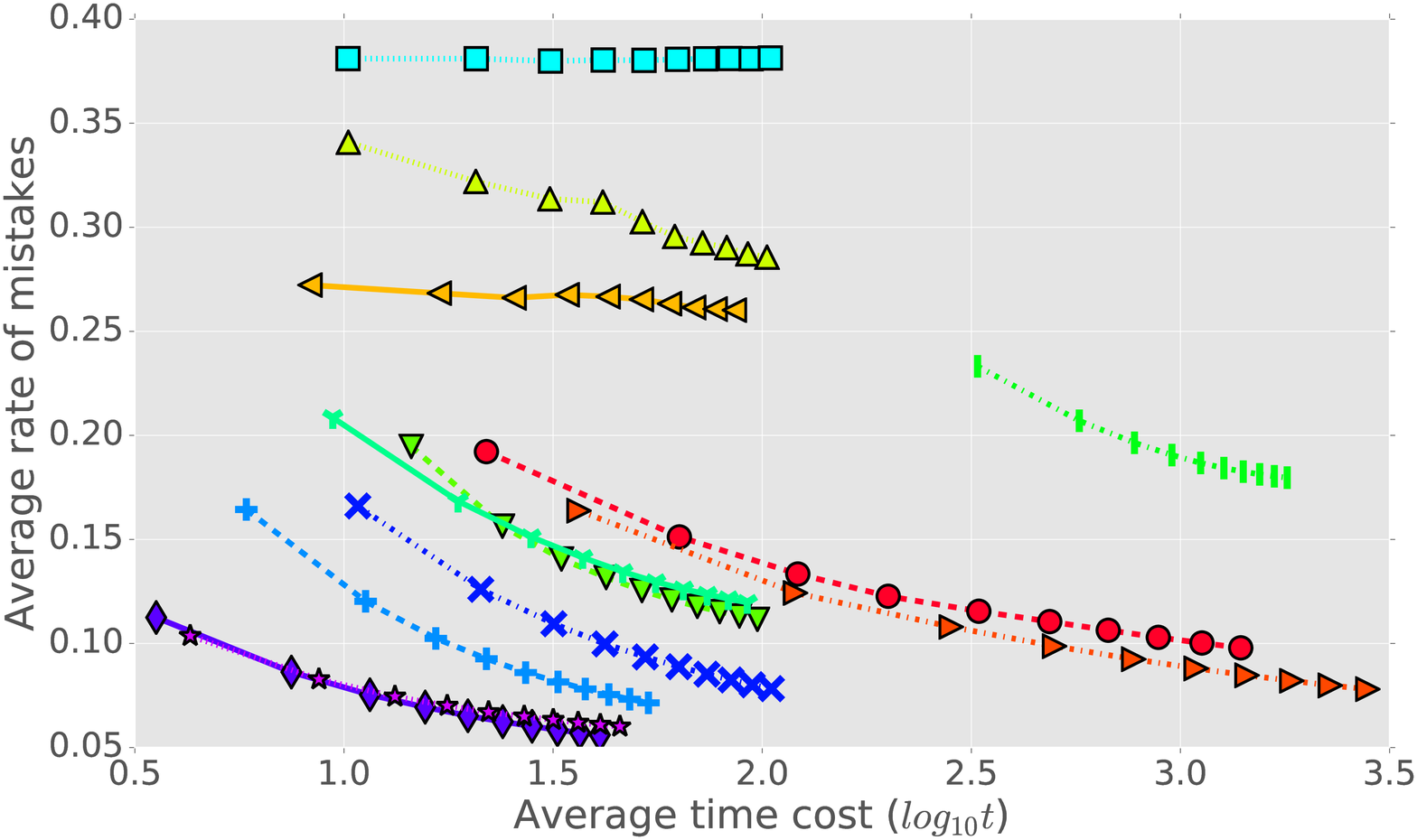}}\hspace{0.04\textwidth}\subfloat[ijcnn1]{\noindent \centering{}\includegraphics[width=0.45\textwidth]{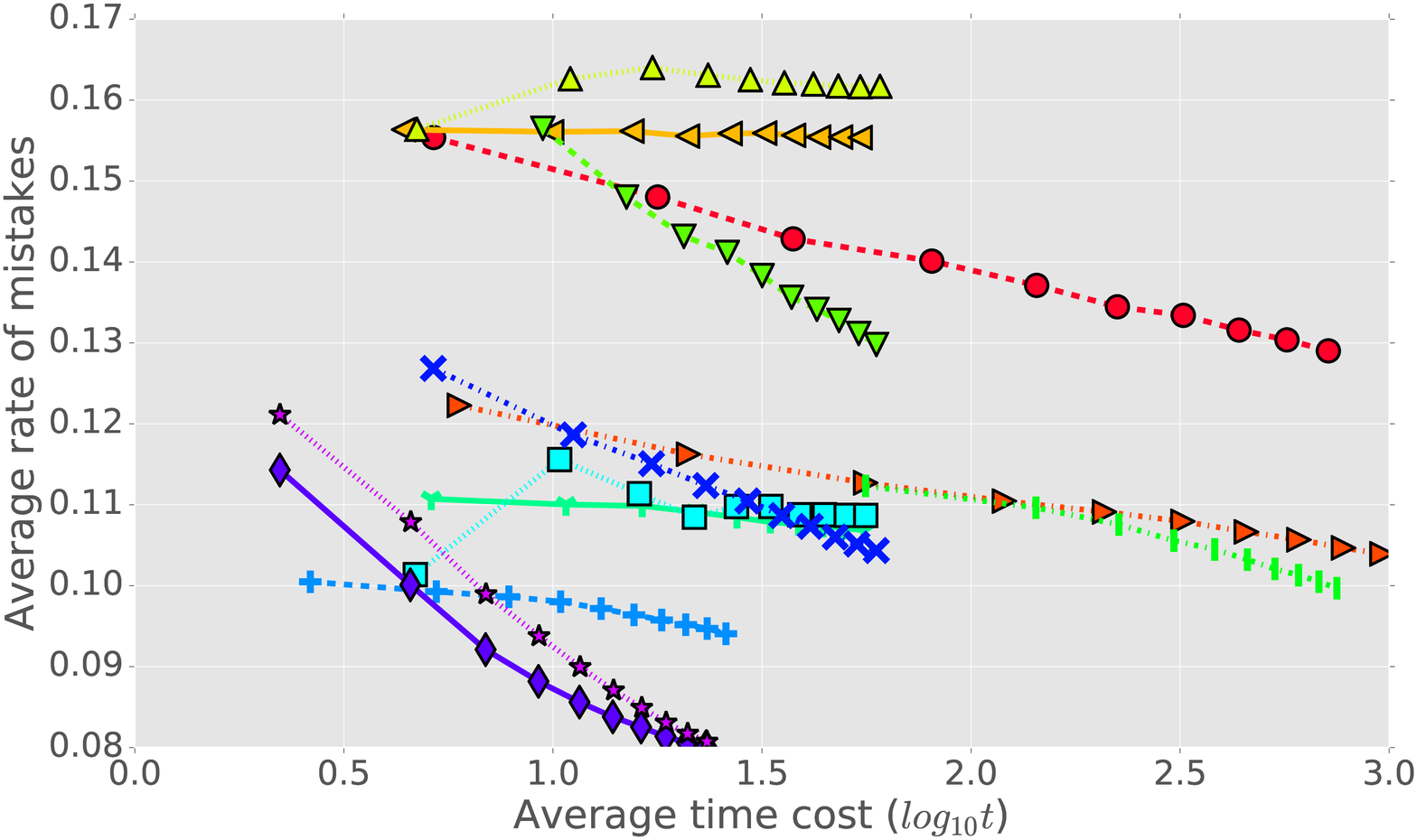}}\vspace{-2mm}
\par\end{centering}
\noindent \begin{centering}
\subfloat[KDDCup99]{\noindent \centering{}\includegraphics[width=0.45\textwidth]{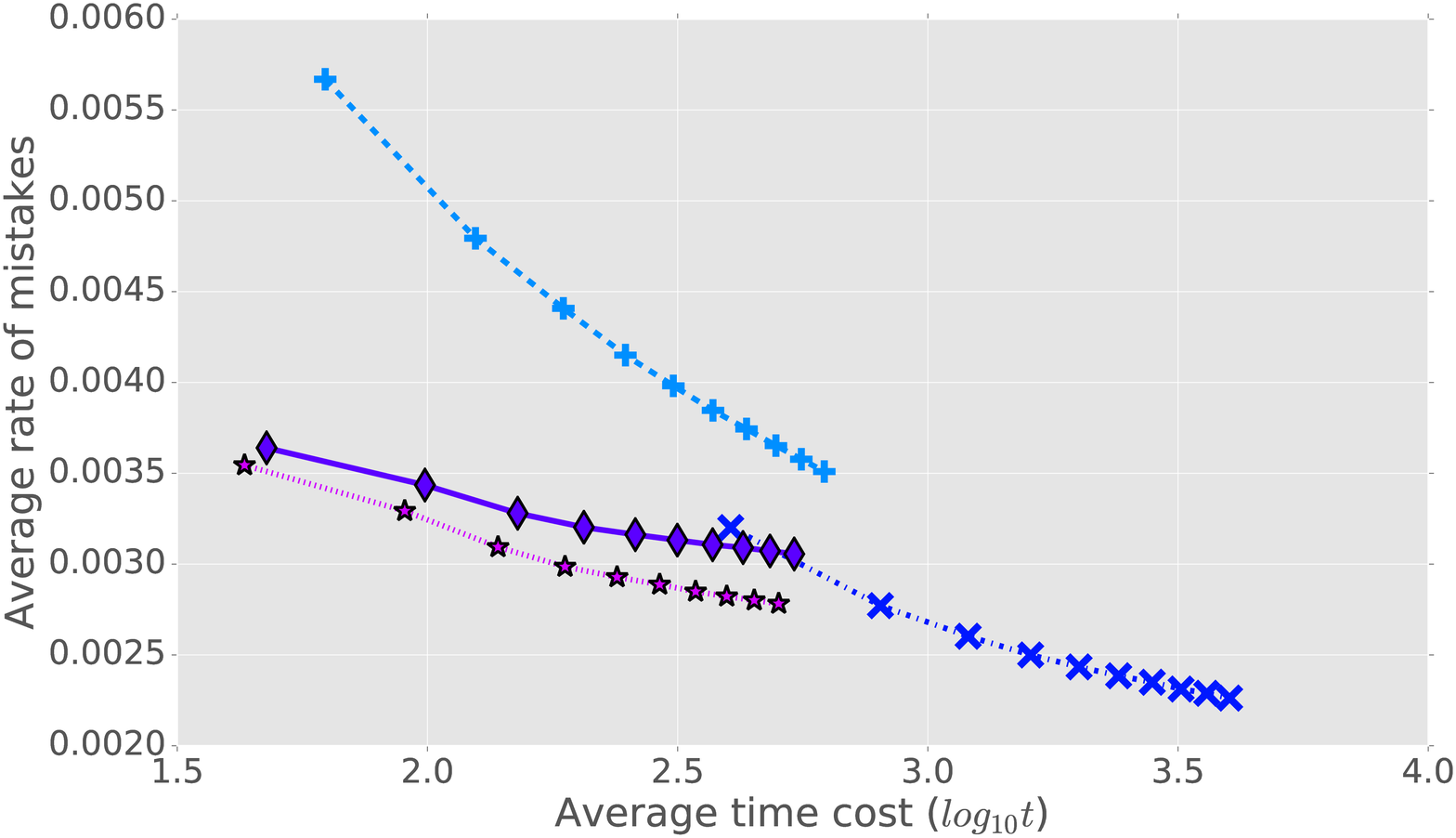}}\hspace{0.04\textwidth}\subfloat[covtype]{\noindent \centering{}\includegraphics[width=0.45\textwidth]{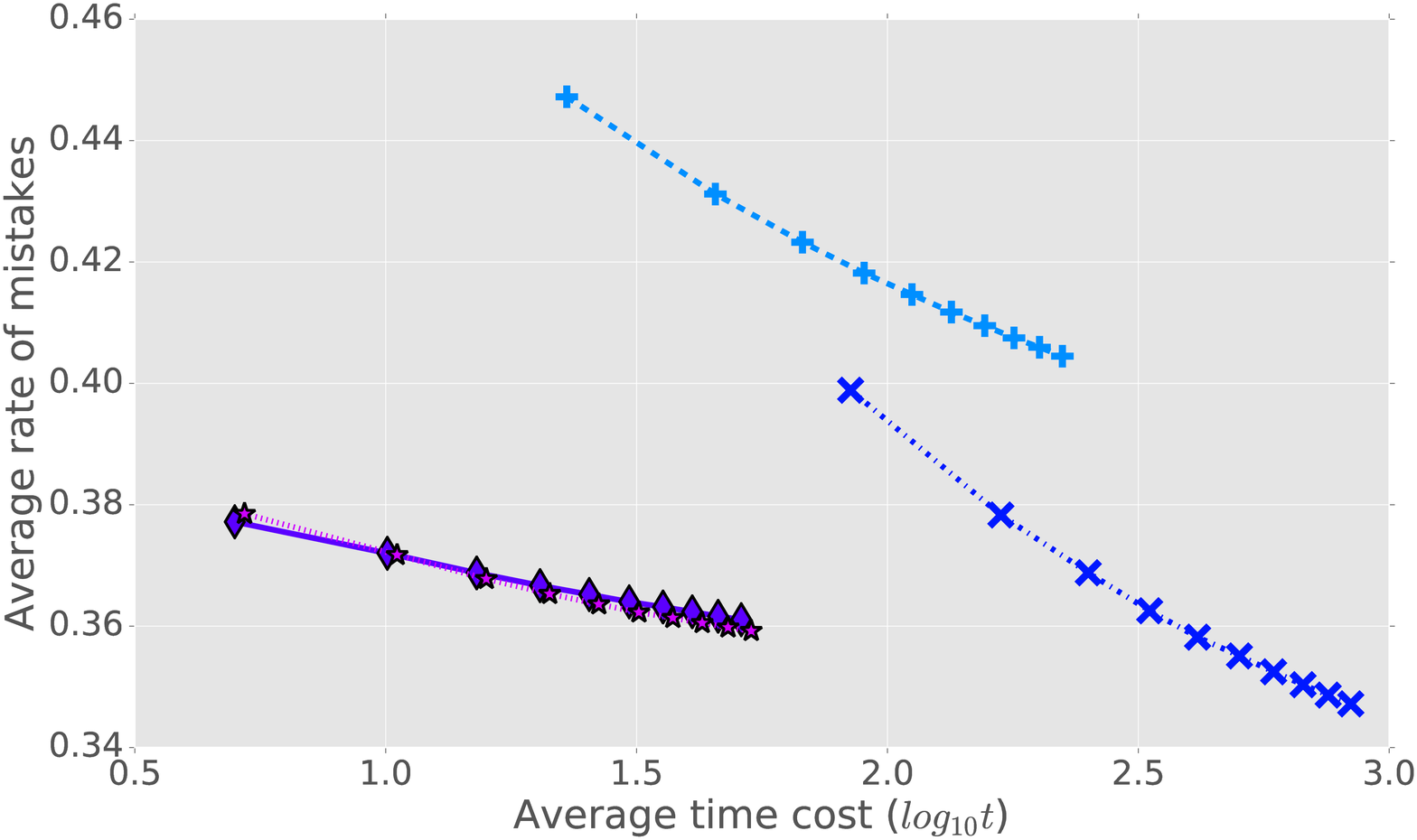}}\vspace{-2mm}
\par\end{centering}
\noindent \begin{centering}
\subfloat[poker]{\noindent \centering{}\includegraphics[width=0.45\textwidth]{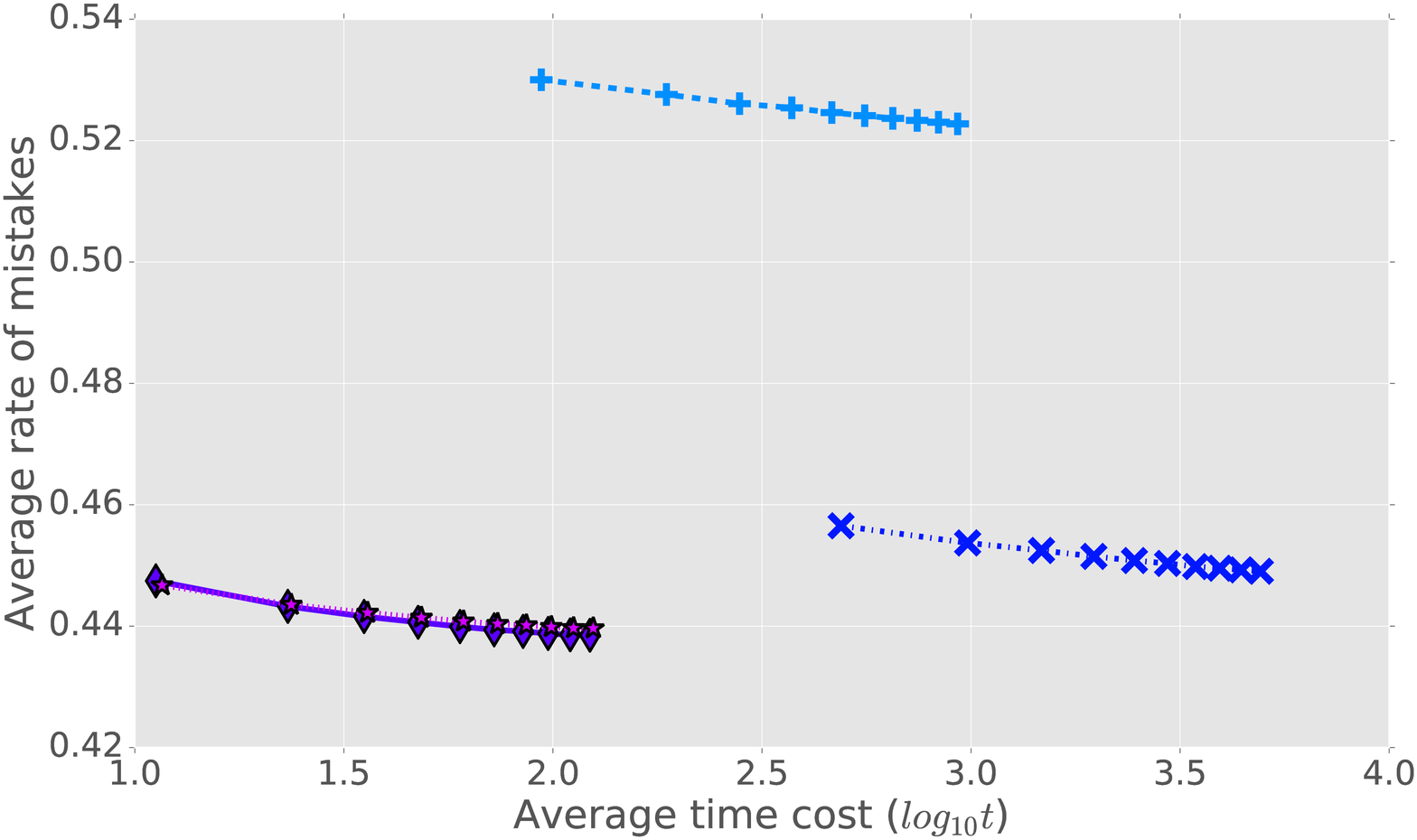}}\hspace{0.04\textwidth}\subfloat[airlines]{\noindent \centering{}\includegraphics[width=0.45\textwidth]{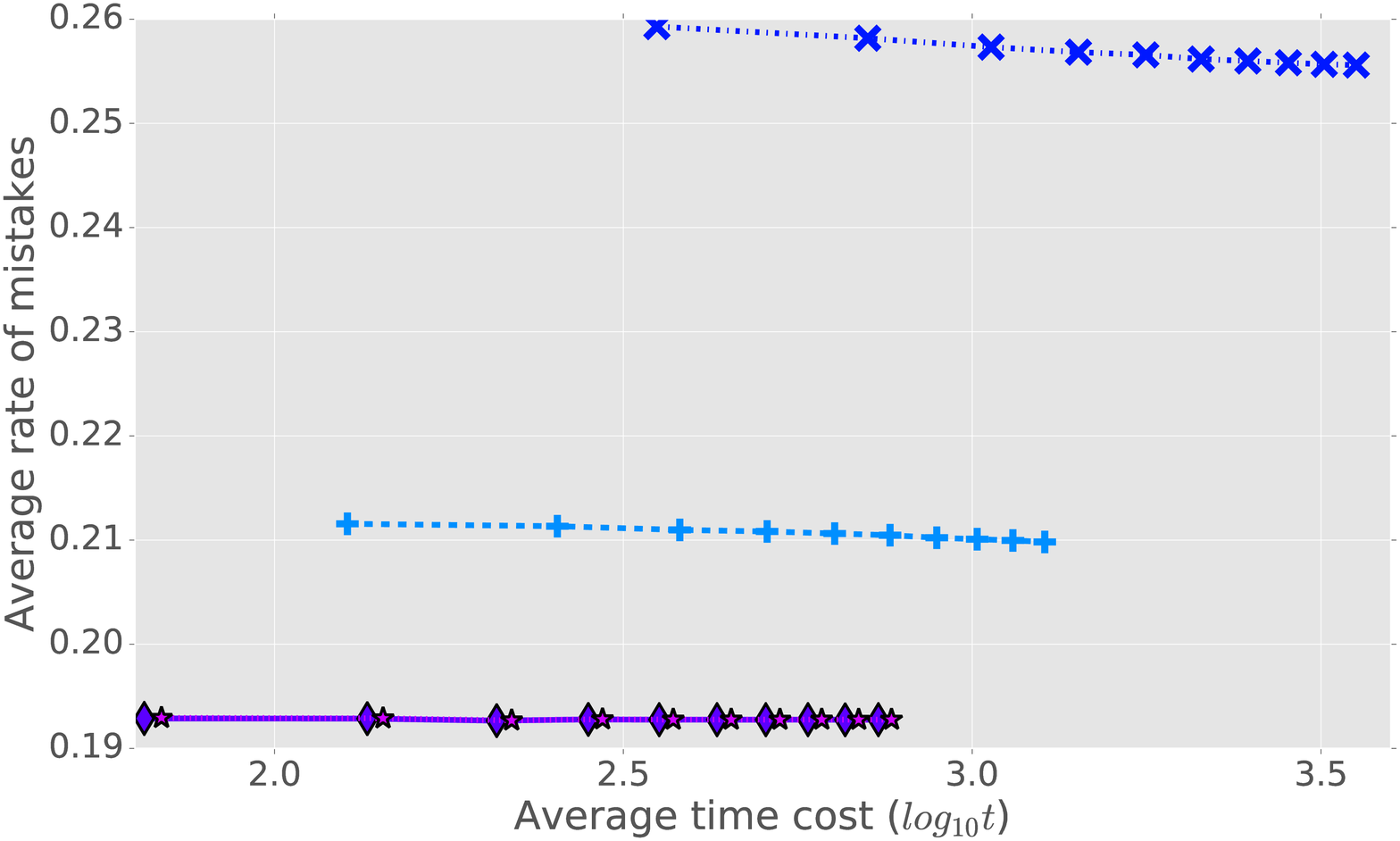}}
\par\end{centering}
\noindent \centering{}\caption{Average mistake rate vs. time cost for online classification. The
average time (seconds) is shown in the logarithm with base 10. (Best
viewed in colors).\label{fig:exp_online_classification_mistake_vs_time}}
\end{figure}

\noindent 
\begin{table}[H]
\noindent \begin{centering}
\caption{Classification performance of our proposed methods and the baselines
in online mode. Note that $\delta$, $B$ and $D$ are set to be the
same as in batch classification tasks (cf., Section~\ref{subsec:exp_batch_classification}).
The mistake rate is reported in percent $\left(\%\right)$ and the
execution time is in second. The best performance is in \textbf{bold}.\label{tab:exp_online_classification}}
\par\end{centering}
\noindent \centering{}%
\begin{tabular}{|c|r|r|r|r|}
\hline 
\textbf{\emph{\cellcolor{header_color}Dataset $\left[S\right]$}} & \multicolumn{2}{c|}{\textbf{\emph{\cellcolor{header_color}a9a $\left[142\right]$}}} & \multicolumn{2}{c|}{\textbf{\emph{\cellcolor{header_color}w8a $\left[131\right]$}}}\tabularnewline
\hline 
\textbf{\emph{\cellcolor{header_color}Algorithm}} & Mistake Rate & Time & Mistake Rate & Time\tabularnewline
\hline 
Perceptron & 21.05$\pm$0.12 & 976.79 & 3.51$\pm$0.03 & 691.80\tabularnewline
OGD & \textbf{16.50$\pm$0.06} & 2,539.46 & 2.54$\pm$0.03 & 1,290.13\tabularnewline
\hline 
RBP & 23.76$\pm$0.21 & 118.25 & 4.02$\pm$0.07 & 544.83\tabularnewline
Forgetron & 23.15$\pm$0.34 & 109.71 & 3.96$\pm$0.10 & 557.75\tabularnewline
Projectron & 21.86$\pm$1.73 & 122.08 & 4.76$\pm$1.13 & 572.20\tabularnewline
Projectron++ & 19.47$\pm$2.22 & 449.20 & 3.08$\pm$0.63 & 1321.93\tabularnewline
BPAS & 19.09$\pm$0.17 & 95.81 & \textbf{2.37$\pm$0.02} & 681.46\tabularnewline
BOGD & 22.14$\pm$0.25 & 96.11 & 3.16$\pm$0.08 & 589.47\tabularnewline
FOGD & 20.11$\pm$0.10 & 13.79 & 3.52$\pm$0.05 & 26.40\tabularnewline
NOGD & 16.55$\pm$0.07 & 99.54 & 2.55$\pm$0.05 & 585.23\tabularnewline
\hline 
$\model$-Hinge & 17.46$\pm$0.12 & \textbf{8.74} & 4.62$\pm$0.78 & \textbf{16.89}\tabularnewline
$\model$-Logit & 17.33$\pm$0.16 & 9.31 & 5.80$\pm$0.02 & 17.86\tabularnewline
\hline 
\hline 
\textbf{\emph{\cellcolor{header_color}Dataset $\left[S\right]$}} & \multicolumn{2}{c|}{\textbf{\emph{\cellcolor{header_color}cod-rna $\left[436\right]$}}} & \multicolumn{2}{c|}{\textbf{\emph{\cellcolor{header_color}ijcnn1 $\left[500\right]$}}}\tabularnewline
\hline 
\textbf{\emph{\cellcolor{header_color}Algorithm}} & Mistake Rate & Time & Mistake Rate & Time\tabularnewline
\hline 
Perceptron & 9.79$\pm$0.04 & 1,393.56 & 12.85$\pm$0.09 & 727.90\tabularnewline
OGD & 7.81$\pm$0.03 & 2,804.01 & 10.39$\pm$0.06 & 960.44\tabularnewline
\hline 
RBP & 26.02$\pm$0.39 & 85.84 & 15.54$\pm$0.21 & 54.29\tabularnewline
Forgetron & 28.56$\pm$2.22 & 102.64 & 16.17$\pm$0.26 & 60.54\tabularnewline
Projectron & 11.16$\pm$3.61 & 97.38 & 12.98$\pm$0.23 & 59.37\tabularnewline
Projectron++ & 17.97$\pm$15.60 & 1,799.93 & 9.97$\pm$0.09 & 749.70\tabularnewline
BPAS & 11.97$\pm$0.09 & 92.08 & 10.68$\pm$0.05 & 55.44\tabularnewline
BOGD & 38.13$\pm$0.11 & 104.60 & 10.87$\pm$0.18 & 55.99\tabularnewline
FOGD & 7.15$\pm$0.03 & 53.45 & 9.41$\pm$0.03 & 25.93\tabularnewline
NOGD & 7.83$\pm$0.06 & 105.18 & 10.43$\pm$0.08 & 59.36\tabularnewline
\hline 
$\model$-Hinge & \textbf{5.61$\pm$0.17} & \textbf{40.89} & \textbf{8.01$\pm$0.18} & \textbf{23.26}\tabularnewline
$\model$-Logit & 6.01$\pm$0.20 & 45.67 & 8.07$\pm$0.20 & 23.36\tabularnewline
\hline 
\hline 
\textbf{\emph{\cellcolor{header_color}Dataset $\left[S\right]$}} & \multicolumn{2}{c|}{\textbf{\emph{\cellcolor{header_color}KDDCup99 $\left[115\right]$}}} & \multicolumn{2}{c|}{\textbf{\emph{\cellcolor{header_color}covtype $\left[59\right]$}}}\tabularnewline
\hline 
\textbf{\emph{\cellcolor{header_color}Algorithm}} & Mistake rate & Time & Mistake rate & Time\tabularnewline
\hline 
FOGD & 0.35$\pm$0.00 & 620.95 & 40.45$\pm$0.05 & 223.20\tabularnewline
NOGD & \textbf{0.23$\pm$0.00} & 4,009.03 & \textbf{34.72}$\pm$\textbf{0.07} & 838.47\tabularnewline
\hline 
$\model$-Hinge & 0.31$\pm$0.07 & 540.65 & 36.11$\pm$0.16 & \textbf{51.12}\tabularnewline
$\model$-Logit & 0.28$\pm$0.03 & \textbf{503.34} & 35.92$\pm$0.16 & 53.51\tabularnewline
\hline 
\hline 
\textbf{\emph{\cellcolor{header_color}Dataset $\left[S\right]$}} & \multicolumn{2}{c|}{\textbf{\emph{\cellcolor{header_color}poker $\left[393\right]$}}} & \multicolumn{2}{c|}{\textbf{\emph{\cellcolor{header_color}airlines $\left[388\right]$}}}\tabularnewline
\hline 
\textbf{\emph{\cellcolor{header_color}Algorithm}} & Mistake Rate & Time & Mistake Rate & Time\tabularnewline
\hline 
FOGD & 52.28$\pm$0.04 & 928.89 & 20.98$\pm$0.01 & 1,270.75\tabularnewline
NOGD & 44.90$\pm$0.16 & 4,920.33 & 25.56$\pm$0.01 & 3,553.50\tabularnewline
\hline 
$\model$-Hinge & \textbf{43.85$\pm$0.09} & \textbf{122.59} & \textbf{19.28}$\pm$\textbf{0.00} & \textbf{733.72}\tabularnewline
$\model$-Logit & 43.97\textbf{$\pm$}0.07 & 124.86 & \textbf{19.28}$\pm$\textbf{0.00} & 766.19\tabularnewline
\hline 
\end{tabular}
\end{table}

For classification capability, the non-budgeted methods only surpass
the budgeted ones for the smallest dataset, that is, the OGD obtains
the best performance for \emph{a9a} data. This again demonstrates
the importance of exploring budget online kernel learning algorithms.
Between the two non-budgeted algorithms, the OGD achieves considerably
better error rates than the Perceptron. The method, however, must
perform much more expensive updates, resulting in a significantly
larger number of support vectors and significantly higher computational
time costs. This represents the trade-off between classification accuracy
and computational complexity of the OGD.

Furthermore, comparing the performance of different existing budgeted
online kernel learning algorithms, the $\model$-Hinge and $\model$-Logit
outperform others in both discriminative performance and computation
efficiency for almost all datasets. In particular, the $\model$-based
methods achieve the best mistake rates \textendash{} 5.61$\pm$0.17,
8.01$\pm$0.18, 43.85$\pm$0.09, 19.28$\pm$0.00 for the \emph{cod-rna},
\emph{ijcnn1}, \emph{poker} and \emph{airlines} data, that are, respectively,
$27.5\%$, $17.5\%$, $2.4\%$, $8.8\%$ lower than the error rates
of the second best models \textendash{} two recent approaches FOGD
and NOGD. On the other hand, the computation costs of the $\model$s
are significantly lower with large margins of hundreds of percents
for large-scale databases \emph{covtype}, \emph{poker}, and \emph{airlines}
as shown in Table~\ref{tab:exp_online_classification}.

In all experiments, our proposed method produces the model sizes that
are much smaller than the budget sizes of baseline methods. Thus we
further investigate the performance of the budgeted baselines by varying
the budget size $B$, and compare with our $\model$ with Hinge loss.
Fig.~\ref{fig:exp_online_classification_mistake_time} shows our
analysis on two datasets \emph{a9a} and \emph{cod-rna}. It can be
seen that the larger $B$ helps model obtain better classification
results, but hurts their running speed. For both datasets, the budgeted
baselines with larger budget sizes still fail to beat the predictive
performance of $\model$. On the other hand, the baselines with smaller
budget sizes run faster than the $\model$ on \emph{cod-rna} dataset,
but slower on \emph{a9a} dataset.

\noindent 
\begin{figure}[H]
\noindent \begin{centering}
\subfloat[a9a]{\noindent \centering{}\includegraphics[width=0.48\textwidth]{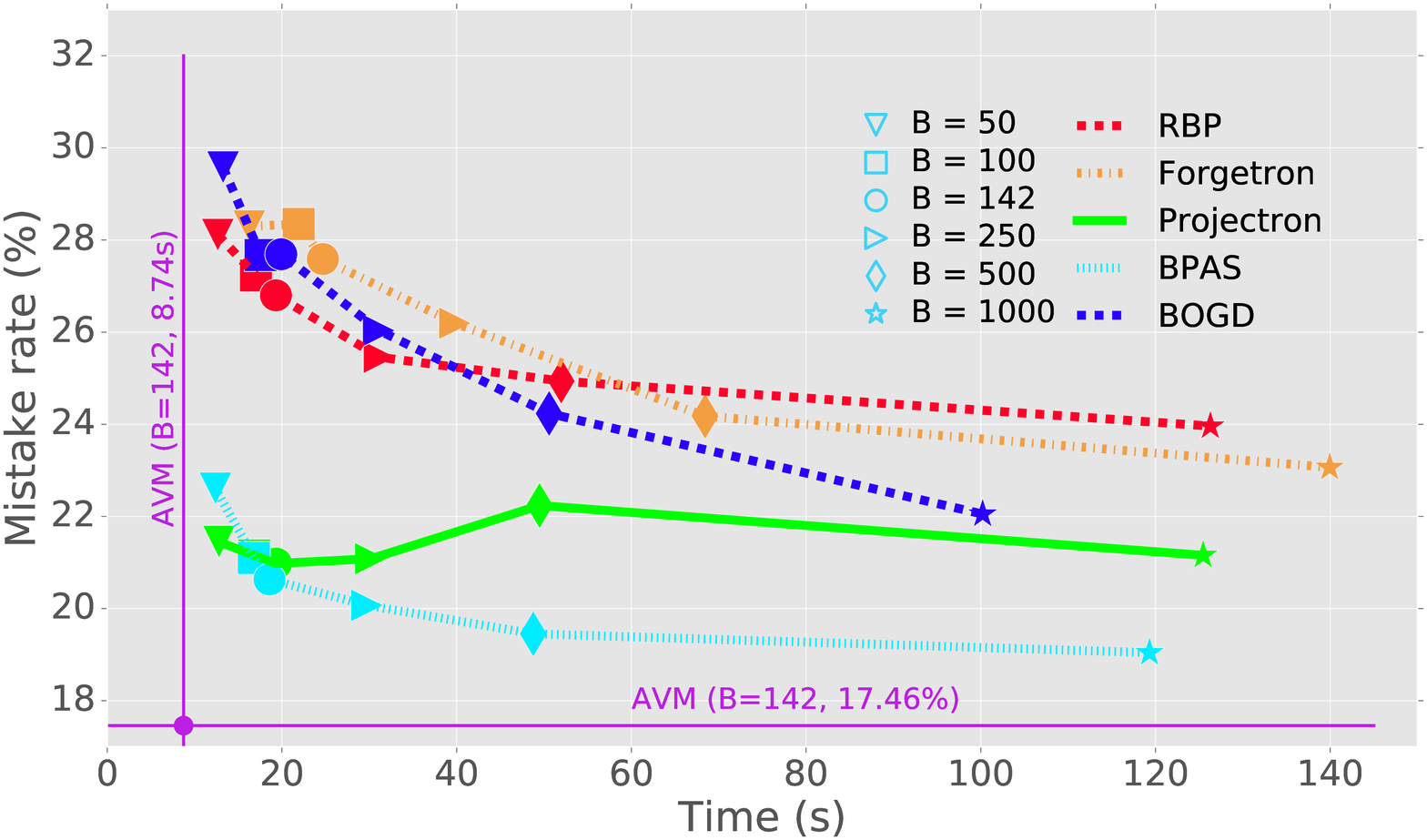}}\hspace{0.02\textwidth}\subfloat[cod-rna]{\noindent \centering{}\includegraphics[width=0.48\textwidth]{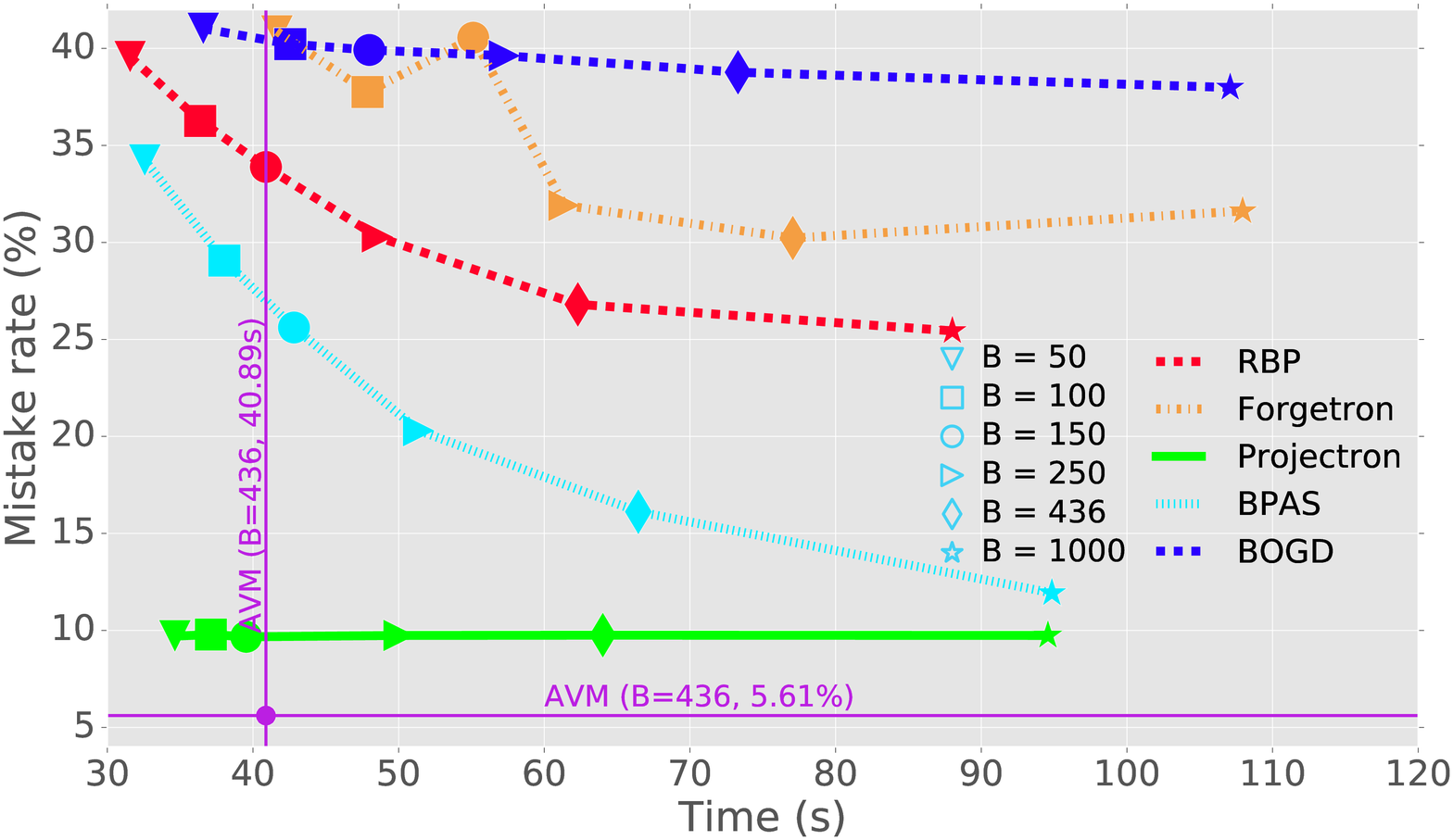}}
\par\end{centering}
\noindent \centering{}\caption{Predictive and wall-clock performance on two datasets: \emph{a9a}
and \emph{cod-rna} of budgeted methods when the budget size $B$ is
varied. (Best viewed in colors).\label{fig:exp_online_classification_mistake_time}}
\end{figure}

Finally, two versions of $\model$s demonstrate similar discriminative
performances and computational complexities wherein the $\model$-Logit
is slightly slower due to the additional exponential operators as
also seen in batch classification task. All aforementioned observations
validate the effectiveness and efficiency of our proposed technique.
Thus, we believe that our approximation machine is a promising technique
for building scalable online kernel learning algorithms for large-scale
classification tasks.

\subsection{Online regression}

The last experiment addresses the online regression problem to evaluate
the capabilities of our approach with three proposed loss functions
\textendash{} $\ell_{1}$,$\ell_{2}$ and $\boldsymbol{\varepsilon}$-insensitive
losses as described in Section~\ref{sec:Loss-Function}. Incorporating
these loss functions creates three versions: $\model$-$\boldsymbol{\varepsilon}$,
$\model$-$\ell_{1}$ and $\model$-$\ell_{2}$. We use four datasets:
\emph{casp}, \emph{slice}, \emph{year} and \emph{airlines} (delay
minutes) with a wide range of sizes for this task. We recruit six
baselines: RBP, Forgetron, Projectron, BOGD, FOGD and NOGD (cf. more
detailed description in Section~\ref{subsec:exp_online_classification}).

\paragraph{Hyperparameters setting.}

We adopt the same hyperparameter searching procedure for batch classification
task as in Section~\ref{subsec:exp_batch_classification}. Furthermore,
for the budget size $B$ and the feature dimension $D$ in FOGD, we
follow the same strategy used in Section~7.1.1 of \citep{Lu_2015large}.
More specifically, these hyperparameters are separately set for different
datasets as reported in Table~\ref{tab:exp_online_regression}. They
are chosen such that they are roughly proportional to the number of
support vectors produced by the batch SVM algorithm in LIBSVM running
on a small subset. The aim is to achieve competitive accuracy using
a relatively larger budget size for tackling more challenging regression
tasks.

\paragraph{Results.}

Fig.~\ref{fig:exp_online_regression} and Fig.~\ref{fig:exp_online_regression_time_cost}
shows the relative performance convergence w.r.t regression error
(root mean square root - RMSE) and computation cost (seconds) of the
$\model$s in comparison with those of the baselines. Combining these
two figures, we compare the average error and running time in Fig.~\ref{fig:exp_online_reg_rmse_time}.
Table~\ref{tab:exp_online_regression} reports the final average
results in detailed numbers after the methods traverse all data samples.
From these results, we can draw some observations as follows.

First of all, as can be seen from Fig.~\ref{fig:exp_online_regression},
there are several different learning behaviors w.r.t regression loss,
of the methods training on individual datasets. All algorithms, in
general, reach their regression error plateaus very quickly as observed
in the datasets \emph{year} and \emph{airlines} where they converge
at certain points from the initiation of the learning. On the other
hand, for \emph{casp} and \emph{slice} databases, the $\model$-based
models regularly obtain better performance, that is, their average
RMSE scores keep reducing when receiving more data, except in \emph{slice}
data, the regression performance of $\model$-$\ell_{2}$ are almost
unchanged during the learning. Note that, for these two datasets,
the learning curve of $\model$-$\boldsymbol{\varepsilon}$ coincides,
thus is overplotted by that of $\model$-$\ell_{1}$, resulting in
its no-show in the figure. Interestingly, the errors of RBP and Forgetron
slightly increase throughout their online learning in these two cases.

Second, Fig.~\ref{fig:exp_online_reg_rmse_time} plots average error
against computational cost, which shows similar learning behaviors
as in the our first observation. The computational cost progresses
are simple and more obvious to comprehend than the regression progresses.
As illustrated in Fig.~\ref{fig:exp_online_regression_time_cost},
all algorithms have nicely plausible execution time curves in which
the time is accumulated over the learning procedure.

\begin{figure}[H]
\noindent \begin{centering}
\includegraphics[width=0.8\textwidth]{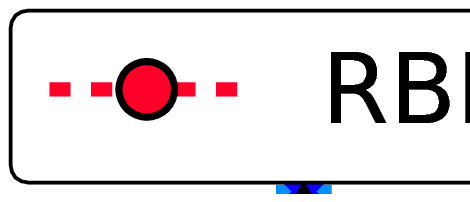}
\par\end{centering}
\noindent \begin{centering}
\subfloat[The average RMSE as a function of the number of samples seen by the
models.\label{fig:exp_online_regression}]{\noindent \begin{centering}
\begin{tabular}{cc}
\includegraphics[width=0.45\textwidth]{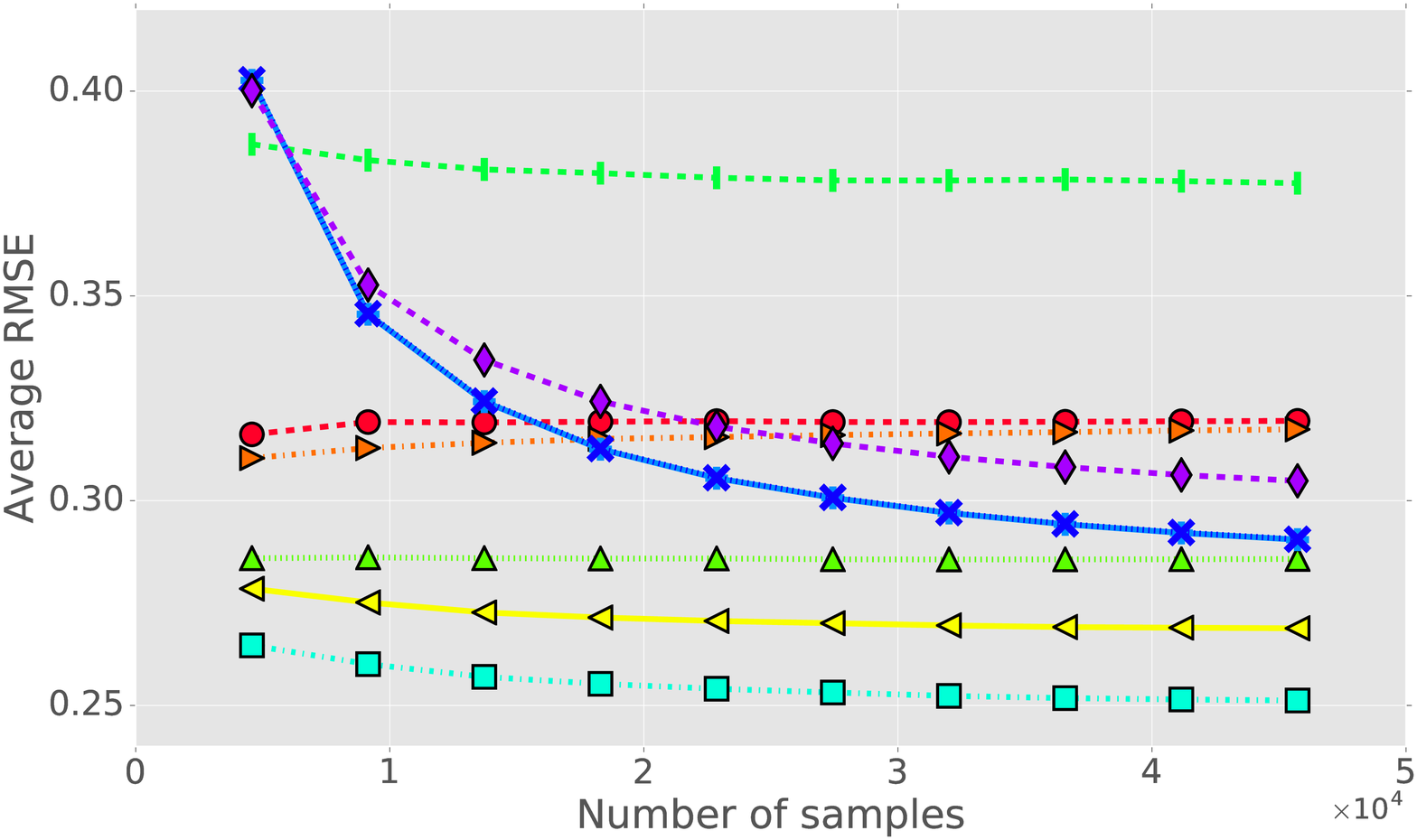} & \includegraphics[width=0.45\textwidth]{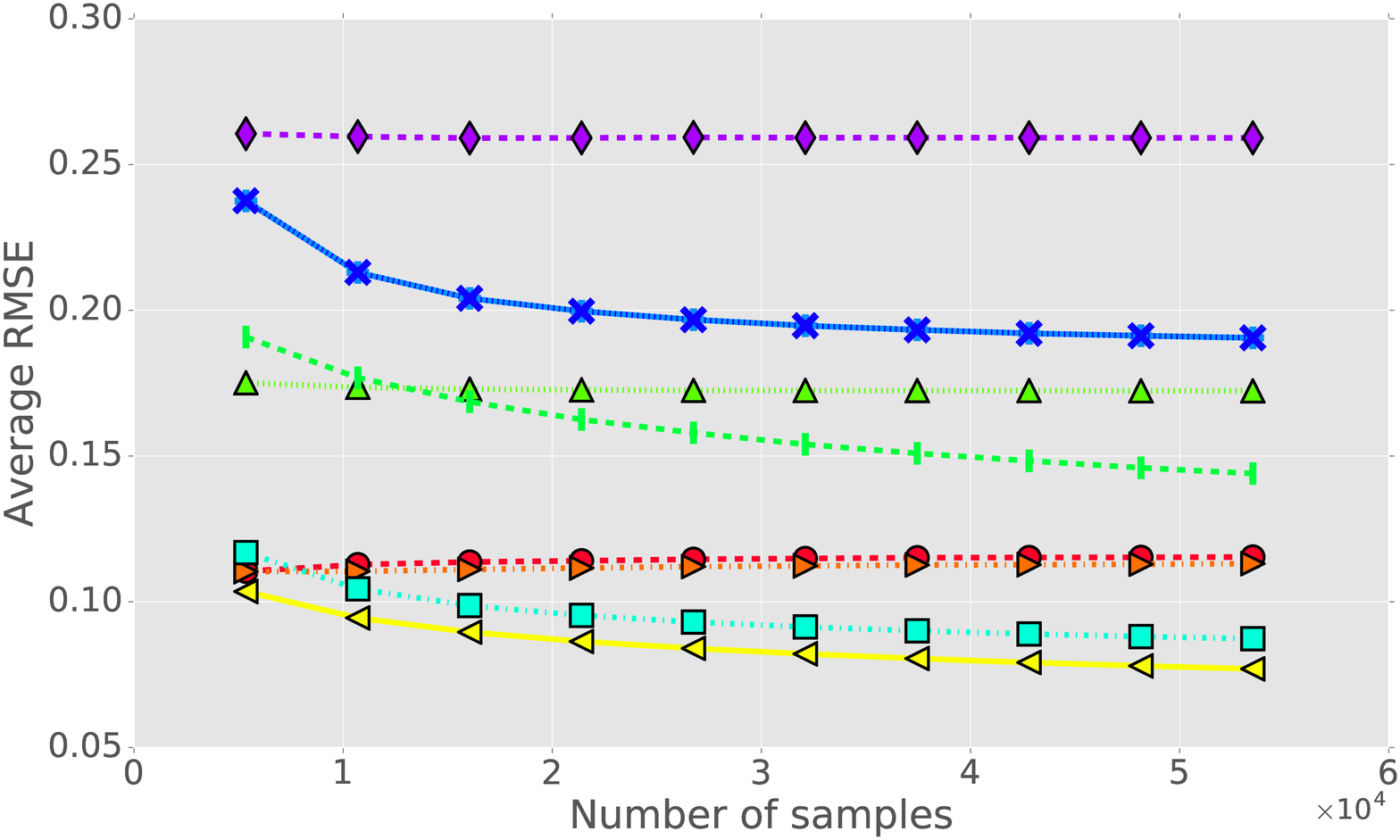}\tabularnewline
casp & slice\tabularnewline
\includegraphics[width=0.45\textwidth]{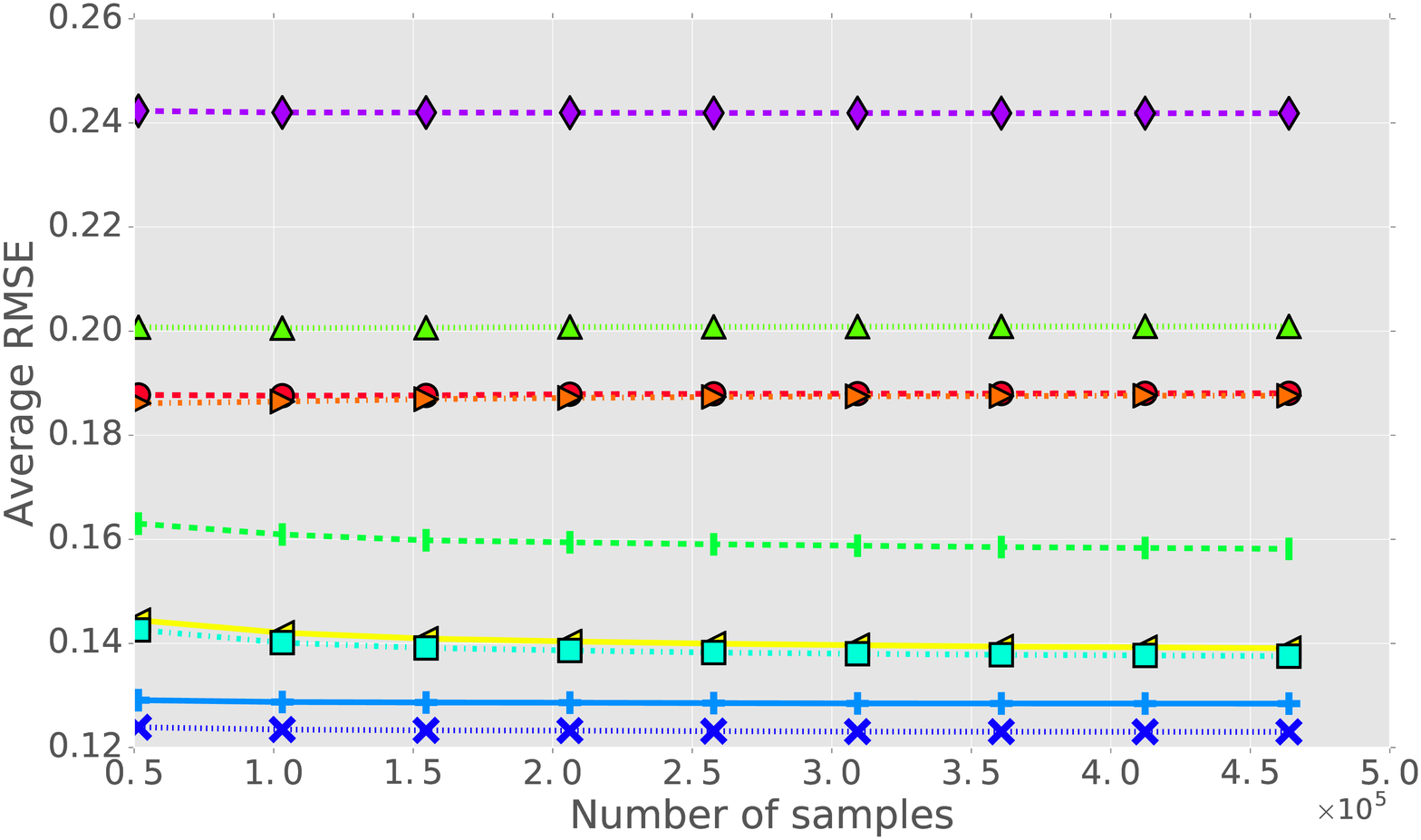} & \includegraphics[width=0.45\textwidth]{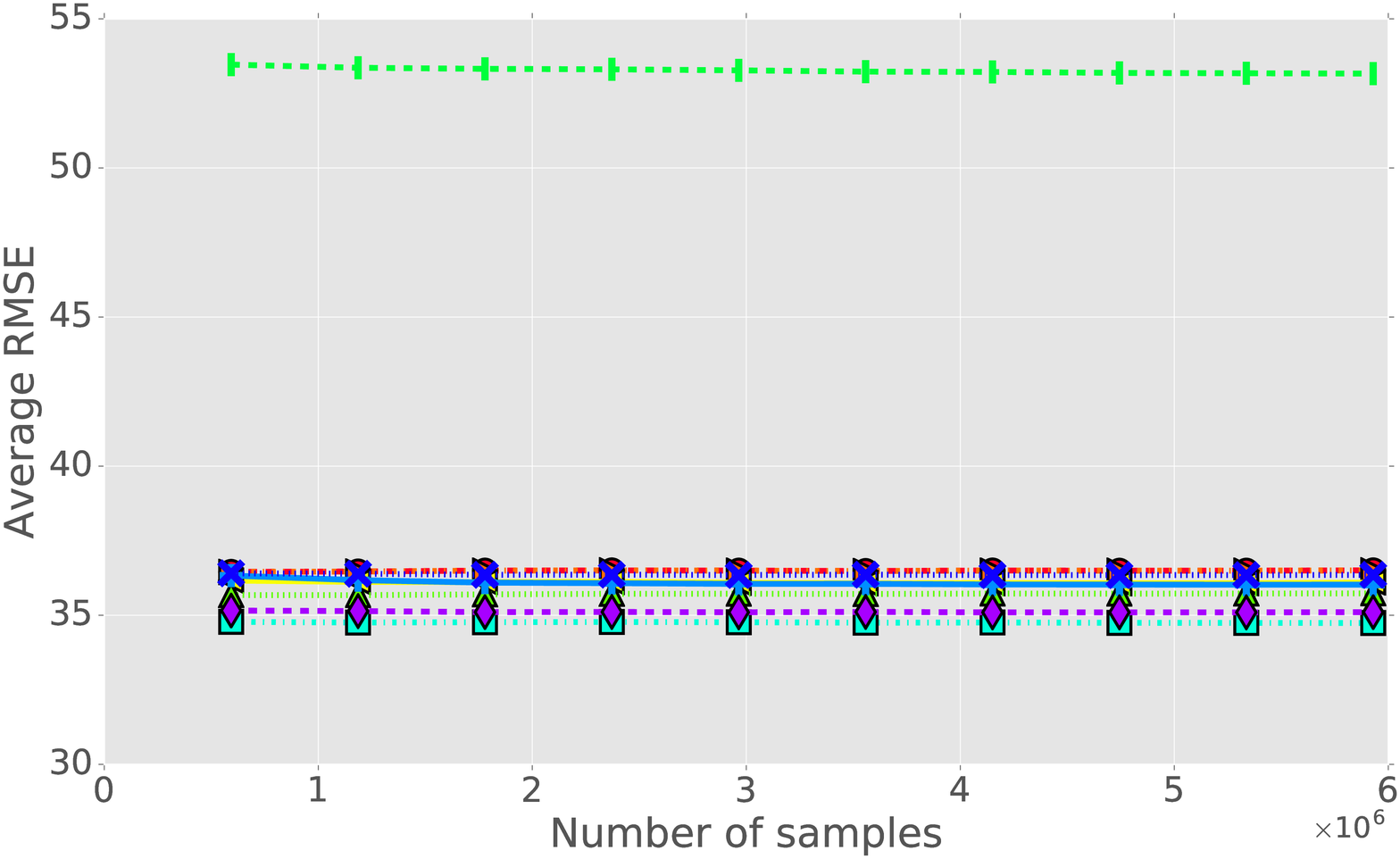}\tabularnewline
year & airlines\tabularnewline
\end{tabular}
\par\end{centering}
\noindent \centering{}}\vspace{-2mm}
\par\end{centering}
\noindent \begin{centering}
\subfloat[The average time costs (seconds in logarithm of 10) as a function
of the number of samples seen by the models.\label{fig:exp_online_regression_time_cost}]{\noindent \centering{}%
\begin{tabular}{cc}
\includegraphics[width=0.45\textwidth]{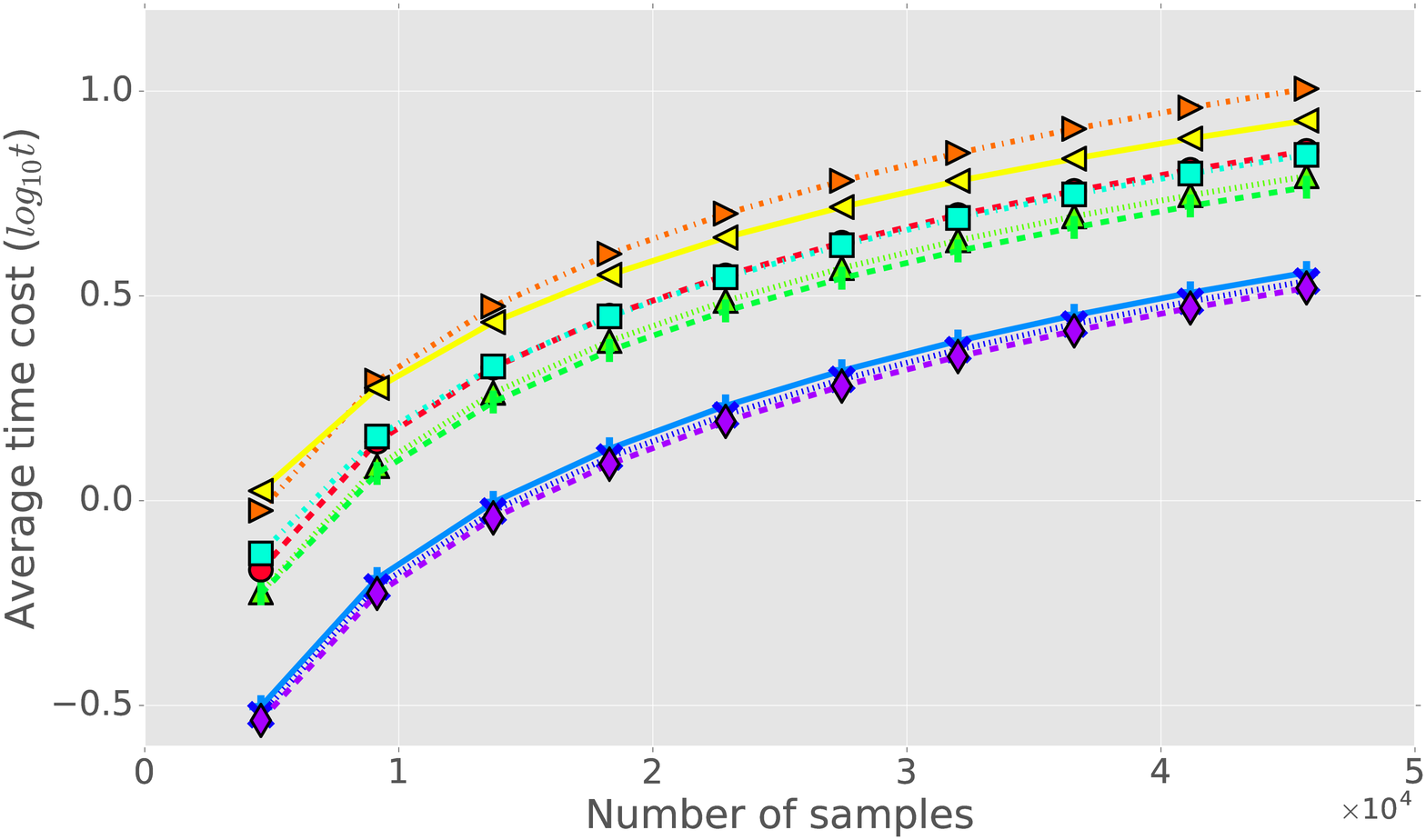} & \includegraphics[width=0.45\textwidth]{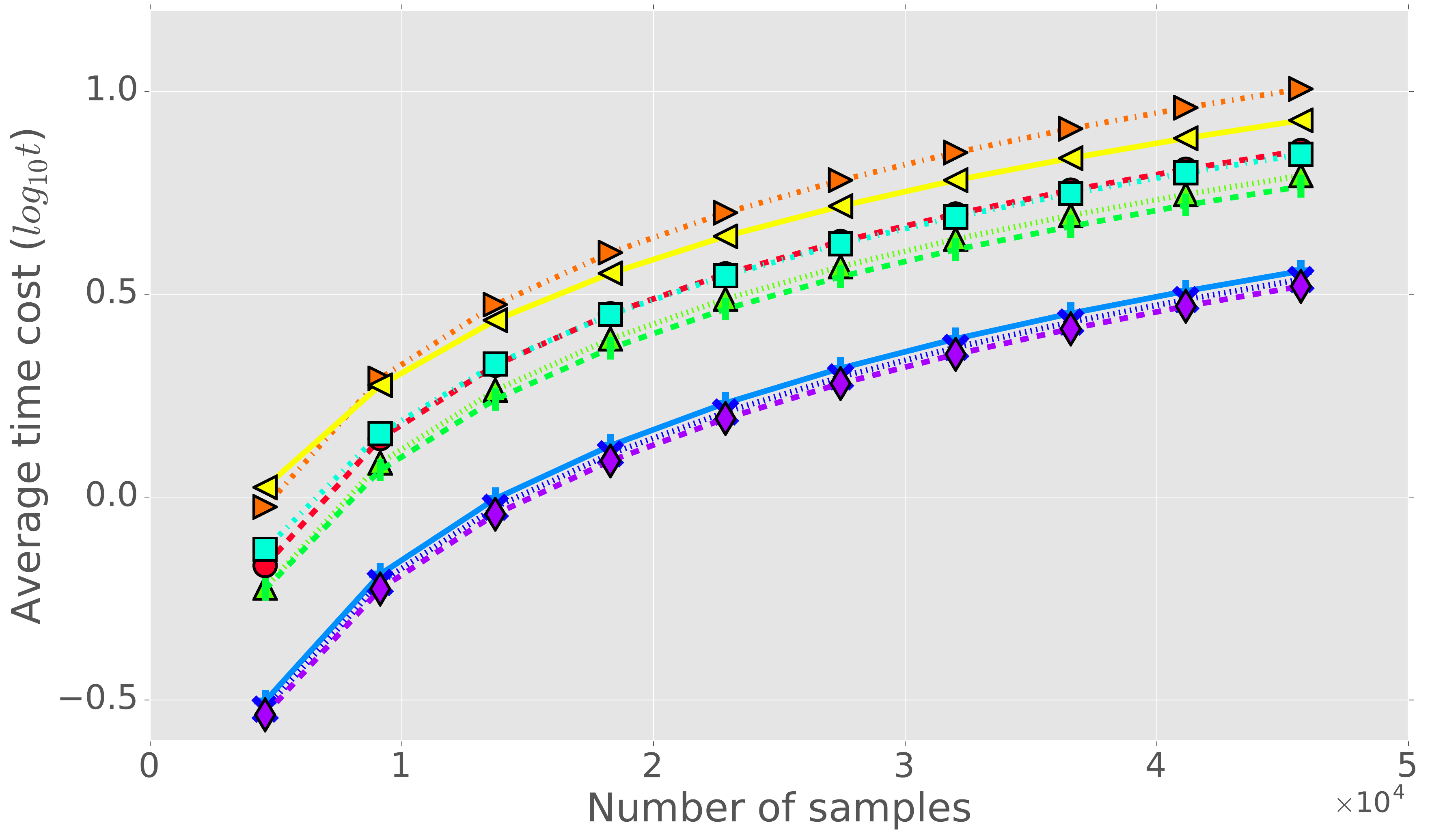}\tabularnewline
casp & slice\tabularnewline
\includegraphics[width=0.45\textwidth]{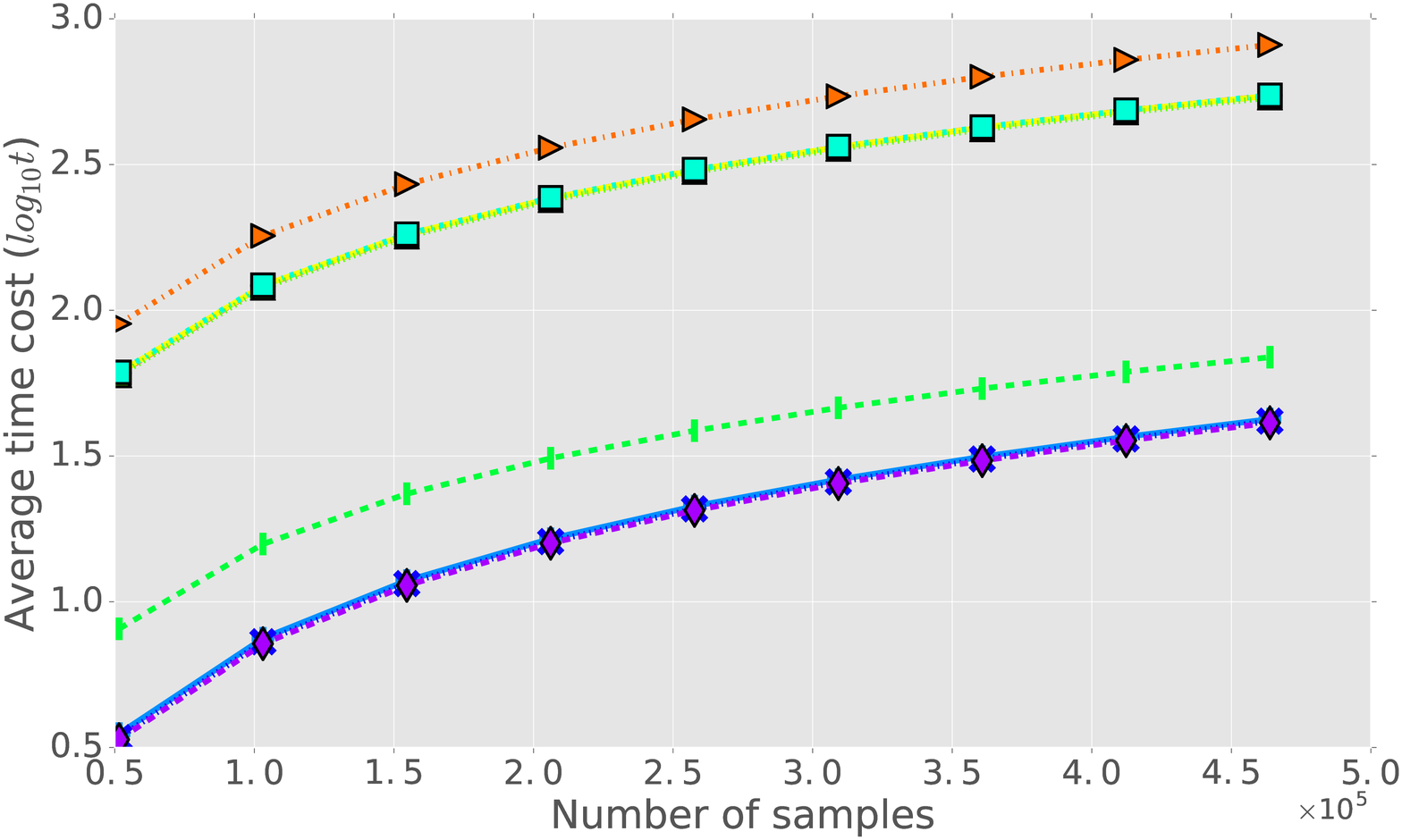} & \includegraphics[width=0.45\textwidth]{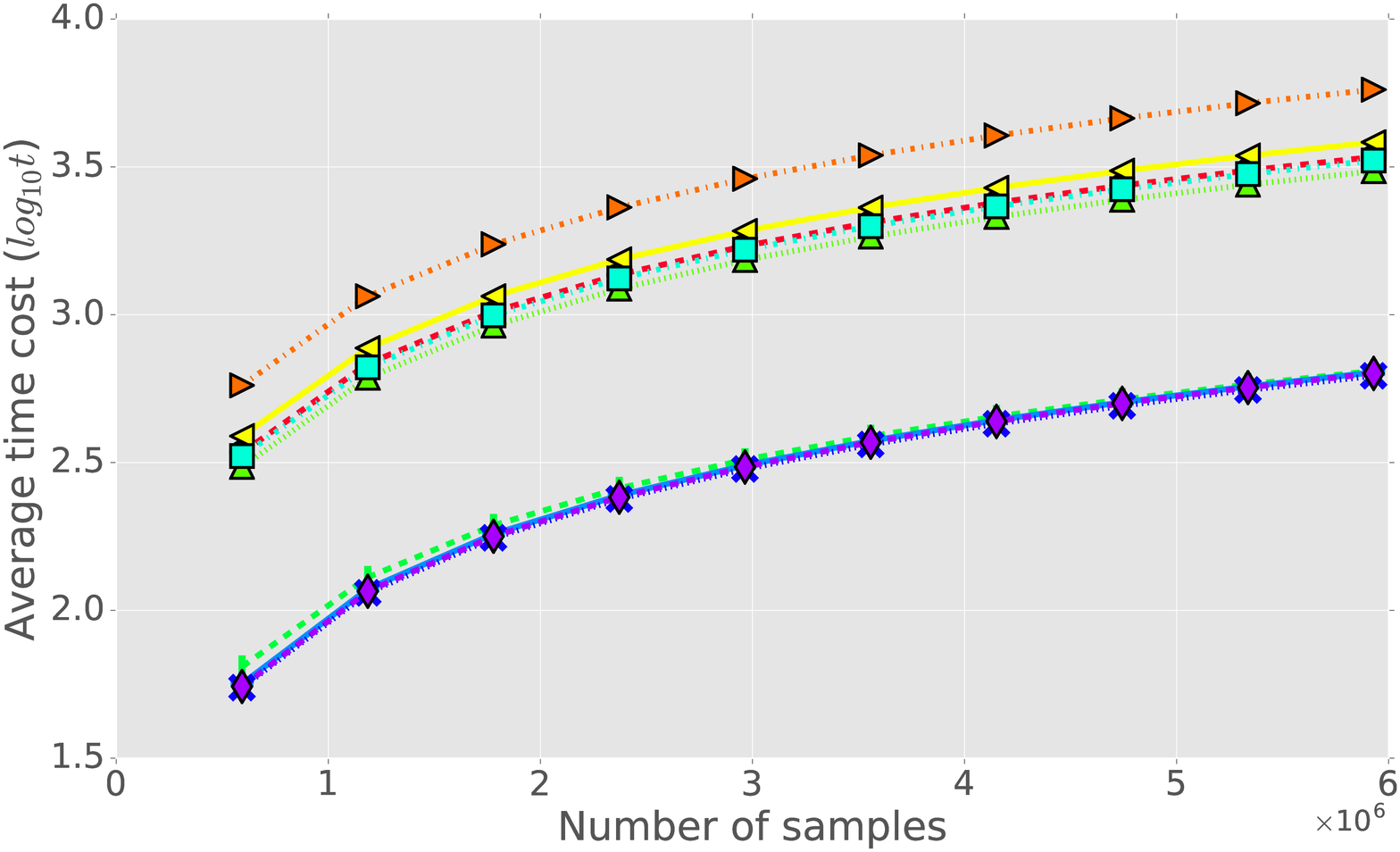}\tabularnewline
year & airlines\tabularnewline
\end{tabular}}\vspace{-2mm}
\par\end{centering}
\noindent \centering{}\caption{Convergence evaluation of online regresion. (Best viewed in colors).}
\end{figure}
\begin{table}[H]
\noindent \centering{}\caption{Online regression results of 6 baselines and 3 versions of our $\protect\model$s.
The notation $\left[\delta;S;B;D\right]$ denotes the same meanings
as those in Table~\ref{tab:exp_batch_classification}. The regression
loss is measured using root mean squared error (RMSE) and the execution
time is reported in second. The best performance is in \textbf{bold}.\label{tab:exp_online_regression}}
\begin{tabular}{|c|r|r|r|r|}
\hline 
\textbf{\emph{\cellcolor{header_color}Dataset }} & \multicolumn{2}{c|}{\textbf{\emph{\cellcolor{header_color}casp }}} & \multicolumn{2}{c|}{\textbf{\emph{\cellcolor{header_color}slice }}}\tabularnewline
\textbf{\emph{\cellcolor{header_color}}}$\left[\delta\mid S\mid B\mid D\right]$ & \multicolumn{2}{c|}{\textbf{\emph{\cellcolor{header_color}$\left[4.0\mid166\mid400\mid2,000\right]$}}} & \multicolumn{2}{c|}{\textbf{\emph{\cellcolor{header_color}$\left[16.0\mid27\mid1,000\mid3,000\right]$}}}\tabularnewline
\hline 
\textbf{\emph{\cellcolor{header_color}Algorithm}} & RMSE & Time & RMSE & Time\tabularnewline
\hline 
RBP & 0.3195$\pm$0.0012 & 7.15 & 0.1154$\pm$0.0006 & 810.14\tabularnewline
Forgetron & 0.3174$\pm$0.0008 & 10.14 & 0.1131$\pm$0.0004 & 1,069.15\tabularnewline
Projectron & 0.2688$\pm$0.0002 & 8.48 & \textbf{0.0770}$\pm$\textbf{0.0002} & 814.37\tabularnewline
BOGD & 0.2858$\pm$0.0002 & 6.20 & 0.1723$\pm$0.0001 & 816.16\tabularnewline
FOGD & 0.3775$\pm$0.0014 & 5.83 & 0.1440$\pm$0.0009 & 20.65\tabularnewline
NOGD & \textbf{0.2512}$\pm$\textbf{0.0001} & 6.99 & 0.0873$\pm$0.0002 & 812.69\tabularnewline
\hline 
$\model$-$\boldsymbol{\varepsilon}$ & 0.3165$\pm$0.0329 & 3.53 & 0.2013$\pm$0.0137 & 7.07\tabularnewline
$\model$-$\ell_{1}$ & 0.3166$\pm$0.0330 & 3.44 & 0.2013$\pm$0.0138 & 7.13\tabularnewline
$\model$-$\ell_{2}$ & 0.3274$\pm$0.0280 & \textbf{3.31} & 0.2590$\pm$0.0002 & \textbf{6.88}\tabularnewline
\hline 
\hline 
\textbf{\emph{\cellcolor{header_color}Dataset}} & \multicolumn{2}{c|}{\textbf{\emph{\cellcolor{header_color}year }}} & \multicolumn{2}{c|}{\textbf{\emph{\cellcolor{header_color}airlines }}}\tabularnewline
\textbf{\emph{\cellcolor{header_color} }}$\left[\delta\mid S\mid B\mid D\right]$ & \multicolumn{2}{c|}{\textbf{\emph{\cellcolor{header_color}$\left[60.0\mid67\mid400\mid1,600\right]$}}} & \multicolumn{2}{c|}{\textbf{\emph{\cellcolor{header_color}$\left[1.0\mid388\mid1,000\mid2,000\right]$}}}\tabularnewline
\hline 
\textbf{\emph{\cellcolor{header_color}Algorithm}} & RMSE & Time & RMSE & Time\tabularnewline
\hline 
RBP & 0.1881$\pm$0.0002 & 605.42 & 36.5068$\pm$0.0010 & 3,418.89\tabularnewline
Forgetron & 0.1877$\pm$0.0004 & 904.09 & 36.5065$\pm$0.0003 & 5,774.47\tabularnewline
Projectron & 0.1390$\pm$0.0003 & 605.19 & 36.1365$\pm$0.0009 & 3,834.19\tabularnewline
BOGD & 0.2009$\pm$0.0000 & 596.10 & 35.7346$\pm$0.0010 & 3,058.96\tabularnewline
FOGD & 0.1581$\pm$0.0002 & 76.70 & 53.1638$\pm$0.0120 & 646.15\tabularnewline
NOGD & 0.1375$\pm$0.0005 & 607.37 & \textbf{34.7421}$\pm$\textbf{0.0013} & 3,324.38\tabularnewline
\hline 
$\model$-$\boldsymbol{\varepsilon}$ & 0.1286$\pm$0.0002 & 48.01 & 36.0901$\pm$0.0914 & 638.60\tabularnewline
$\model$-$\ell_{1}$ & \textbf{0.1232}$\pm$\textbf{0.0003} & 47.29 & 36.3632$\pm$0.0192 & \textbf{621.57}\tabularnewline
$\model$-$\ell_{2}$ & 0.2420$\pm$0.0001 & \textbf{46.63} & 35.1128$\pm$0.0192 & 633.27\tabularnewline
\hline 
\end{tabular}
\end{table}

According to final results summarized in Table~\ref{tab:exp_online_regression},
our proposed models enjoy a significant advantage in computational
efficacy whilst achieve better (for \emph{year} dataset) or competitive
regression results with other methods. The $\model$, again, secures
the best performance in terms of model sparsity. Among the baselines,
the FOGD is the fastest, that is, its time costs can be considered
to compare with those of our methods, but its regression performances
are worse. The remaining algorithms usually obtain better results,
but is traded off by the sacrifice of scalability. This, once again,
verifies the effectiveness and efficiency of our proposed techniques.
We believe that the $\model$ is a promising machine to perform online
regression task for large-scale datasets.

Finally, comparing the capability of three $\model$'s variants, all
models demonstrate similar computational complexities wherein the
$\model$-$\ell_{2}$ is slightly faster due to its simpler operator
in computing the gradient as derived in Section~\ref{sec:Loss-Function}.
However, its regression errors are higher than two other methods \textendash{}
$\model$-$\boldsymbol{\varepsilon}$ and $\model$-$\ell_{1}$.

\begin{figure}[H]
\noindent \begin{centering}
\includegraphics[width=0.8\textwidth]{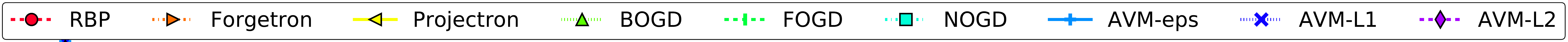}
\par\end{centering}
\noindent \begin{centering}
\begin{tabular}{cc}
\includegraphics[width=0.45\textwidth]{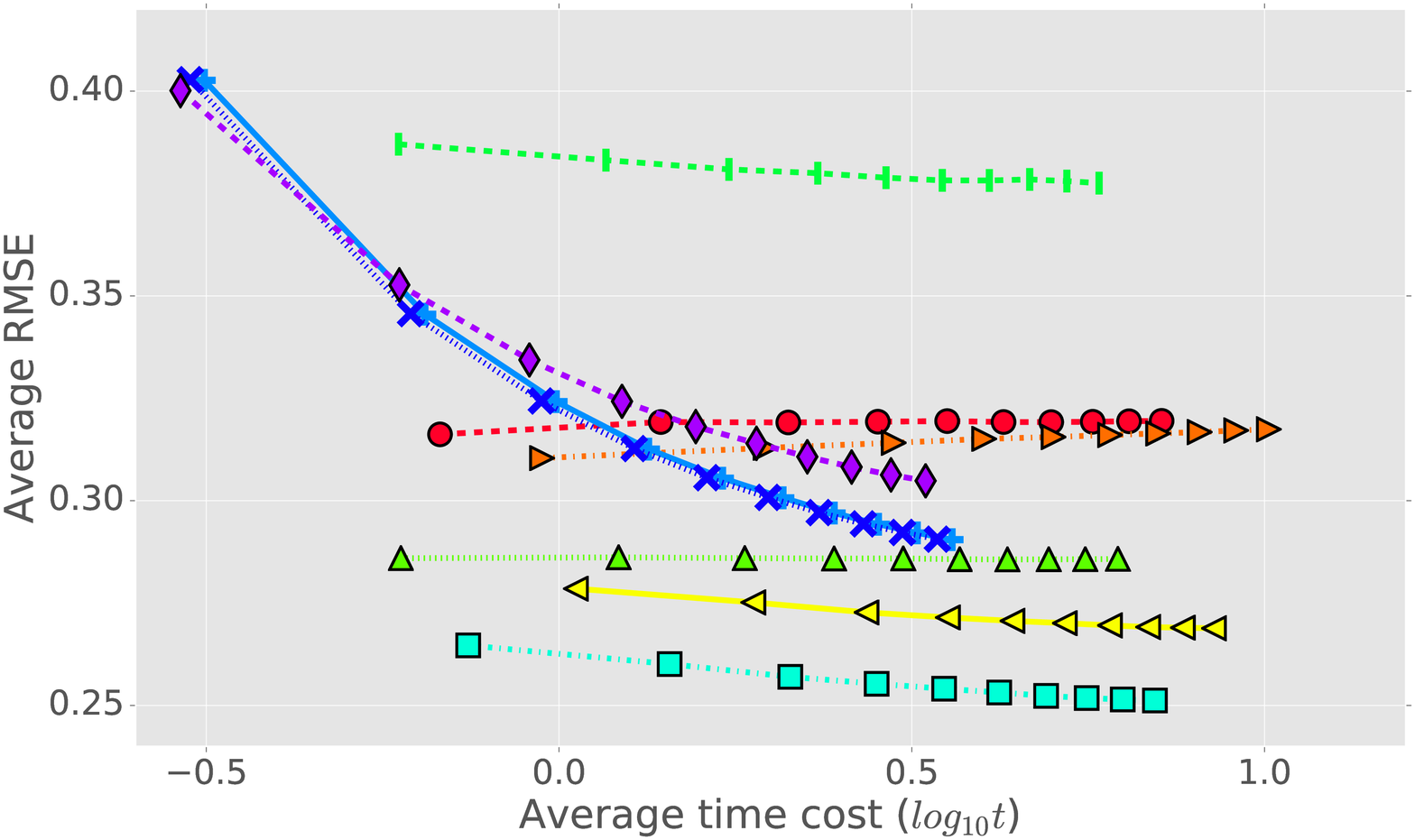} & \includegraphics[width=0.45\textwidth]{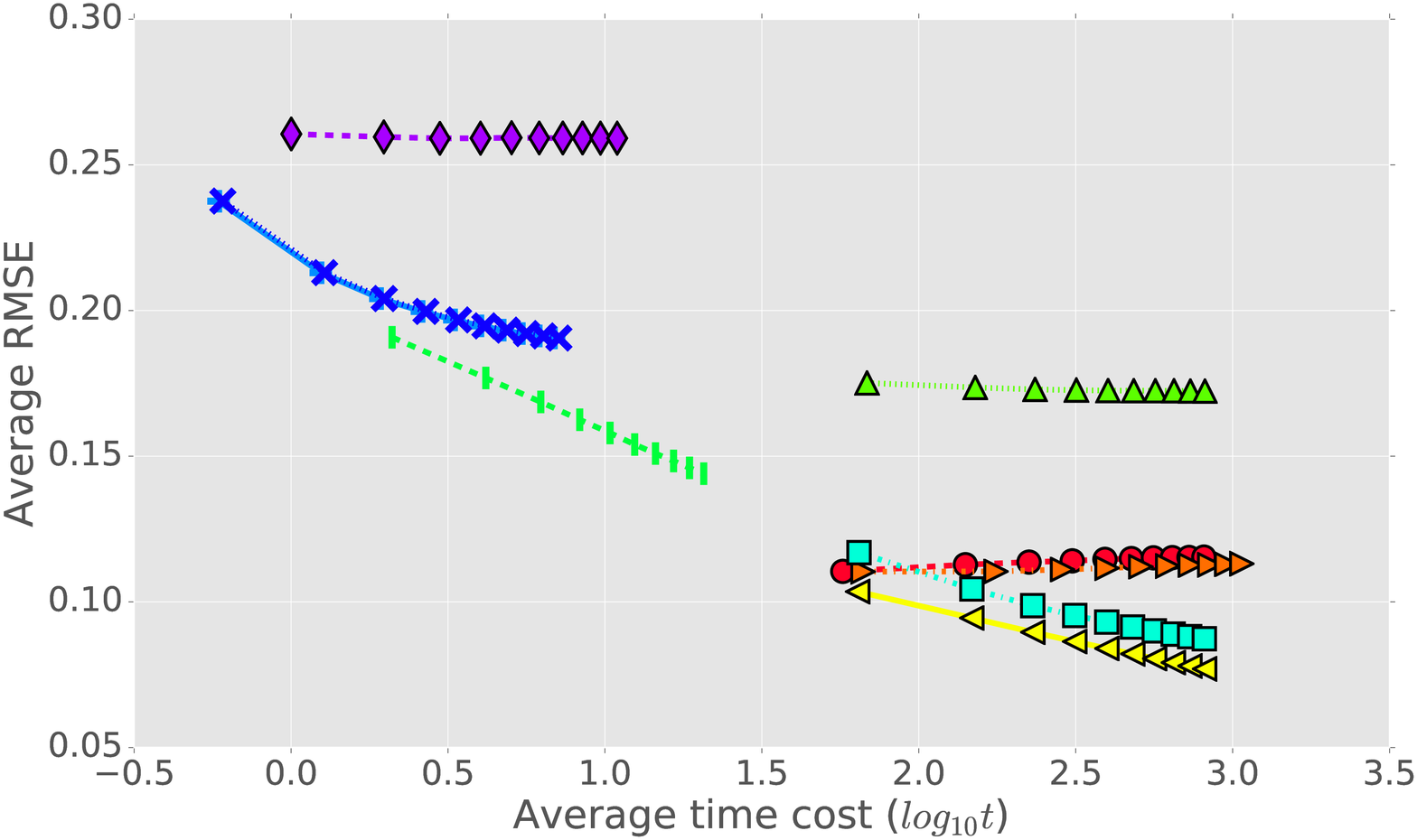}\tabularnewline
casp & slice\tabularnewline
\includegraphics[width=0.45\textwidth]{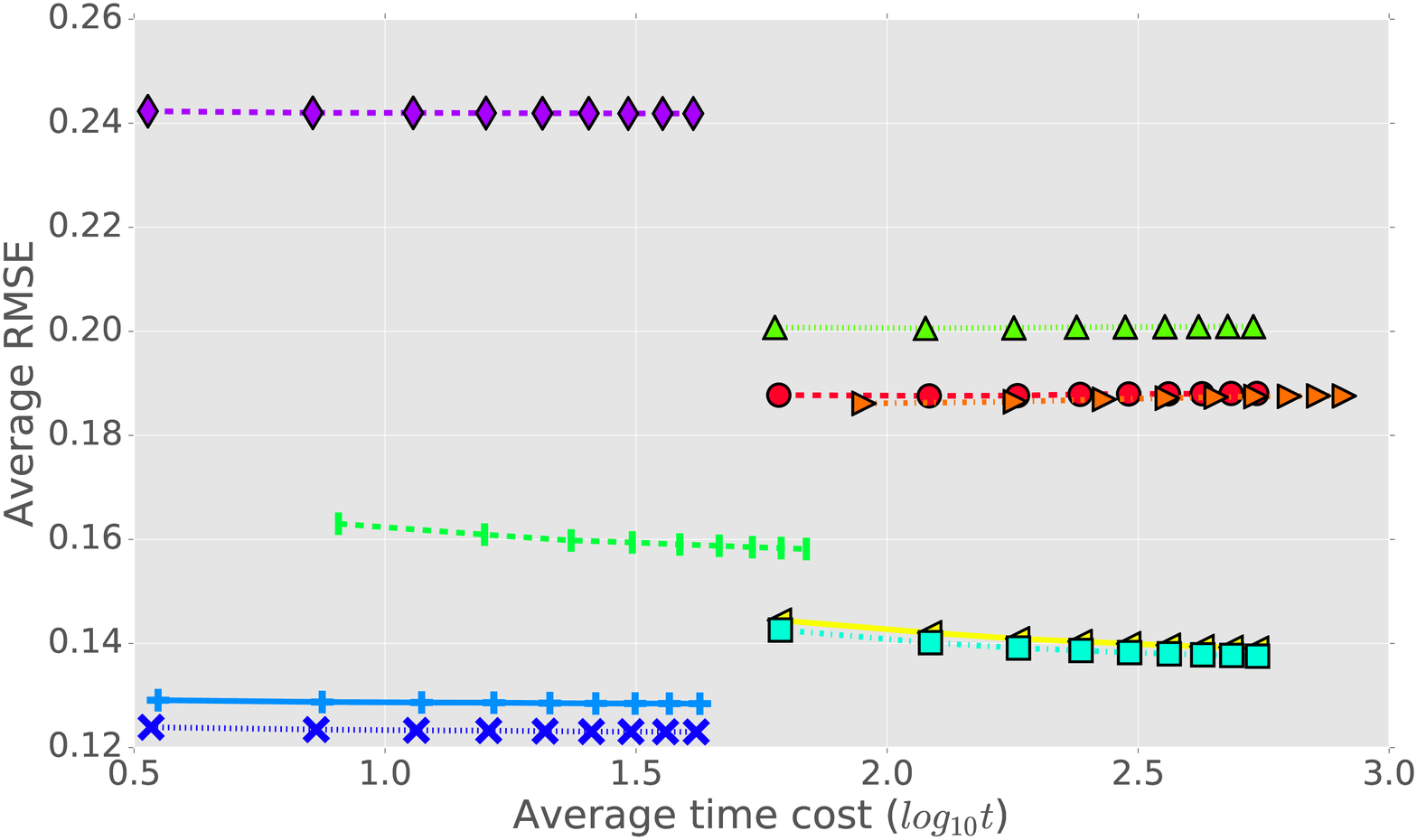} & \includegraphics[width=0.45\textwidth]{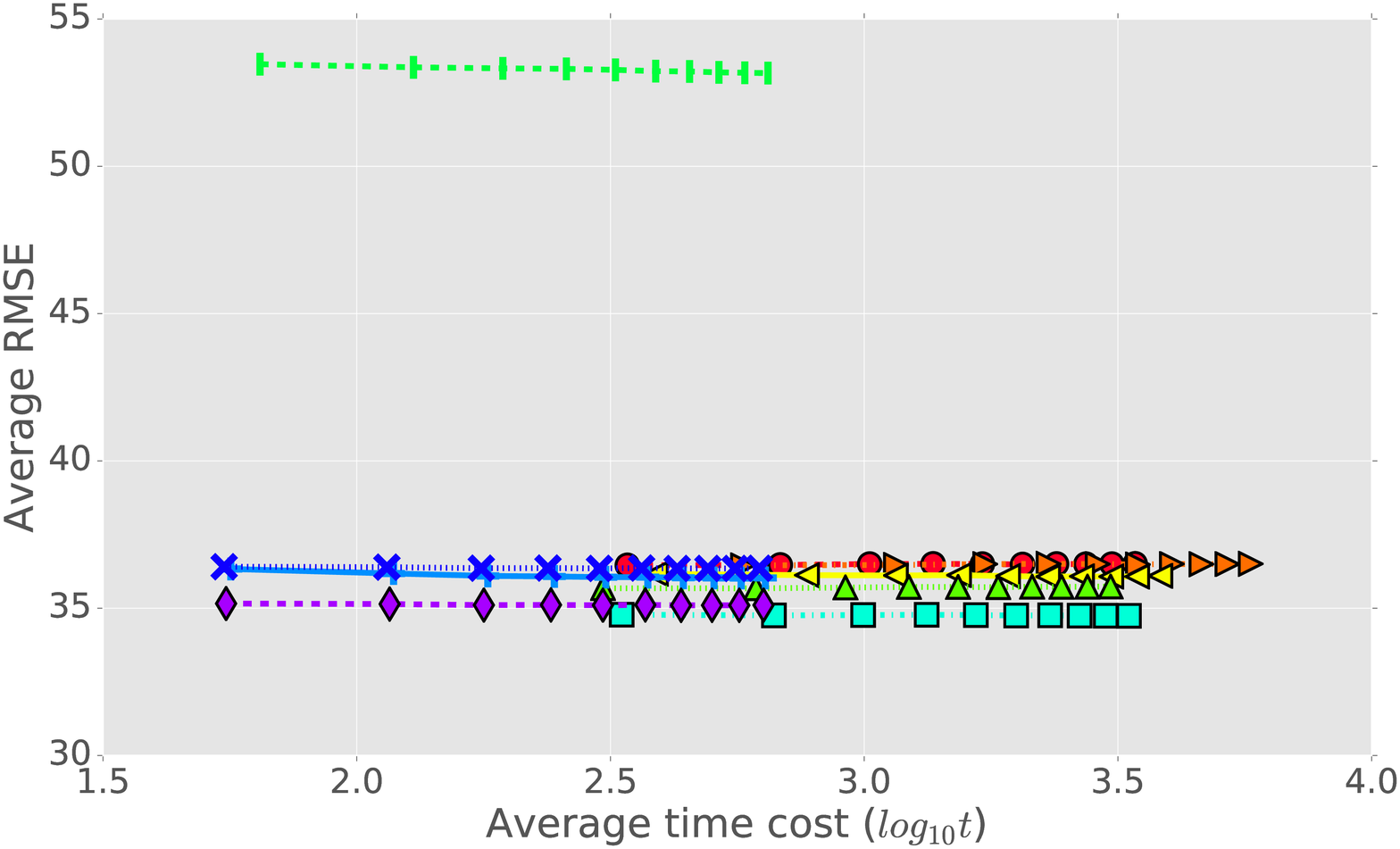}\tabularnewline
year & airlines\tabularnewline
\end{tabular}
\par\end{centering}
\noindent \centering{}\caption{Average RMSE vs. time cost for online regression. The average time
(seconds) is shown in the logarithm with base 10. (Best viewed in
colors).\label{fig:exp_online_reg_rmse_time}}
\end{figure}

\noindent

\section{Conclusion\label{sec:Conclusion}}

In this paper, we have proposed Approximation Vector Machine ($\model$)
for large-scale online learning. The $\model$ is theoretically proven
to have bounded and sparse model size while not hurting the predictive
performance. We have validated our proposed method on several benchmark
datasets. The experimental results show that the proposed $\model$
obtains a comparable predictive performance while simultaneously achieving
an impressive model size and a computational speed-up compared with
those of the baselines. Our future works are to apply AVM to the context
of semi-supervised learning, anomaly detection, and support vector
clustering.

\section*{Acknowledgment}

We gratefully thank the editor and the anonymous reviewers for their
valuable comments and thorough inspection of the article. This work
was partially supported by the Australian Research Council (ARC) under
the Discovery Project DP160109394.

\newpage{}

\appendix

\section{Proofs Regarding $\delta$-Coverage\label{sec:appxA}}
\begin{proof}
\textbf{of Theorem \ref{thm:coverage_image}}

Assume that $\norm{x-x^{'}}\leq\delta$, then we have
\begin{flalign*}
\norm{\Phi\left(x\right)-\Phi\left(x^{'}\right)}^{2} & =K\left(x,x\right)+K\left(x^{'},x^{'}\right)-2K\left(x,x^{'}\right)=2\left(1-k\left(\norm{x-x^{'}}^{2}\right)\right)\\
\leq & 2\left(1-k\left(\delta^{2}\right)\right)=\delta_{\Phi}^{2}
\end{flalign*}

Furthermore, we have
\[
\underset{\delta\goto0}{\lim}\delta_{\Phi}=2^{1/2}\underset{\delta\goto0}{\lim}\left(1-k\left(\delta^{2}\right)\right)^{1/2}=2^{1/2}\left(1-k\left(0\right)\right)^{1/2}=0
\]

Finally, since Gaussian kernel is a special radial kernel with $k\left(t\right)=\exp\left(-\gamma t\right)$,
we obtain the final conclusion.
\end{proof}

\begin{proof}
\textbf{of Theorem \ref{thm:finite_coverage}}

Since the proof is similar for the hyperrectangle cell case, we present
the proof for the hypersphere case. Let us consider the open coverage
$\mathcal{U}=\left\{ \mathcal{B}\left(z,\frac{\delta}{2}\right)\right\} _{z\in\mathcal{X}}$.
From the compactness of the data domain $\mathcal{X}$, it apparent
that from $\mathcal{U}$ we must be able to extract a finite subcoverage
of size $m$, that is, $\mathcal{U}_{m}=\left\{ \mathcal{B}\left(z_{i},\frac{\delta}{2}\right)\right\} _{i=1}^{m}\subset\mathcal{U}$.
From the construction of the coverage $\mathcal{P}$ in Algorithm
\ref{alg:sphere_coverage}, we know that 
\[
\norm{c_{i}-c_{j}}>\delta/2\,\,\text{if}\,\,i\neq j
\]

Hence, each open sphere in the finite coverage $\mathcal{U}_{m}$
is able to contain at most one core point of $\mathcal{P}$. It means
that the cardinality of $\mathcal{P}$ must be less than or equal
$m$, that is, $|\mathcal{P}|\leq m$.
\end{proof}

\clearpage{}

\newpage{}
\section{Proofs Regarding Convergence Analysis\label{sec:appxB}}

Given a finite $\delta$-coverage $\mathcal{P}=\left(P_{i}\right)_{i\in I}$
with the core set $\mathcal{C}=\left(c_{i}\right)_{i\in I}$, when
receiving an incoming instance $\left(x_{t},y_{t}\right)$ we approximate
$\left(x_{t},y_{t}\right)$ by $\left(c_{i_{t}},y_{t}\right)$ with
$c_{i_{t}}$ is a core point whose cell contains $x_{t}$, that is,
$x_{t}\in P_{i_{t}}$. We use a Bernoulli random variable $Z_{t}$
to control if the approximation is performed or not, that is, $Z_{t}=1$
indicates the approximation is performed. 

Let us define $g_{t}=\lambda\bw_{t}+l^{'}\left(\bw_{t};x_{t},y_{t}\right)=\lambda\bw_{t}+\alpha_{t}\Phi\left(x_{t}\right)$.
We have the following
\[
h_{t}=g_{t}+Z_{t}\Delta_{t}
\]

\noindent where $\Delta_{t}=\alpha_{t}\left(\Phi\left(c_{i_{t}}\right)-\Phi\left(x_{t}\right)\right)$.

The update rule becomes
\[
\bw_{t+1}=\prod_{S}\left(\bw_{t}-\eta_{t}h_{t}\right)
\]
where $S=\mathbb{R}^{D}$ (i.e., the feature space) or $\mathcal{B}\left(\bzero,y_{\text{max}}\lambda^{-1/2}\right)$.
\begin{lem}
\label{lem:main_lemma}The following statements hold

i) There exist two positive constants $P$ and $M$ such that $\mathbb{E}\left[\norm{\bw_{t}}^{2}\right]\leq P^{2}$
and $\mathbb{E}\left[\alpha_{t}^{2}\right]\leq M$ for all $t$.

ii) $\mathbb{E}\left[\norm{l^{'}\left(\bw_{t};x_{t},y_{t}\right)}^{2}\right]\leq L=\left(A\sqrt{P}+B\right)^{2}$
for all $t$.

iii) $\mathbb{E}\left[\norm{g_{t}}^{2}\right]\leq G=\left(\lambda P+A\sqrt{P}+B\right)^{2}$
for all $t$.

iv) $\mathbb{E}\left[\norm{h_{t}}^{2}\right]\leq H=\left(\sqrt{G}+\delta_{\Phi}\left(A\sqrt{P}+B\right)\right)^{2}$
for all $t$.
\end{lem}
\begin{proof}
i) We prove by induction that $\mathbb{E}\left[\norm{\bw_{t}}^{2}\right]\leq P^{2}$
where $P=\left(\frac{\left(\delta_{\Phi}+1\right)A+\sqrt{\left(\delta_{\Phi}+1\right)^{2}A^{2}+4B\lambda\left(\delta_{\Phi}+1\right)}}{2\lambda}\right)^{2}$
for all $t$. Assume that the claim is holding for $t$, using Minkowski
inequality, we have\\
\begin{flalign*}
\sqrt{\mathbb{E}\left[\norm{\bw_{t+1}}^{2}\right]} & \leq\sqrt{\mathbb{E}\left[\norm{\prod_{S}\left(\bw_{t}-\eta_{t}h_{t}\right)}^{2}\right]}\leq\sqrt{\mathbb{E}\left[\norm{\bw_{t}-\eta_{t}h_{t}}^{2}\right]}\\
 & \leq\frac{t-1}{t}\sqrt{\mathbb{E}\left[\norm{\bw_{t}}^{2}\right]}+\eta_{t}\sqrt{\mathbb{E}\left[\norm{l^{'}\left(\bw_{t};x_{t},y_{t}\right)}^{2}\right]}+\eta_{t}\sqrt{\mathbb{E}\left[\norm{\Delta_{t}}^{2}\right]}\\
 & \leq\frac{t-1}{t}\sqrt{\mathbb{E}\left[\norm{\bw_{t}}^{2}\right]}+\frac{A\sqrt{\mathbb{E}\left[\norm{\bw_{t}}\right]}+B}{\lambda t}+\frac{\delta_{\Phi}\sqrt{\mathbb{E}\left[\alpha_{t}^{2}\right]}}{\lambda t}\\
 & \leq\frac{\left(t-1\right)P}{t}+\frac{A\sqrt{P}+B}{\lambda t}+\frac{\delta_{\Phi}\sqrt{\mathbb{E}\left[\norm{l^{'}\left(\bw_{t};x_{t},y_{t}\right)}^{2}\right]}}{\lambda t}\\
 & \leq\frac{\left(t-1\right)P}{t}+\frac{\left(\delta_{\Phi}+1\right)\left(A\sqrt{P}+B\right)}{\lambda t}=P
\end{flalign*}

Note that we have used $\norm{l^{'}\left(\bw_{t};x_{t},y_{t}\right)}^{2}=\alpha_{t}^{2}K\left(x_{t},x_{t}\right)=\alpha_{t}^{2}$,
$\mathbb{E}\left[\norm{\bw_{t}}\right]\leq\sqrt{\mathbb{E}\left[\norm{\bw_{t}}^{2}\right]}\leq P$,
and $u=P^{1/2}=\frac{\left(\delta_{\Phi}+1\right)A+\sqrt{\left(\delta_{\Phi}+1\right)^{2}A^{2}+4B\lambda\left(\delta_{\Phi}+1\right)}}{2\lambda}$
is the solution of the quadratic equation
\[
u^{2}-\frac{\left(\delta_{\Phi}+1\right)Au}{\lambda}-\frac{\left(\delta_{\Phi}+1\right)B}{\lambda}=0
\]

The proof of $\mathbb{E}\left[\alpha_{t}^{2}\right]\leq M$ is trivial
for the case of Hinge, $\ell_{1}$, Logistic, $\varepsilon-$insensitive
losses. In these cases, we simply choose $M=\max\left(y_{\text{max}},1\right)^{2}$.
We only need to consider the $\ell_{2}$-loss case. In particular,
we have
\begin{align*}
\alpha_{t}^{2} & =\left(\transp{\bw_{t}}\Phi\left(x_{t}\right)-y_{t}\right)^{2}\leq2\left(\left(\transp{\bw_{t}}\Phi\left(x_{t}\right)\right)^{2}+y_{\max}^{2}\right)\\
 & \leq2\left(\norm{\bw_{t}}^{2}\norm{\Phi\left(x_{t}\right)}^{2}+y_{\max}^{2}\right)\leq2\left(\norm{\bw_{t}}^{2}+y_{\max}^{2}\right)
\end{align*}

\[
\mathbb{E}\left[\alpha_{t}^{2}\right]\leq2\left(P^{2}+y_{\max}^{2}\right)=M
\]

ii) We have the following
\[
\sqrt{\mathbb{E}\left[\norm{l^{'}\left(\bw_{t};x_{t},y_{t}\right)}^{2}\right]}\leq\sqrt{\mathbb{E}\left[\left(A\norm{\bw_{t}}^{1/2}+B\right)^{2}\right]}\leq A\sqrt{\mathbb{E}\left[\norm{\bw_{t}}\right]}+B\leq A\sqrt{P}+B
\]

Note that we have used the inequality $\mathbb{E}\left[\norm{\bw_{t}}\right]\leq\sqrt{\mathbb{E}\left[\norm{\bw_{t}}^{2}\right]}\leq P$.

iii) Using Minkowski inequality, we yield

\[
\sqrt{\mathbb{E}\left[\norm{g_{t}}^{2}\right]}\le\lambda\sqrt{\mathbb{E}\left[\norm{\bw_{t}}^{2}\right]}+\sqrt{\mathbb{E}\left[\norm{l^{'}\left(\bw_{t};x_{t},y_{t}\right)}^{2}\right]}\leq\lambda P+A\sqrt{P}+B
\]

iv) We have the following 
\begin{align*}
\sqrt{\mathbb{E}\left[\norm{h_{t}}^{2}\right]} & \le\sqrt{\mathbb{E}\left[\norm{g_{t}}^{2}\right]}+\delta_{\Phi}\sqrt{\mathbb{E}\left[\alpha_{t}^{2}\right]}\\
 & \leq\sqrt{G}+\delta_{\Phi}\sqrt{\mathbb{E}\left[\norm{l^{'}\left(\bw_{t};x_{t},y_{t}\right)}^{2}\right]}=\sqrt{G}+\delta_{\Phi}\left(A\sqrt{P}+B\right)
\end{align*}
\end{proof}
\begin{lem}
\label{lem:lem_W}There exists a positive constant $W$ such that
$\mathbb{E}\left[\norm{\bw_{t}-\bw^{*}}^{2}\right]\leq W$ for all
$t$.
\end{lem}
\begin{proof}
We first remind the definitions of the relevant quantities
\begin{align*}
g_{t} & =\lambda\bw_{t}+l^{'}\left(\bw_{t};x_{t},y_{t}\right)=\lambda\bw_{t}+\alpha_{t}\Phi\left(x_{t}\right)\\
h_{t} & =g_{t}+Z_{t}\Delta_{t}\,\text{where \ensuremath{\Delta_{t}=\alpha_{t}\left(\Phi\left(c_{i_{t}}\right)-\Phi\left(x_{t}\right)\right)}}
\end{align*}

We now prove by induction in $t$. We derive as follows
\begin{flalign*}
\norm{\bw_{t+1}-\bw^{*}}^{2} & =\norm{\prod_{S}\left(\bw_{t}-\eta_{t}h_{t}\right)-\bw^{*}}^{2}\leq\norm{\bw_{t}-\eta_{t}h_{t}-\bw^{*}}^{2}\\
 & =\norm{\bw_{t}-\bw^{*}}+\eta_{t}^{2}\norm{h_{t}}^{2}-2\eta_{t}\left\langle \bw_{t}-\bw^{*},g_{t}+Z_{t}\Delta_{t}\right\rangle 
\end{flalign*}

\noindent where $S=\mathbb{R}^{D}$ or $\mathcal{B}\left(\bzero,y_{\text{max}}\lambda^{-1/2}\right)$.

Taking conditional expectation w.r.t $\bw_{t}$, we gain 

\begin{flalign*}
\mathbb{E}\left[\norm{\bw_{t+1}-\bw^{*}}^{2}\right] & \leq\mathbb{E}\left[\norm{\bw_{t}-\bw^{*}}^{2}\right]\\
 & +\eta_{t}^{2}\mathbb{E}\left[\norm{h_{t}}^{2}\right]-2\eta_{t}\left\langle \bw_{t}-\bw^{*},f^{'}\left(\bw_{t}\right)\right\rangle -2\eta_{t}\left\langle \bw_{t}-\bw^{*},Z_{t}\Delta_{t}\right\rangle \\
\leq & \mathbb{E}\left[\norm{\bw_{t}-\bw^{*}}^{2}\right]-2\eta_{t}\lambda\norm{\bw_{t}-\bw^{*}}^{2}+\eta_{t}^{2}\mathbb{E}\left[\norm{h_{t}}^{2}\right]-2\eta_{t}\left\langle \bw_{t}-\bw^{*},Z_{t}\Delta_{t}\right\rangle 
\end{flalign*}

Here we note that we have used $\left\langle \bw_{t}-\bw^{*},f^{'}\left(\bw_{t}\right)\right\rangle \geq\lambda\norm{\bw_{t}-\bw^{*}}^{2}$.
It comes from the following derivation
\[
f\left(\bw^{*}\right)-f\left(\bw_{t}\right)\geq\left\langle f^{'}\left(\bw_{t}\right),\bw^{*}-\bw_{t}\right\rangle +\frac{\lambda}{2}\norm{\bw_{t}-\bw^{*}}^{2}
\]
\begin{flalign*}
\left\langle f^{'}\left(\bw_{t}\right),\bw_{t}-\bw^{*}\right\rangle  & \geq f\left(\bw_{t}\right)-f\left(\bw^{*}\right)+\frac{\lambda}{2}\norm{\bw_{t}-\bw^{*}}^{2}\geq\left\langle f^{'}\left(\bw^{*}\right),\bw_{t}-\bw^{*}\right\rangle +\lambda\norm{\bw_{t}-\bw^{*}}^{2}\\
 & \geq\lambda\norm{\bw_{t}-\bw^{*}}^{2}\,\,\,\,\,\,\,\,\,\,\,\,\,\,\,\,\,\,\,\,\,\,\,\,\,\,\,\,\,\,\,\text{\,\,\,\,\,\,\,\,\,\,thanks to}\left\langle f^{'}\left(\bw^{*}\right),\bw_{t}-\bw^{*}\right\rangle \geq0
\end{flalign*}

Taking expectation again, we gain 

\[
\mathbb{E}\left[\norm{\bw_{t+1}-\bw^{*}}^{2}\right]\leq\frac{t-2}{t}\mathbb{E}\left[\norm{\bw_{t}-\bw^{*}}^{2}\right]+\frac{H}{\lambda^{2}t^{2}}+\frac{2W^{1/2}M^{1/2}\delta_{\Phi}}{\lambda t}
\]

Choosing $W=\left(\frac{M^{1/2}\delta_{\phi}+\left(M\delta_{\phi}^{2}+2H\right)^{1/2}}{2\lambda}\right)^{2}$,
we gain if $\mathbb{E}\left[\norm{\bw_{t}-\bw^{*}}^{2}\right]\leq W$
then $\mathbb{E}\left[\norm{\bw_{t+1}-\bw^{*}}^{2}\right]\leq W$.
The reason is that $W=\left(\frac{M^{1/2}\delta_{\phi}+\left(M\delta_{\phi}^{2}+2H\right)^{1/2}}{2\lambda}\right)^{2}$
is the solution of the equation 
\[
W=\frac{t-2}{t}W+\frac{H}{\lambda^{2}t}+\frac{2W^{1/2}M^{1/2}\delta_{\Phi}}{\lambda t}\,\text{or\,\ensuremath{2W-\frac{2W^{1/2}M^{1/2}\delta_{\Phi}}{\lambda}-\frac{H}{\lambda^{2}}=0}}
\]
. Hence, if $\mathbb{E}\left[\norm{\bw_{t}-\bw^{*}}^{2}\right]\leq W$,
we arrive at
\[
\mathbb{E}\left[\norm{\bw_{t+1}-\bw^{*}}^{2}\right]\leq\frac{t-2}{t}W+\frac{H}{\lambda^{2}t}+\frac{2W^{1/2}M^{1/2}\delta_{\Phi}}{\lambda t}=W
\]
\end{proof}

We now show the proof of Theorem \ref{thm:regret1}.
\begin{proof}
\textbf{of Theorem \ref{thm:regret1}}

We first remind the definitions of the relevant quantities
\begin{align*}
g_{t} & =\lambda\bw_{t}+l^{'}\left(\bw_{t};x_{t},y_{t}\right)=\lambda\bw_{t}+\alpha_{t}\Phi\left(x_{t}\right)\\
h_{t} & =g_{t}+Z_{t}\Delta_{t}\,\text{where \ensuremath{\Delta_{t}=\alpha_{t}\left(\Phi\left(c_{i_{t}}\right)-\Phi\left(x_{t}\right)\right)}}
\end{align*}

We then derive as follows
\begin{flalign*}
\norm{\bw_{t+1}-\bw^{*}}^{2} & =\norm{\prod_{S}\left(\bw_{t}-\eta_{t}h_{t}\right)-\bw^{*}}^{2}\leq\norm{\bw_{t}-\eta_{t}h_{t}-\bw^{*}}^{2}\\
 & =\norm{\bw_{t}-\bw^{*}}+\eta_{t}^{2}\norm{h_{t}}^{2}-2\eta_{t}\left\langle \bw_{t}-\bw^{*},g_{t}+Z_{t}\Delta_{t}\right\rangle 
\end{flalign*}

\noindent where $S=\mathbb{R}^{D}$ or $\mathcal{B}\left(\bzero,y_{\text{max}}\lambda^{-1/2}\right)$.

\[
\left\langle \bw_{t}-\bw^{*},g_{t}\right\rangle \leq\frac{\norm{\bw_{t}-\bw^{*}}^{2}-\norm{\bw_{t+1}-\bw^{*}}^{2}}{2\eta_{t}}+\frac{\eta_{t}}{2}\norm{h_{t}}^{2}-\left\langle \bw_{t}-\bw^{*},Z_{t}\Delta_{t}\right\rangle 
\]

Taking conditional expectation w.r.t $\bw_{t}$, we obtain
\begin{align*}
\left\langle \bw_{t}-\bw^{*},f^{'}\left(\bw_{t}\right)\right\rangle  & \leq\frac{\mathbb{E}\left[\norm{\bw_{t}-\bw^{*}}^{2}\right]-\mathbb{E}\left[\norm{\bw_{t+1}-\bw^{*}}^{2}\right]}{2\eta_{t}}\\
 & +\frac{\eta_{t}}{2}\mathbb{E}\left[\norm{h_{t}}^{2}\right]-\left\langle \bw_{t}-\bw^{*},\mathbb{E}\left[Z_{t}\Delta_{t}\right]\right\rangle 
\end{align*}
\begin{flalign*}
f\left(\bw_{t}\right)-f\left(\bw^{*}\right)+\frac{\lambda}{2}\norm{\bw_{t}-\bw^{*}}^{2} & \leq\frac{\mathbb{E}\left[\norm{\bw_{t}-\bw^{*}}^{2}\right]-\mathbb{E}\left[\norm{\bw_{t+1}-\bw^{*}}^{2}\right]}{2\eta_{t}}\\
 & +\frac{\eta_{t}}{2}\mathbb{E}\left[\norm{h_{t}}^{2}\right]-\left\langle \bw_{t}-\bw^{*},\mathbb{E}\left[Z_{t}\Delta_{t}\right]\right\rangle 
\end{flalign*}

Taking expectation again, we achieve
\begin{flalign*}
\mathbb{E}\left[f\left(\bw_{t}\right)-f\left(\bw^{*}\right)\right] & \leq\frac{\lambda}{2}\left(t-1\right)\mathbb{E}\left[\norm{\bw_{t}-\bw^{*}}^{2}\right]-\frac{\lambda}{2}t\mathbb{E}\left[\norm{\bw_{t+1}-\bw^{*}}^{2}\right]\\
 & +\frac{\eta_{t}}{2}\mathbb{E}\left[\norm{h_{t}}^{2}\right]-\mathbb{E}\left[\left\langle \bw_{t}-\bw^{*},Z_{t}\Delta_{t}\right\rangle \right]
\end{flalign*}

i) If $Z_{t}$ is independent with $\bw_{t}$, we derive as

\begin{flalign}
\mathbb{E}\left[f\left(\bw_{t}\right)-f\left(\bw^{*}\right)\right] & \leq\frac{\lambda}{2}\left(t-1\right)\mathbb{E}\left[\norm{\bw_{t}-\bw^{*}}^{2}\right]-\frac{\lambda}{2}t\mathbb{E}\left[\norm{\bw_{t+1}-\bw^{*}}^{2}\right]\nonumber \\
 & +\frac{\eta_{t}}{2}\mathbb{E}\left[\norm{h_{t}}^{2}\right]-\mathbb{E}\left[\left\langle Z_{t}\left(\bw_{t}-\bw^{*}\right),\Delta_{t}\right\rangle \right]\nonumber \\
 & \leq\frac{\lambda}{2}\left(t-1\right)\mathbb{E}\left[\norm{\bw_{t}-\bw^{*}}^{2}\right]-\frac{\lambda}{2}t\mathbb{E}\left[\norm{\bw_{t+1}-\bw^{*}}^{2}\right]+\frac{\eta_{t}}{2}\mathbb{E}\left[\norm{h_{t}}^{2}\right]\nonumber \\
 & +\mathbb{E}\left[Z_{t}^{2}\norm{\bw_{t}-\bw^{*}}^{2}\right]^{1/2}\mathbb{E}\left[\norm{\Delta_{t}}^{2}\right]^{1/2}\nonumber \\
 & \leq\frac{\lambda}{2}\left(t-1\right)\mathbb{E}\left[\norm{\bw_{t}-\bw^{*}}^{2}\right]-\frac{\lambda}{2}t\mathbb{E}\left[\norm{\bw_{t+1}-\bw^{*}}^{2}\right]+\frac{\eta_{t}}{2}\mathbb{E}\left[\norm{h_{t}}^{2}\right]\nonumber \\
 & +\mathbb{E}\left[Z_{t}^{2}\right]^{1/2}\mathbb{E}\left[\norm{\bw_{t}-\bw^{*}}^{2}\right]^{1/2}\mathbb{E}\left[\norm{\Delta_{t}}^{2}\right]^{1/2}\nonumber \\
 & \leq\frac{\lambda}{2}\left(t-1\right)\mathbb{E}\left[\norm{\bw_{t}-\bw^{*}}^{2}\right]-\frac{\lambda}{2}t\mathbb{E}\left[\norm{\bw_{t+1}-\bw^{*}}^{2}\right]+\frac{\eta_{t}}{2}H\label{eq:independent_w}\\
 & +\mathbb{P}\left(Z_{t}=1\right)^{1/2}\mathbb{E}\left[\norm{\bw_{t}-\bw^{*}}^{2}\right]^{1/2}\mathbb{E}\left[\norm{\Delta_{t}}^{2}\right]^{1/2}\nonumber 
\end{flalign}

Taking sum over $1,2,\ldots,T$ and using the inequalities in Lemmas
\ref{lem:main_lemma} and \ref{lem:lem_W}, we yield

\begin{equation}
\sum_{t=1}^{T}\mathbb{E}\left[f\left(\bw_{t}\right)\right]-Tf\left(\bw^{*}\right)\leq\frac{H}{2\lambda}\sum_{t=1}^{T}\frac{1}{t}+\sum_{t=1}^{T}\mathbb{P}\left(Z_{t}=1\right)^{1/2}\mathbb{E}\left[\norm{\bw_{t}-\bw^{*}}^{2}\right]^{1/2}\mathbb{E}\left[\norm{\Delta_{t}}^{2}\right]^{1/2}\label{eq:regret1}
\end{equation}
\[
T\mathbb{E}\left[f\left(\overline{\bw}_{T}\right)-f\left(\bw^{*}\right)\right]\leq\frac{H}{2\lambda}\sum_{t=1}^{T}\frac{1}{t}+\sum_{t=1}^{T}\mathbb{P}\left(Z_{t}=1\right)^{1/2}\mathbb{E}\left[\norm{\bw_{t}-\bw^{*}}^{2}\right]^{1/2}\mathbb{E}\left[\norm{\Delta_{t}}^{2}\right]^{1/2}
\]

\begin{flalign*}
\mathbb{E}\left[f\left(\overline{\bw}_{T}\right)-f\left(\bw^{*}\right)\right] & \leq\frac{H\left(\log\left(T\right)+1\right)}{2\lambda T}+\frac{\delta_{\Phi}M^{1/2}W^{1/2}}{T}\sum_{t=1}^{T}\mathbb{P}\left(Z_{t}=1\right)^{1/2}
\end{flalign*}

ii) If $Z_{t}$ is independent with $\bw_{t}$ and $\left(x_{t},y_{t}\right)$,
we derive as

\begin{flalign}
\mathbb{E}\left[f\left(\bw_{t}\right)-f\left(\bw^{*}\right)\right] & \leq\frac{\lambda}{2}\left(t-1\right)\mathbb{E}\left[\norm{\bw_{t}-\bw^{*}}^{2}\right]-\frac{\lambda}{2}t\mathbb{E}\left[\norm{\bw_{t+1}-\bw^{*}}^{2}\right]\nonumber \\
 & +\frac{\eta_{t}}{2}H-\mathbb{E}\left[Z_{t}\left\langle \bw_{t}-\bw^{*},\Delta_{t}\right\rangle \right]\nonumber \\
 & \leq\frac{\lambda}{2}\left(t-1\right)\mathbb{E}\left[\norm{\bw_{t}-\bw^{*}}^{2}\right]-\frac{\lambda}{2}t\mathbb{E}\left[\norm{\bw_{t+1}-\bw^{*}}^{2}\right]\nonumber \\
 & +\frac{\eta_{t}}{2}H-\mathbb{E}\left[Z_{t}\right]\mathbb{E}\left[\left\langle \bw_{t}-\bw^{*},\Delta_{t}\right\rangle \right]\nonumber \\
 & \leq\frac{\lambda}{2}\left(t-1\right)\mathbb{E}\left[\norm{\bw_{t}-\bw^{*}}^{2}\right]-\frac{\lambda}{2}t\mathbb{E}\left[\norm{\bw_{t+1}-\bw^{*}}^{2}\right]+\frac{\eta_{t}}{2}H\label{eq:independent_wxy}\\
 & +\mathbb{P}\left(Z_{t}=1\right)\mathbb{E}\left[\norm{\bw_{t}-\bw^{*}}^{2}\right]^{1/2}\mathbb{E}\left[\norm{\Delta_{t}}^{2}\right]^{1/2}\nonumber 
\end{flalign}

Taking sum over $1,2,\ldots,T$ and using the inequalities in Lemmas
\ref{lem:main_lemma} and \ref{lem:lem_W}, we yield

\begin{equation}
\sum_{t=1}^{T}\mathbb{E}\left[f\left(\bw_{t}\right)\right]-Tf\left(\bw^{*}\right)\leq\frac{H}{2\lambda}\sum_{t=1}^{T}\frac{1}{t}+\sum_{t=1}^{T}P\left(Z_{t}=1\right)\mathbb{E}\left[\norm{\bw_{t}-\bw^{*}}^{2}\right]^{1/2}\mathbb{E}\left[\norm{\Delta_{t}}^{2}\right]^{1/2}\label{eq:regret2}
\end{equation}
\[
T\mathbb{E}\left[f\left(\overline{\bw}_{T}\right)-f\left(\bw^{*}\right)\right]\leq\frac{H}{2\lambda}\sum_{t=1}^{T}\frac{1}{t}+\sum_{t=1}^{T}P\left(Z_{t}=1\right)\mathbb{E}\left[\norm{\bw_{t}-\bw^{*}}^{2}\right]^{1/2}\mathbb{E}\left[\norm{\Delta_{t}}^{2}\right]^{1/2}
\]

\begin{flalign*}
\mathbb{E}\left[f\left(\overline{\bw}_{T}\right)-f\left(\bw^{*}\right)\right] & \leq\frac{H\left(\log\left(T\right)+1\right)}{2\lambda T}+\frac{\delta_{\Phi}M^{1/2}W^{1/2}}{T}\sum_{t=1}^{T}P\left(Z_{t}=1\right)\\
\leq & \frac{H\left(\log\left(T\right)+1\right)}{2\lambda T}+\delta_{\Phi}M^{1/2}W^{1/2}
\end{flalign*}

iii) In general, we derive as
\begin{flalign}
\mathbb{E}\left[f\left(\bw_{t}\right)-f\left(\bw^{*}\right)\right] & \leq\frac{\lambda}{2}\left(t-1\right)\mathbb{E}\left[\norm{\bw_{t}-\bw^{*}}^{2}\right]-\frac{\lambda}{2}t\mathbb{E}\left[\norm{\bw_{t+1}-\bw^{*}}^{2}\right]\nonumber \\
 & +\frac{\eta_{t}}{2}H-\mathbb{E}\left[Z_{t}\left\langle \bw_{t}-\bw^{*},\Delta_{t}\right\rangle \right]\nonumber \\
 & \leq\frac{\lambda}{2}\left(t-1\right)\mathbb{E}\left[\norm{\bw_{t}-\bw^{*}}^{2}\right]-\frac{\lambda}{2}t\mathbb{E}\left[\norm{\bw_{t+1}-\bw^{*}}^{2}\right]\label{eq:no_independent}\\
 & +\frac{\eta_{t}}{2}H+\mathbb{E}\left[\norm{\bw_{t}-\bw^{*}}^{2}\right]^{1/2}\mathbb{E}\left[\norm{\Delta_{t}}^{2}\right]^{1/2}\nonumber 
\end{flalign}

Taking sum over $1,2,\ldots,T$ and using the inequalities in Lemmas
\ref{lem:main_lemma} and \ref{lem:lem_W}, we yield
\begin{equation}
\sum_{t=1}^{T}\mathbb{E}\left[f\left(\bw_{t}\right)\right]-Tf\left(\bw^{*}\right)\leq\frac{H}{2\lambda}\sum_{t=1}^{T}\frac{1}{t}+\sum_{t=1}^{T}\mathbb{E}\left[\norm{\bw_{t}-\bw^{*}}^{2}\right]^{1/2}\mathbb{E}\left[\norm{\Delta_{t}}^{2}\right]^{1/2}\label{eq:regret3}
\end{equation}

\[
T\mathbb{E}\left[f\left(\overline{\bw}_{T}\right)-f\left(\bw^{*}\right)\right]\leq\frac{H}{2\lambda}\sum_{t=1}^{T}\frac{1}{t}+\sum_{t=1}^{T}\mathbb{E}\left[\norm{\bw_{t}-\bw^{*}}^{2}\right]^{1/2}\mathbb{E}\left[\norm{\Delta_{t}}^{2}\right]^{1/2}
\]

\begin{flalign*}
\mathbb{E}\left[f\left(\overline{\bw}_{T}\right)-f\left(\bw^{*}\right)\right] & \leq\frac{H\left(\log\left(T\right)+1\right)}{2\lambda T}+\delta_{\Phi}M^{1/2}W^{1/2}
\end{flalign*}
\end{proof}

\begin{proof}
\textbf{of Theorem \ref{thm:ModelSize}}

Let us denote the model size, i.e., the number of vectors in support
set, after the iteration $t$ by $S_{t}$. We also define $N_{t}$
by the binary random variable which specifies whether the incoming
instance $\left(x_{t},y_{t}\right)$ locates in a new cell of the
coverage, that is, $N_{t}=1$ indicating the current cell $P_{i_{t}}$
is a new cell. We assume that $Z_{t}$ is independent with $\left(x_{t},y_{t}\right)$
and so does with $N_{t}$. Since a new instance is added to the support
set if either a new cell is discovered or the old cell is found but
approximation is not performed, we reach the following
\[
S_{t}\leq S_{t-1}+N_{t}+\left(1-Z_{t}\right)\left(1-N_{t}\right)
\]

Taking expectation, we obtain
\begin{flalign*}
\mathbb{E}\left[S_{t}\right] & \leq\mathbb{E}\left[S_{t-1}\right]+\mathbb{E}\left[N_{t}\right]+\left(1-\mathbb{E}\left[Z_{t}\right]\right)\left(1-\mathbb{E}\left[N_{t}\right]\right)\\
 & \leq\mathbb{E}\left[S_{t-1}\right]+\mathbb{E}\left[N_{t}\right]+\left(1-p_{t}\right)\left(1-\mathbb{E}\left[N_{t}\right]\right)\\
 & \leq\mathbb{E}\left[S_{t-1}\right]+\mathbb{E}\left[N_{t}\right]+q_{t}\left(1-\mathbb{E}\left[N_{t}\right]\right)
\end{flalign*}
\[
\mathbb{E}\left[S_{t}\right]-\mathbb{E}\left[S_{t-1}\right]\leq\mathbb{E}\left[N_{t}\right]+q_{t}\left(1-\mathbb{E}\left[N_{t}\right]\right)
\]

Summing over the above when $t=1,\ldots,T$, we have
\begin{flalign}
\mathbb{E}\left[S_{T}\right] & \leq\sum_{t=1}^{T}\mathbb{E}\left[N_{t}\right]+\sum_{t=1}^{T}q_{t}\left(1-\mathbb{E}\left[N_{t}\right]\right)=\sum_{t=1}^{T}q_{t}+\sum_{t=1}^{T}p_{t}\mathbb{E}\left[N_{t}\right]\nonumber \\
 & \leq\sum_{t=1}^{T}q_{t}+\sum_{t=1}^{T}\mathbb{E}\left[N_{t}\right]\leq\sum_{t=1}^{T}q_{t}+\mathbb{E}\left[M_{T}\right]\label{eq:bound_model_size}
\end{flalign}

\noindent where we have denoted $\mathbb{P}\left(Z_{t}=1\right)=p_{t}$,
$\mathbb{P}\left(Z_{t}=0\right)=q_{t}$, and $M_{T}=\sum_{t=1}^{T}N_{t}$
indicates the number of cells discovered so far.

We consider some specific cases and investigate the model size $\mathbb{E}\left[S_{T}\right]$
in these cases.

i) $p_{t}=\mathbb{P}\left(Z_{t}=1\right)=1,\,\forall t$, that is,
we always do approximation. From Eq. (\ref{eq:bound_model_size}),
we obtain
\[
\mathbb{E}\left[S_{T}\right]\leq\mathbb{E}\left[M_{T}\right]\leq\left|\mathcal{P}\right|
\]

ii) $p_{t}=\mathbb{P}\left(Z_{t}=1\right)=\max\left(0,1-\frac{\beta}{t}\right),\,\forall t$.
It follows that
\[
q_{t}=1-p_{t}\leq1-\left(1-\frac{\beta}{t}\right)=\frac{\beta}{t}
\]

From Eq. (\ref{eq:bound_model_size}), we gain
\begin{align*}
\mathbb{E}\left[S_{T}\right] & \leq\beta\sum_{t=1}^{T}\frac{1}{t}+\mathbb{E}\left[M_{T}\right]\leq\beta\left(1+\int_{1}^{T}\frac{1}{t}dt\right)+\mathbb{E}\left[M_{T}\right]\\
 & \leq\beta\left(\log T+1\right)+\mathbb{E}\left[M_{T}\right]
\end{align*}

iii) $p_{t}=\mathbb{P}\left(Z_{t}=1\right)=\max\left(0,1-\frac{\beta}{t^{\rho}}\right),\,\forall t$
where $0<\rho<1$. It follows that
\[
q_{t}=1-p_{t}\leq1-\left(1-\frac{\beta}{t^{\rho}}\right)=\frac{\beta}{t^{\rho}}
\]

From Eq. (\ref{eq:bound_model_size}), we gain
\[
\mathbb{E}\left[S_{T}\right]\leq\beta\sum_{t=1}^{T}\frac{1}{t^{\rho}}+\mathbb{E}\left[M_{T}\right]\leq\beta\left(1+\int_{1}^{T}t^{-\rho}dt\right)+\mathbb{E}\left[M_{T}\right]\leq\frac{\beta T^{1-\rho}}{1-\rho}+\mathbb{E}\left[M_{T}\right]
\]

iv) $p_{t}=\mathbb{P}\left(Z_{t}=1\right)=\max\left(0,1-\frac{\beta}{t^{\rho}}\right),\,\forall t$
where $\rho>1$. It follows that
\[
q_{t}=1-p_{t}\leq1-\left(1-\frac{\beta}{t^{\rho}}\right)=\frac{\beta}{t^{\rho}}
\]

From Eq. (\ref{eq:bound_model_size}), we gain
\[
\mathbb{E}\left[S_{T}\right]\leq\beta\sum_{t=1}^{T}\frac{1}{t^{\rho}}+\mathbb{E}\left[M_{T}\right]\leq\beta\zeta\left(\rho\right)+\mathbb{E}\left[M_{T}\right]\leq\beta\zeta\left(\rho\right)+|\mathcal{P}|
\]

where $\zeta\left(.\right)$ is $\zeta$- Riemann function defined
by the integral $\zeta\left(s\right)=\frac{1}{\Gamma\left(s\right)}\int_{0}^{+\infty}\frac{t^{s-1}}{e^{s}-1}dt$.
\end{proof}
We now show the proof of Theorem \ref{thm:wT_confidence}. To realize
this proof, we use the famous inequality, namely Hoeffding which for
completeness we state below.
\begin{thm*}
(Hoeffding inequality) Let the independent variables $X_{1},\ldots,X_{n}$
where $a_{i}\leq X_{i}\leq b_{i}$ for each $i\in\left[n\right]$.
Let $S=\sum_{i=1}^{n}X_{i}$ and $\Delta_{i}=b_{i}-a_{i}$. The following
hold

i) $\mathbb{P}\left(S-\mathbb{E}\left[S\right]>\varepsilon\right)\leq\exp\left(-\frac{2\varepsilon^{2}}{\sum_{i=1}^{n}\Delta_{i}^{2}}\right)$.

ii) $\mathbb{P}\left(\left|S-\mathbb{E}\left[S\right]\right|>\varepsilon\right)\leq2\exp\left(-\frac{2\varepsilon^{2}}{\sum_{i=1}^{n}\Delta_{i}^{2}}\right)$
\end{thm*}

\begin{proof}
\textbf{of Theorem \ref{thm:wT_confidence}}

From Eqs. (\ref{eq:regret1}, \ref{eq:regret2}, \ref{eq:regret3}),
we achieve

\[
\frac{1}{T}\sum_{t=1}^{T}\mathbb{E}\left[f\left(\bw_{t}\right)\right]-f\left(\bw^{*}\right)\leq\frac{H\left(\log\left(T\right)+1\right)}{2\lambda T}+d_{T}
\]

Let us denote $X=f\left(\bw_{r}\right)-f\left(\bw^{*}\right)$, where
$r$ is uniformly sampled from $\left\{ 1,2,\ldots,T\right\} $. We
have
\[
\mathbb{E}_{r}\left[X\right]=\frac{1}{T}\sum_{t=1}^{T}\mathbb{E}\left[f\left(\bw_{t}\right)\right]-f\left(\bw^{*}\right)\leq\frac{H\left(\log\left(T\right)+1\right)}{2\lambda T}+d_{T}
\]

It follows that
\[
\mathbb{E}\left[X\right]=\mathbb{E}_{\left(x_{t},y_{t}\right)_{t=1}^{T}}\left[\mathbb{E}_{r}\left[X\right]\right]\leq\frac{H\left(\log\left(T\right)+1\right)}{2\lambda T}+d_{T}
\]

Let us denote $\Delta_{T}=\underset{1\leq t\leq T}{\max}\left(f\left(\bw_{t}\right)-f\left(\bw^{*}\right)\right)$
which implies that $0<f\left(\bw_{r}\right)-f\left(\bw^{*}\right)<\Delta_{T}$.
Applying Hoeffding inequality for the random variable $X$, we gain
\[
\mathbb{P}\left(X-\mathbb{E}\left[X\right]>\varepsilon\right)\leq\exp\left(-\frac{2\varepsilon^{2}}{\Delta_{T}^{2}}\right)
\]
\[
\mathbb{P}\left(X-\frac{H\left(\log\left(T\right)+1\right)}{2\lambda T}-d_{T}>\varepsilon\right)\leq\exp\left(-\frac{2\varepsilon^{2}}{\Delta_{T}^{2}}\right)
\]
\[
\mathbb{P}\left(X\leq\frac{H\left(\log\left(T\right)+1\right)}{2\lambda T}+d_{T}+\varepsilon\right)>1-\exp\left(-\frac{2\varepsilon^{2}}{\Delta_{T}^{2}}\right)
\]

Choosing $\delta=\exp\left(-\frac{2\varepsilon^{2}}{\Delta_{T}^{2}}\right)$
or $\varepsilon=\Delta_{T}\sqrt{\frac{1}{2}\log\frac{1}{\delta}}$,
then with the probability at least $1-\delta$, we have
\[
f\left(\bw_{r}\right)-f\left(\bw^{*}\right)\leq\frac{H\left(\log\left(T\right)+1\right)}{2\lambda T}+d_{T}+\Delta_{T}\sqrt{\frac{1}{2}\log\frac{1}{\delta}}
\]
\end{proof}

\begin{proof}
\textbf{of Theorem \ref{thm:regret_alpha}}

We denote $W_{T}^{\alpha}=\mathbb{E}\left[\norm{\bw_{\left(1-\alpha\right)T+1}-\bw^{*}}^{2}\right]$.
Our proof proceeds as follows.

i) If $Z_{t}$ is independent with $\bw_{t}$, taking sum in Eq. (\ref{eq:independent_w})
when $t=\left(1-\alpha\right)T+1,\ldots,T$, we gain
\begin{flalign}
\sum_{t=\left(1-\alpha\right)T+1}^{T}\mathbb{E}\left[f\left(\bw_{t}\right)\right]-\alpha Tf\left(\bw^{*}\right) & \leq\frac{\lambda\left(1-\alpha\right)T}{2}W_{T}^{\alpha}+\frac{H}{2\lambda}\sum_{t=\left(1-\alpha\right)T+1}^{T}\frac{1}{t}\nonumber \\
 & +\sum_{t=1}^{T}\mathbb{P}\left(Z_{t}=1\right)^{1/2}\mathbb{E}\left[\norm{\bw_{t}-\bw^{*}}^{2}\right]^{1/2}\mathbb{E}\left[\norm{\Delta_{t}}^{2}\right]^{1/2}\label{eq:regret4}\\
\leq & \frac{\lambda\left(1-\alpha\right)T}{2}W_{T}^{\alpha}+\frac{H\log\left(1/\left(1-\alpha\right)\right)}{2\lambda}\nonumber \\
 & +\delta_{\Phi}M^{1/2}W^{1/2}\sum_{t=\left(1-\alpha\right)T+1}^{T}\mathbb{P}\left(Z_{t}=1\right)^{1/2}\nonumber 
\end{flalign}

\noindent where we have used the inequality $\sum_{t=\left(1-\alpha\right)T+1}^{T}\frac{1}{t}\leq\log\left(1/\left(1-\alpha\right)\right)$.

\begin{align*}
\mathbb{E}\left[f\left(\overline{\bw}_{T}^{\alpha}\right)-f\left(\bw^{*}\right)\right] & \leq\frac{\lambda\left(1-\alpha\right)}{2\alpha}W_{T}^{\alpha}+\frac{\delta_{\Phi}M^{1/2}W^{1/2}}{\alpha T}\sum_{t=\left(1-\alpha\right)T+1}^{T}\mathbb{P}\left(Z_{t}=1\right)^{1/2}\\
 & +\frac{H\log\left(1/\left(1-\alpha\right)\right)}{2\lambda\alpha T}
\end{align*}

ii) If $Z_{t}$ is independent with $\bw_{t}$ and $\left(x_{t},y_{t}\right)$,
taking sum in Eq. (\ref{eq:independent_wxy}) when $t=\left(1-\alpha\right)T+1,\ldots,T$,
we gain
\begin{flalign}
\sum_{t=\left(1-\alpha\right)T+1}^{T}\mathbb{E}\left[f\left(\bw_{t}\right)\right]-\alpha Tf\left(\bw^{*}\right) & \leq\frac{\lambda\left(1-\alpha\right)T}{2}W_{T}^{\alpha}+\frac{H}{2\lambda}\sum_{t=\left(1-\alpha\right)T+1}^{T}\frac{1}{t}\nonumber \\
 & +\sum_{t=\left(1-\alpha\right)T+1}^{T}\mathbb{P}\left(Z_{t}=1\right)\mathbb{E}\left[\norm{\bw_{t}-\bw^{*}}^{2}\right]^{1/2}\mathbb{E}\left[\norm{\Delta_{t}}^{2}\right]^{1/2}\label{eq:regret5}\\
\leq & \frac{\lambda\left(1-\alpha\right)T}{2}W_{T}^{\alpha}+\frac{H\log\left(1/\left(1-\alpha\right)\right)}{2\lambda}\nonumber \\
 & +\delta_{\Phi}M^{1/2}W^{1/2}\sum_{t=\left(1-\alpha\right)T+1}^{T}\mathbb{P}\left(Z_{t}=1\right)\nonumber 
\end{flalign}

\begin{align*}
\mathbb{E}\left[f\left(\overline{\bw}_{T}^{\alpha}\right)-f\left(\bw^{*}\right)\right] & \leq\frac{\lambda\left(1-\alpha\right)}{2\alpha}W_{T}^{\alpha}+\frac{\delta_{\Phi}M^{1/2}W^{1/2}}{\alpha T}\sum_{t=\left(1-\alpha\right)T+1}^{T}\mathbb{P}\left(Z_{t}=1\right)\\
 & +\frac{H\log\left(1/\left(1-\alpha\right)\right)}{2\lambda\alpha T}
\end{align*}

iii) In general, taking sum in Eq. (\ref{eq:no_independent}) when
$t=\left(1-\alpha\right)T+1,\ldots,T$, we gain

\begin{flalign}
\sum_{t=\left(1-\alpha\right)T+1}^{T}\mathbb{E}\left[f\left(\bw_{t}\right)\right]-\alpha Tf\left(\bw^{*}\right) & \leq\frac{\lambda\left(1-\alpha\right)T}{2}W_{T}^{\alpha}+\frac{H}{2\lambda}\sum_{t=\left(1-\alpha\right)T+1}^{T}\frac{1}{t}\nonumber \\
 & +\sum_{t=\left(1-\alpha\right)T+1}^{T}\mathbb{E}\left[\norm{\bw_{t}-\bw^{*}}^{2}\right]^{1/2}\mathbb{E}\left[\norm{\Delta_{t}}^{2}\right]^{1/2}\nonumber \\
\leq & \frac{\lambda\left(1-\alpha\right)T}{2}W_{T}^{\alpha}+\frac{H\log\left(1/\left(1-\alpha\right)\right)}{2\lambda}+\delta_{\Phi}M^{1/2}W^{1/2}\alpha T\label{eq:regret6}
\end{flalign}

\[
\mathbb{E}\left[f\left(\overline{\bw}_{T}^{\alpha}\right)-f\left(\bw^{*}\right)\right]\leq\frac{\lambda\left(1-\alpha\right)}{2\alpha}W_{T}^{\alpha}+\delta_{\Phi}M^{1/2}W^{1/2}+\frac{H\log\left(1/\left(1-\alpha\right)\right)}{2\lambda\alpha T}
\]
\end{proof}

\begin{proof}
\textbf{of Theorem \ref{thm:walphaT_confidence}}

The proof of this theorem is similar to that of Theorem \ref{thm:wT_confidence}
which relies on Hoeffding inequality. 

From Eqs. (\ref{eq:regret4}, \ref{eq:regret5}, \ref{eq:regret6}),
we achieve

\[
\frac{1}{\alpha T}\sum_{t=\left(1-\alpha\right)T+1}^{T}\mathbb{E}\left[f\left(\bw_{t}\right)\right]-f\left(\bw^{*}\right)\leq\frac{H\log\left(1/\left(1-\alpha\right)\right)}{2\lambda\alpha T}+d_{T}
\]

Let us denote $X=f\left(\bw_{r}\right)-f\left(\bw^{*}\right)$, where
$r$ is uniformly sampled from $\left\{ \left(1-\alpha\right)T+1,2,\ldots,T\right\} $.
We have
\[
\mathbb{E}_{r}\left[X\right]=\frac{1}{\alpha T}\sum_{t=\left(1-\alpha\right)T+1}^{T}\mathbb{E}\left[f\left(\bw_{t}\right)\right]-f\left(\bw^{*}\right)\leq\frac{H\log\left(1/\left(1-\alpha\right)\right)}{2\lambda\alpha T}+d_{T}
\]

It follows that
\[
\mathbb{E}\left[X\right]=\mathbb{E}_{\left(x_{t},y_{t}\right)_{t=1}^{T}}\left[\mathbb{E}_{r}\left[X\right]\right]\leq\frac{H\log\left(1/\left(1-\alpha\right)\right)}{2\lambda\alpha T}+d_{T}
\]

Let us denote $\Delta_{T}^{\alpha}=\underset{\left(1-\alpha\right)T+1\leq t\leq T}{\max}\left(f\left(\bw_{t}\right)-f\left(\bw^{*}\right)\right)$
which implies that $0<f\left(\bw_{r}\right)-f\left(\bw^{*}\right)<\Delta_{T}^{\alpha}$.
Applying Hoeffding inequality for the random variable $X$, we gain
\[
\mathbb{P}\left(X-\mathbb{E}\left[X\right]>\varepsilon\right)\leq\exp\left(-\frac{2\varepsilon^{2}}{\left(\Delta_{T}^{\alpha}\right)^{2}}\right)
\]
\[
\mathbb{P}\left(X-\frac{H\log\left(1/\left(1-\alpha\right)\right)}{2\lambda\alpha T}-d_{T}>\varepsilon\right)\leq\exp\left(-\frac{2\varepsilon^{2}}{\left(\Delta_{T}^{\alpha}\right)^{2}}\right)
\]
\[
\mathbb{P}\left(X\leq\frac{H\log\left(1/\left(1-\alpha\right)\right)}{2\lambda\alpha T}+d_{T}+\varepsilon\right)>1-\exp\left(-\frac{2\varepsilon^{2}}{\left(\Delta_{T}^{\alpha}\right)^{2}}\right)
\]

Choosing $\delta=\exp\left(-\frac{2\varepsilon^{2}}{\left(\Delta_{T}^{\alpha}\right)^{2}}\right)$
or $\varepsilon=\Delta_{T}^{\alpha}\sqrt{\frac{1}{2}\log\frac{1}{\delta}}$,
then with the probability at least $1-\delta$, we have
\[
f\left(\bw_{r}\right)-f\left(\bw^{*}\right)\leq\frac{H\log\left(1/\left(1-\alpha\right)\right)}{2\lambda\alpha T}+d_{T}+\Delta_{T}^{\alpha}\sqrt{\frac{1}{2}\log\frac{1}{\delta}}
\]
\end{proof}

\begin{proof}
\textbf{of Theorem \ref{thm:smooth_loss}}

It is apparent that $f\left(\bw\right)$ is $L$-strongly smooth w.r.t
$\norm ._{2}$. Therefore, we have
\[
f\left(\bw_{r}\right)-f\left(\bw^{*}\right)\leq\transp{f^{'}\left(\bw^{*}\right)}\left(\bw_{r}-\bw^{*}\right)+\frac{L}{2}\norm{\bw_{r}-\bw^{*}}^{2}\leq\frac{L}{2}\norm{\bw_{r}-\bw^{*}}^{2}
\]

It follows that $\Delta_{T}^{\alpha}\leq\frac{1}{2}LM_{T}^{\alpha}$.
Hence we gain the conclusion.
\end{proof}

\clearpage{}\newpage{}
\section{Proofs of Bound for L2 Loss\label{sec:appxC}}

We now consider the upper bound of $\norm{\bw^{*}}$ in the case that
$\ell_{2}$ loss is being used for regression problem. Concretely,
we have the following theorem whose proof is similar to that of Theorem
1 in \citep{Shalev-Shwartz:2007,shalev2011pegasos}.
\begin{thm}
\label{thm:bound_w*}If $\bw^{*}=\text{argmi\ensuremath{n_{\bw}}}\left(\frac{\lambda}{2}\norm{\bw}^{2}+\frac{1}{N}\sum_{i=1}^{N}\left(y_{i}-\transp{\bw}\Phi\left(x_{i}\right)\right)^{2}\right)$
then $\norm{\bw^{*}}\leq y_{\text{max}}\lambda^{-1/2}$.
\end{thm}
\begin{proof}
Let us consider the equivalent constrains optimization problem
\begin{flalign*}
 & \underset{\bw,\boldsymbol{\xi}}{\min}\left(\frac{\lambda}{2}\norm{\bw}^{2}+\frac{1}{N}\sum_{i=1}^{N}\xi_{i}^{2}\right)\\
\text{s.t.:} & \,\xi_{i}=y_{i}-\transp{\bw}\Phi\left(x_{i}\right),\,\forall i
\end{flalign*}

The Lagrange function is of the following form
\[
\mathcal{L}\left(\bw,\boldsymbol{\xi,\alpha}\right)=\frac{\lambda}{2}\norm{\bw^{2}}+\frac{1}{N}\sum_{i=1}^{N}\xi_{i}^{2}+\sum_{i=1}^{N}\alpha_{i}\left(y_{i}-\transp{\bw}\Phi\left(x_{i}\right)-\xi_{i}\right)
\]

Setting the derivatives to $0$, we gain
\[
\nabla_{\bw}\mathcal{L}=\lambda\bw-\sum_{i=1}^{N}\alpha_{i}\Phi\left(x_{i}\right)=0\goto\bw=\lambda^{-1}\sum_{i=1}^{N}\alpha_{i}\Phi\left(x_{i}\right)
\]

\[
\nabla_{\xi_{i}}\mathcal{L}=\frac{2}{N}\xi_{i}-\alpha_{i}=0\goto\xi_{i}=\frac{N\alpha_{i}}{2}
\]

Substituting the above to the Lagrange function, we gain the dual
form
\begin{flalign*}
\mathcal{W}\left(\boldsymbol{\alpha}\right)= & -\frac{\lambda}{2}\norm{\bw}^{2}+\sum_{i=1}^{N}y_{i}\alpha_{i}-\frac{N}{4}\sum_{i=1}^{N}\alpha_{i}^{2}\\
 & =-\frac{1}{2\lambda}\norm{\sum_{i=1}\alpha_{i}\Phi\left(x_{i}\right)}^{2}+\sum_{i=1}^{N}y_{i}\alpha_{i}-\frac{N}{4}\sum_{i=1}^{N}\alpha_{i}^{2}
\end{flalign*}

Let us denote $\left(\bw^{*},\boldsymbol{\xi^{*}}\right)$ and $\boldsymbol{\alpha^{*}}$
be the primal and dual solutions, respectively. Since the strong duality
holds, we have
\[
\frac{\lambda}{2}\norm{\bw^{*}}^{2}+\frac{1}{N}\sum_{i=1}^{N}\xi_{i}^{*2}=-\frac{\lambda}{2}\norm{\bw^{*}}^{2}+\sum_{i=1}^{N}y_{i}\alpha_{i}^{*}-\frac{N}{4}\sum_{i=1}^{N}\alpha_{i}^{*2}
\]
\begin{flalign*}
\lambda\norm{\bw^{*}}^{2} & =\sum_{i=1}^{N}y_{i}\alpha_{i}^{*}-\frac{N}{4}\sum_{i=1}^{N}\alpha_{i}^{*2}-\frac{1}{N}\sum_{i=1}^{N}\xi_{i}^{*2}\\
 & \leq\sum_{i=1}^{N}\left(y_{i}\alpha_{i}^{*}-\frac{N}{4}\alpha_{i}^{*2}\right)\leq\sum_{i=1}^{N}\frac{y_{i}^{2}}{N}\leq y_{\text{max}}^{2}
\end{flalign*}

We note that we have used $g\left(\alpha_{i}^{*}\right)=y_{i}\alpha_{i}^{*}-\frac{N}{4}\alpha_{i}^{*2}\leq g\left(\frac{2y_{i}}{N}\right)=\frac{y_{i}^{2}}{N}$.
Hence, we gain the conclusion. 
\end{proof}
\begin{lem}
Assume that $\ell_{2}$ loss is using, the following statement holds
\[
\norm{\bw_{T+1}}\leq\lambda^{-1}\left(y_{\text{max}}+\frac{1}{T}\sum_{t=1}^{T}\norm{\bw_{t}}\right)
\]

\noindent where $y_{\text{max}}=\underset{y\in\mathcal{Y}}{\max}\left|y\right|$.
\end{lem}
\begin{proof}
We have the following
\[
\bw_{t+1}=\begin{cases}
\prod_{S}\left(\frac{t-1}{t}\bw_{t}-\eta_{t}\alpha_{t}\Phi\left(x_{t}\right)\right) & \text{if}\,\,Z_{t}=0\\
\prod_{S}\left(\frac{t-1}{t}\bw_{t}-\eta_{t}\alpha_{t}\Phi\left(c_{i_{t}}\right)\right) & \text{otherwise}
\end{cases}
\]

It follows that
\[
\norm{\bw_{t+1}}\leq\frac{t-1}{t}\norm{\bw_{t}}+\frac{1}{\lambda t}\left|\alpha_{t}\right|\,\,\,\,\,\,\,\,\text{since}\,\,\norm{\Phi\left(x_{t}\right)}=\norm{\Phi\left(c_{i_{t}}\right)}=1
\]

It happens that $l^{'}\left(\bw_{t};x_{t},y_{t}\right)=\alpha_{t}\Phi\left(x_{t}\right)$.
Hence, we gain 
\[
\left|\alpha_{t}\right|=\left|y_{t}-\transp{\bw_{t}}\Phi\left(x_{t}\right)\right|\leq y_{\text{max}}+\norm{\bw_{t}}\norm{\Phi\left(x_{t}\right)}\le y_{\text{max}}+\norm{\bw_{t}}
\]

It implies that
\[
t\norm{\bw_{t+1}}\leq\left(t-1\right)\norm{\bw_{t}}+\lambda^{-1}\left(y_{\text{max}}+\norm{\bw_{t}}\right)
\]

Taking sum when $t=1,2,\ldots,T$, we achieve
\[
T\norm{\bw_{T+1}}\leq\lambda^{-1}\left(Ty_{\text{max}}+\sum_{t=1}^{T}\norm{\bw_{t}}\right)
\]
\begin{equation}
\norm{\bw_{T+1}}\leq\lambda^{-1}\left(y_{\text{max}}+\frac{1}{T}\sum_{t=1}^{T}\norm{\bw_{t}}\right)\label{eq:bound_w_L2}
\end{equation}
\end{proof}
\begin{thm}
\label{thm:bound_wT}If $\lambda>1$ then $\norm{\bw_{T+1}}\leq\frac{y_{\text{max}}}{\lambda-1}\left(1-\frac{1}{\lambda^{T}}\right)<\frac{y_{\text{max}}}{\lambda-1}$
for all $T$.
\end{thm}
\begin{proof}
First we consider the sequence $\left\{ s_{T}\right\} _{T}$ which
is identified as $s_{T+1}=\lambda^{-1}\left(y_{\text{max}}+s_{T}\right)$
and $s_{1}=0$. It is easy to find the formula of this sequence as
\[
s_{T+1}-\frac{y_{\text{max}}}{\lambda-1}=\lambda^{-1}\left(s_{T}-\frac{y_{\text{max}}}{\lambda-1}\right)=\ldots=\lambda^{-T}\left(s_{1}-\frac{y_{\text{max}}}{\lambda-1}\right)=\frac{-\lambda^{-T}y_{\text{max}}}{\lambda-1}
\]

\[
s_{T+1}=\frac{y_{\text{max}}}{\lambda-1}\left(1-\frac{1}{\lambda^{T}}\right)
\]

We prove by induction by $T$ that $\norm{\bw_{T}}\leq s_{T}$ for
all $T$. It is obvious that $\norm{\bw_{1}}=s_{1}=0$. Assume that
$\norm{\bw_{t}}\leq s_{t}$ for $t\leq T$, we verify it for $T+1$.
Indeed, we have
\begin{flalign*}
\norm{\bw_{T+1}} & \leq\lambda^{-1}\left(y_{\text{max}}+\frac{1}{T}\sum_{t=1}^{T}\norm{\bw_{t}}\right)\leq\lambda^{-1}\left(y_{\text{max}}+\frac{1}{T}\sum_{t=1}^{T}s_{t}\right)\\
 & \leq\lambda^{-1}\left(y_{\text{max}}+s_{T}\right)=s_{T+1}
\end{flalign*}
\end{proof}
In addition, from Eq. (\ref{eq:bound_w_L2}) in case that $\lambda\leq1$
we cannot bound $\norm{\bw_{T+1}}$. Concretely, we have the following
theorem.
\begin{thm}
If $\left\{ z_{T}\right\} _{T}$ is a sequence such that $z_{T+1}=\lambda^{-1}\left(y_{\text{max}}+\frac{1}{T}\sum_{t=1}^{T}z_{t}\right)$
with $z_{1}=0$ then in case that $\lambda\leq1$ this sequence is
not upper-bounded.
\end{thm}
\begin{proof}
Let us denote $s_{t}=y_{\text{max}}^{-1}z_{t}$. It is obvious that
$s_{1}=0$ and $s_{T+1}=\lambda^{-1}\left(1+\frac{1}{T}\sum_{t=1}^{T}s_{t}\right)$.
We now prove by induction by $T$ that
\[
s_{T}\geq\sum_{t=1}^{T-1}\frac{1}{t}\,\,\,\text{for all \ensuremath{T\geq2}}
\]

With $T=2$, we have $s_{2}=\lambda^{-1}\geq1$. Assume that this
holds for all $2\leq t\leq T$, we verify it for $T+1$.
\begin{flalign*}
s_{T+1} & =\lambda^{-1}\left(1+\frac{1}{T}\sum_{t=1}^{T}s_{t}\right)\geq1+\frac{1}{T}\sum_{t=1}^{T}s_{t}\\
 & \geq1+\frac{1}{T}\sum_{t=1}^{T}\sum_{n=1}^{t-1}\frac{1}{n}\geq1+\sum_{t=1}^{T-1}\frac{T-t}{Tt}\\
 & \geq1+\sum_{t=1}^{T}\left(\frac{1}{t}-\frac{1}{T}\right)=\sum_{t=1}^{T}\frac{1}{t}
\end{flalign*}

The final conclusion is obvious from the fact $\sum_{t=1}^{T}\frac{1}{t}>\log\left(T+1\right)$
and hence can exceed any positive number.
\end{proof}

\clearpage{}

\newpage{}

\bibliographystyle{plainnat}
\bibliography{AVM}

\end{document}